\pdfoutput=1
\documentclass[11pt]{article}
\usepackage{enumerate}
\usepackage[OT1]{fontenc}
\usepackage{amsmath,amssymb}
\usepackage{natbib}
\usepackage[usenames]{color}
\usepackage{dsfont}
\usepackage{pgfplots}
\usepackage{smile}
\usepackage{multirow}
\usepackage{rotating}
\usepackage{enumerate}
\usepackage{esvect}
\usepackage{tikz}
\usetikzlibrary{patterns}
\usetikzlibrary{arrows}

\usepackage[colorlinks,
            linkcolor=red,
            anchorcolor=blue,
            citecolor=blue
            ]{hyperref}
\usepackage{algorithm}
\usepackage{algorithmic}
\usepackage{enumitem}

\allowdisplaybreaks

\def\supp{\mathop{\text{supp}}}

\long\def\comment#1{}

\def\cS{{\mathcal{S}}}

\newcommand{\bel}{\begin{eqnarray}\label}
\newcommand{\eel}{\end{eqnarray}}
\newcommand{\bes}{\begin{eqnarray*}}
\newcommand{\ees}{\end{eqnarray*}}

\let\hat\widehat
\let\tilde\widetilde

\def\supp{\mathop{\text{supp}\kern.2ex}}
\def\argmin{\mathop{\text{\rm arg\,min}}}
\def\argmax{\mathop{\text{\rm arg\,max}}}

\def\supp{\mathop{\text{supp}}}

\theoremstyle{plain}

\usepackage{mathrsfs}
\usepackage{fullpage}

\usepackage{hyperref}
\usepackage[protrusion=false,expansion=true]{microtype}

\def\##1\#{\begin{align}#1\end{align}}
\def\$#1\${\begin{align*}#1\end{align*}}

\usepackage{undertilde}

\theoremstyle{mytheoremstyle}

\usepackage{mathtools}

\title{Pessimism meets VCG: Learning Dynamic Mechanism Design via Offline Reinforcement Learning}

\begin{document}
\author{Boxiang Lyu\thanks{Booth School of Business, University of Chicago.
Email: \texttt{blyu@chicagobooth.edu}.} 
\qquad 
Zhaoran Wang
\thanks{Northwestern University.
Email: \texttt{zhaoranwang@gmail.com}.} 
\qquad 
Mladen Kolar
\thanks{Booth School of Business, University of Chicago.
Email: \texttt{mladen.kolar@chicagobooth.edu}.}
\qquad
Zhuoran Yang\thanks{Yale University. Email: \texttt{zhuoranyang.work@gmail.com}.}}

\maketitle
\begin{abstract}
Dynamic mechanism design has garnered significant attention from both computer scientists and economists in recent years. By allowing agents to interact with the seller over multiple rounds, where agents’ reward functions may change with time and are state-dependent, the framework is able to model a rich class of real-world problems. In these works, the interaction between agents and sellers is often assumed to follow a Markov Decision Process (MDP). 
We focus on the setting where the reward and transition functions of such an MDP are not known a priori, and we are attempting to recover the optimal mechanism using an a priori collected data set. In the setting where the function approximation is employed to handle large state spaces, with only mild assumptions on the expressiveness of the function class, we are able to design a dynamic mechanism using offline reinforcement learning algorithms. Moreover, learned mechanisms approximately have three key desiderata: efficiency, individual rationality, and truthfulness. Our algorithm is based on the pessimism principle and only requires a mild assumption on the coverage of the offline data set. To the best of our knowledge, our work provides the first offline RL algorithm for dynamic mechanism design without assuming uniform coverage.
\end{abstract}

\section{Introduction}
Mechanism design studies how best to allocate goods among rational agents~\citep{maskin2008mechanism, myerson2008perspectives, roughgarden2010algorithmic}. Dynamic mechanism design focuses on analyzing optimal allocation rules in a changing environment, where demands for goods, the amount of available goods, and their valuations can vary over time~\citep{bergemann2019dynamic}. Problems ranging from online commerce and electric vehicle charging to pricing Wi-Fi access at Starbucks have been studied under the dynamic mechanism design framework~\citep{gallien2006dynamic, gerding2011online, friedman2003pricing}. Existing approaches in the literature require knowledge of the problem, such as the evaluation of goods by agents~\citep{bergemann2010dynamic,pavan2014dynamic}, the transition dynamics of the system~\citep{doepke2006dynamic}, or the policy that maximizes social welfare~\citep{parkes2003mdp,parkes2004approximately}.
Unfortunately, such knowledge is often not available in practice.

A practical approach we take in this paper is to learn a dynamic mechanism from data using offline Reinforcement Learning (RL). 
Vickrey-Clarke-Groves (VCG) mechanism provides a blueprint for the design of practical mechanisms in many problems and satisfies crucial mechanisms design desiderata in an extremely general setting~\citep{vickrey1961counterspeculation, clarke1971multipart, groves1979efficient}. In this paper, we approximate the desired VCG mechanism using a priori collected data~\citep{jin2021pessimism, xie2021bellman, zanette2021provable}. We assume that the mechanism designer does not know the utility of the agents or the transition kernel of the states, but has access to an offline data set that contains observed state transitions and utilities~\citep{lange2012batch}. The goal of the mechanism designer is to recover the ideal mechanism purely from this data set, without requiring interaction with the agents. 
We focus on an adaptation of the classic VCG mechanism to the dynamic setting~\citep{parkes2007online} and assume that agents' interactions with the seller follow an episodic Markov Decision Process (MDP), where the agents' rewards are state-dependent and evolve over time within each episode. To accommodate the rich class of quasilinear utility functions considered in the economic literature~\citep{bergemann2019dynamic}, we use offline RL with a general function approximation~\citep{xie2021bellman} to approximate the dynamic VCG mechanism.

{\par \textbf{Related Works.}}
\citet{parkes2003mdp} and \citet{parkes2004approximately} studied dynamic mechanism design from an MDP perspective. The proposed mechanisms can implement social welfare-maximizing policies in a truth-revealing Bayes-Nash equilibrium both exactly and approximately.
\citet{bapna2005efficient} studied the dynamic auction setting from a multi-arm bandit perspective. Using the notion of marginal contribution, \citet{bergemann2006efficient} proposed a dynamic mechanism that is efficient and truth-telling. \citet{pavan2009dynamic} analyzed the first-order conditions of efficient dynamic mechanisms. \citet{athey2013efficient} extended both the VCG and AGV mechanisms \citep{d1979incentives} to the dynamic regime, obtaining an efficient budget-balanced dynamic mechanism.~\citet{kakade2013optimal} proposed the virtual pivot mechanism that achieves incentive compatibility under a separability condition.~See \citet{cavallo2009mechanism},~\citet{bergemann2015introduction}, and~\citet{bergemann2019dynamic} for recent surveys on dynamic mechanism design. Our paper builds on the mechanism in~\citet{parkes2007online} and 
\citet{bergemann2010dynamic}, but focuses on learning a mechanism from data rather than designing a mechanism in a known environment.

Only a few recent works have investigated the learning of mechanisms.
\citet{kandasamy2020mechanism} provided an algorithm that recovers the VCG mechanism in a stationary multi-arm bandit setting. \citet{cen2021regret}, \citet{dai2021learning}, \citet{jagadeesan2021learning}, and \citet{liu2021bandit} studied the recovery of stable matching when the agents' utilities are given by bandit feedback. \citet{balcan2008reducing} shows that incentive-compatible mechanism design problems can be reduced to a structural risk minimization problem. In contrast, our work focuses on learning a dynamic mechanism in an offline setting.

Our paper is also related to the literature on offline RL~\citep{yu2020mopo, kumar2020conservative, liu2020provably, kidambi2020morel, jin2021pessimism,  xie2021bellman, zanette2021provable, yin2021towards, uehara2021pessimistic}. In the context of linear MDPs,
\citet{jin2021pessimism} provided a provably sample-efficient pessimistic value iteration algorithm, while \citet{zanette2021provable} used an actor-critic algorithm to further improve the upper bound. \citet{yin2021towards} proposed an instance-optimal method for tabular MDPs. \citet{uehara2021pessimistic} focused on model-based offline RL, while \citet{xie2021bellman} introduced a pessimistic soft policy iteration algorithm for offline RL with a general function approximation.
Compared to \citet{xie2021bellman}, in addition to the social welfare suboptimality, we also provide bounds on both the agents' and the seller's suboptimalities. We also show that our algorithm asymptotically satisfies key mechanism design desiderata, including truthfulness and individual rationality. Finally, we use optimistic and pessimistic estimates to learn the VCG prices, instead of the purely pessimistic approach discussed in~\citet{xie2021bellman}. This difference shows the difference between dynamic VCG and standard MDP. Our work also features a simplified proof of the main technical results in~\citet{xie2021bellman}.

Concurrent with our work,~\citet{lyu2022learning} studies the learning of a dynamic VCG mechanism in the online RL setting, where the mechanism is recovered through multiple rounds of interaction with the environment. Our work features several significant differences as we focus on general function approximation, whereas~\citet{lyu2022learning} only considers linear function approximation. We also focus on the offline RL setting, where the mechanism designer is not allowed to interact with the environment.

{\par \textbf{Our Contributions.}} We propose the first offline reinforcement learning algorithm that can learn a dynamic mechanism from any given data set. Additionally, our algorithm does not make any assumption about data coverage and only assumes that the underlying action-value functions are approximately realizable and the function class is approximately complete (see Assumptions~\ref{assumption:realizability} and~\ref{assumption:completeness} for detailed discussions), which makes the algorithm applicable to the wide range of real-world mechanism design problems with quasilinear, potentially non-convex utility functions~\citep{carbajal2013mechanism, bergemann2019dynamic}.

Our work features a soft policy iteration algorithm that allows for both optimistic and pessimistic estimates. 
When the data set has sufficient coverage of the optimal policy, the value function is realizable, and the function class is complete, our algorithm sublinearly converges to a mechanism with suboptimality $\cO(K^{-1/3})$, matching the rates obtained in~\citet{xie2021bellman}, where $K$ denotes the number of trajectories contained in the offline dataset. In addition to suboptimality guarantees, we further show that our algorithm is asymptotically individually rational and truthful with the same $\cO(K^{-1/3})$ guarantee. 

On the technical side, our work features a simplified theoretical analysis of pessimistic soft policy iteration algorithms~\citep{xie2021bellman}, using an adaptation of the classic tail bound discussed in~\citet{gyorfi2002distribution}. Moreover, unlike~\citep{xie2021bellman}, our simplified analysis is directly applicable to continuous function classes via a covering-based argument.

{\par \textbf{Notations.}} For any positive integer $z \in \ZZ_{> 0}$, let $[z] = \{1, 2, \ldots, z\}$. For any set $A$, let $\Delta(A)$ be the set of probability distributions supported on $A$. For two sequences $x_n, y_n$, we say $x_n = \cO(y_n)$ if there exist universal constants $n_0, C > 0$ such that $x_n < C y_n$ for all $n \geq n_0$. We use $\tilde{\cO}(\cdot)$ to denote $\cO(\cdot)$ ignoring log factors. Unless stated otherwise, we use $\|\cdot\|$ to denote the $\ell_2$-norm

\section{Background and Preliminaries}
\label{sec:background_and_preliminaries}

In this section, we define the dynamic mechanism and related notions.  In addition, we discuss three key mechanism design desiderata and their asymptotic versions. Finally, we introduce the general function approximation regime and related assumptions.

{\par \textbf{Episodic MDP.}}
Consider an episodic MDP given by $\cM = \rbr{\cS, \cA, H, \cP, \{r_{i,h}\}_{i = 0, h = 1}^{n, H}}$, where $\cS$ is the state space, $\cA$ is the seller's action space, $H$ is the length of each episode, and $\cP = \{\cP_{h}\}_{h = 1}^H$ is the transition kernel, where $\cP_h(s' | s, a)$ denotes the probability that the state $s \in \cS$ transitions to the state $s' \in \cS$ when the seller chooses the action $a \in \cA$ at the $h$-th step.\footnote{In mechanism design literature the reward function is often called ``value function." We use the tem ``reward function" throughout the paper to avoid confusion with state- and action-value functions.}  We assume that $\cS, \cA$ are both finite but can be arbitrarily large. Let $r_{i, h}: \cS \times \cA \to [0, 1]$ denote the reward function of an agent $i$ at step $h$ and $r_{0, h}: \cS \times \cA \to [- R_{\textrm{max}}, -n + R_{\textrm{max}}]$ the seller's reward function at step $h$, which can be negative, as policies can be costly.

A stochastic policy $\pi = \{\pi_h\}_{h = 1}^H$ maps the seller's state $\cS$ to a distribution over the action space $\cA$ at each step $h$, where $\pi_h(a | s)$ denotes the probability that the seller chooses the action $a \in \cA$ when they are in the state $s \in \cS$. We use $d_{\pi}$ to denote the state-action visitation measure over $\{\cS \times \cA\}^H$ induced by the policy $\pi$ and use $\EE_\pi$ as a shorthand notation for the expectation taken over the visitation measure. 

For any given reward function $r$ and any policy $\pi$, the (state-)value function $V_h^\pi(\cdot; r): \cS \to \RR$ is defined as $V_h^\pi(x; r) = \EE_\pi[\sum_{h' = h}^H r_{h'}(s_{h'}, a_{h'}) | s_h = x]$ at each step $h \in [H]$ and the corresponding action-value function ($Q$-function) $Q_h^\pi(\cdot, \cdot; r): \cS \times \cA \to \RR$ is defined as $Q_h^\pi(x, a; r) = \EE_\pi [\sum_{h' = h}^H r_{h'}(s_{h'}, a_{h'}) | s_h = x, a_h = a]$. For any function $g: \cS \times \cA \to \RR$, any policy $\pi$, and $h \in [H]$, we use the shorthand notation $g(s, \pi_h) = \EE_{a \sim \pi_h(\cdot |s)}[g(s, a)]$. We define the policy-specific Bellman evaluation operator at $h$ with respect to reward function $r$ under policy $\pi$ as
\begin{equation}\label{eqn:defn_bellman_op}
\begin{split}
    (\cT_{h, r}^{\pi} g)(x, a) =& r_h(x, a) +  \EE_{\cP}\sbr{g(s_{h + 1}, \pi_{h + 1}) | s_h = x, a_h = a},
\end{split}
\end{equation}
where $\EE_{\cP}$ is taken over the randomness in the transition kernel $\cP$.

We emphasize that while the problem setting we consider features multiple reward functions and interaction between multiple participants, our setting is not an instance of a Markov game~\citep{littman1994markov} as we allow only the seller to take actions.

{\par \textbf{Dynamic Mechanism as an MDP.}}
We assume that agents and sellers interact in the following way. Without loss of generality, assume that the seller starts at some fixed state $s_{0} \in \cS$ when $h = 1$. For each $h \in [H]$, the seller observes its state $s$ and takes some action $a \in \cA$. The agent receives the reward $r_{i, h}(s, a)$ and reports to the seller the received reward as $\tilde{r}_{i, h}(s_h, a_h) \in [0, 1]$, which may be different from the true reward. The seller receives a reward $r_{0, h}(s, a)$ and transitions to some state $s' \sim \cP_h(\cdot | s, a)$. At the end of each episode, the seller charges each agent $i$ a price $p_i \in \RR$, $i \in [n]$.

We stress the difference between the \emph{reported} reward, $\tilde{r}_{i, h}$, and the \emph{actual} reward, $r_{i, h}$. The reported reward is equal to $r_{i, h}$ if an agent is truthful but may be given by an arbitrary function $\tilde{r}_{i, h}: \cS \times \cA \to [0, 1]$ when the agent is not. In other words, the agent $i$'s reported reward comes from the actual reward function $r_{i, h}$ or some arbitrary reward function $\tilde{r}_{i, h}$. Our algorithm learns a mechanism via the reported rewards and, under certain assumptions, we can provide guarantees on the actual rewards.

For convenience, let $R = \sum_{i = 0}^n r_i$ be the sum of true reward functions and $R_{-i} = \sum_{i' \neq i} r_i$ the sum of true reward functions excluding agent $i$. Let $\tilde{R}$, $\tilde{R}_{-i}$ be defined similarly for the reported reward functions. Let $\cR = \{R_{-i}\}_{i = 1}^n \cup \{R\}$ be the set of all true reward functions that we will estimate and $\tilde{\cR}$ be that for the reported reward functions. When all agents are truthful, $\tilde{\cR} = \cR$. 
We also let 
\begin{align*}
&Q^*_h(\cdot, \cdot; r) = \max_{\pi \in \Pi}Q^\pi_h(\cdot, \cdot; r),
\ 
V^*_h(\cdot; r) = \max_{\pi \in \Pi}V^\pi_h(\cdot; r),
\\
&\pi^{*}_{r} = \argmax_{\pi \in \Pi}V^{\pi}_{1}(s_0; r),
\ \forall r \in \cR \cup \tilde{\cR}.
\end{align*}
As a shorthand notation, let $\pi^* = \pi^*_{R}$, $\pi^*_{-i} = \pi^*_{R_{-i}}$, $\tilde{\pi}^* = \pi^*_{\tilde{R}}$, and $\tilde{\pi}^*_{-i} = \pi^*_{\tilde{R}_{-i}}$.
Following~\citet{kandasamy2020mechanism}, we define the agents' and seller's utilities as follows. For any $i \in [n]$, we define the agent $i$'s utility under policy $\pi$, when charged price $p_i$, as 
\[
    U_i^\pi(p_i) = \EE_{\pi}[\sum_{h = 1}^H r_{i, h}(s_h, a_h)] - p_i = V_1^{\pi}(s_0; r_i) - p_i.
\] The seller's utility is similarly defined as 
\begin{align*}
    U_0^{\pi}(\{p_i\}_{i = 1}^n) &= \EE_{\pi}[\sum_{h = 1}^H r_{0, h}(s_h, a_h)] + \sum_{i = 1}^n p_i = V_1^{\pi}(s_0; r_0) + \sum_{i = 1}^n p_i.
\end{align*}
The social welfare for any policy $\pi \in \Pi$ is the sum of the utilities, $\sum_{i = 0}^n\EE_\pi [u_i] = V_1^\pi(s_0; R)$, similar to its definition in~\citet{bergemann2010dynamic}. 

\subsection{A Dynamic VCG Mechanism}
\label{subsec:a_dynamic_vcg_mechanism}

We now discuss a dynamic adaptation of the VCG mechanism and three key mechanism design desiderata it satisfies~\citep{nisan2007algorithmic}. We begin by introducing the dynamic adaptation of the VCG mechanism.

\begin{definition}[Dynamic VCG Mechanism]
\label{defn:dynamic_vcg_mechanism}
When agents interact according to the aforementioned~MDP, assuming the transition kernel $\cP$ and the reported reward functions $\{\tilde{r}_i\}_{i = 0}^n$ are known, the VCG mechanism selects $\tilde{\pi}^*$, the social welfare maximizing policy based on the reported rewards, and charges the agent $i$ price $p_i: \cS \to \RR$, given by $p_i = V^{*}_1(s_0; \tilde{R}_{-i}) - V^{\tilde{\pi}^*}_1(s_0; \tilde{R}_{-i})$. More generally, when the mechanism chooses to implement some arbitrary policy $\pi$, the VCG price for the agent $i$ is given by 
\begin{equation}
\label{eqn:defn_p_star}
    p_i = V^{*}_1(s_0; \tilde{R}_{-i}) - V^{\pi}_1(s_0; \tilde{R}_{-i}).
\end{equation}
\end{definition} 
Observe that when $H = 1$, the dynamic adaptation we propose reduces to exactly the classic VCG mechanism~\citep{nisan2007algorithmic}.

We highlight the three common mechanism desiderata in the mechanism design literature~\citep{nisan2007algorithmic,bergemann2010dynamic,hartline2012bayesian}.

\vspace{-10pt}
\begin{enumerate} 
\itemsep0em 
    \item \emph{Efficiency:} A mechanism is efficient if it maximizes social welfare when all agents report truthfully.
    \item \emph{Individual rationality:} A mechanism is individually rational if it does not charge an agent more than their reported reward, regardless of other agents' behavior. In other words, if an agent reports truthfully, they attain non-negative utility.
    \item \emph{Truthfulness:} A mechanism is truthful or (dominant strategy) incentive-compatible if, regardless of the truthfulness of other agents' reports, the agent's utility is maximized when they report their rewards truthfully.
\end{enumerate}
\vspace{-10pt}
In the MDP setting, the dynamic VCG mechanism simultaneously satisfies all three desiderata.
\begin{prop}\label{prop:mech_design_desiderata}
    With $\cP$ and the reported rewards $\{\tilde{r}_i\}_{i = 0}^n$ known, choosing $\tilde{\pi}^*$ and charging $p_i$ for all $i \in [n]$ according to~\eqref{eqn:defn_p_star} ensures that the mechanism satisfies truthfulness, individual rationality, and efficiency simultaneously.
\end{prop}
\begin{proof}
    See Appendix~\ref{app:proof_mech_design_desiderata} for a detailed proof.
\end{proof}
{\par \textbf{Performance Metrics.}} We use the following metrics to evaluate the performance of our estimated mechanism. Let the social welfare suboptimality of an arbitrary policy $\pi$ be
\begin{equation}\label{eqn:defn_subopt}
    \textrm{SubOpt}(\pi; s_0) = V_1^{*}(s_0; R) - V_1^{\pi}(s_0; R).
\end{equation}
For any $i \in [n] $, let 
$
    p_i^*(s_0) = V^{*}_1(s_0; R_{-i}) - V^{\pi^*}_1(s_0; R_{-i})
$
be the price charged to the agent $i$ by VCG under truthful reporting. We can similarly define the suboptimality with respect to the agents' and the seller's expected utilities. For any $i \in [n]$, the agent $i$'s suboptimality with respect to policy $\pi$ and price $\{p_i\}_{i = 1}^n$ is defined as
\begin{equation}\label{eqn:defn_subopt_agent}
\begin{split}
    &\textrm{SubOpt}_i(\pi, \{p_i\}_{i = 1}^n; s_0) = U_i^{\pi^*}(p_i^*) - U_i^{\pi}(p_i) = V_1^{\pi^*}(s_0; r_i) - p_i^*(s_0) - V_1^{\pi}(s_0; r_i) + p_i,
\end{split}
\end{equation} and the seller's suboptimality is
\begin{equation}\label{eqn:defn_subopt_seller}
\begin{split}
    \textrm{SubOpt}_0(\pi, \{p_i\}_{i = 1}^n; s_0)  &= U_0^{\pi^*}(\{p_i^*\}_{i = 1}^n) - U_0^{\pi}(\{p_i\}_{i = 1}^n)\\
    &= V_1^{\pi^*}(s_0; r_0)+ \sum_{i = 1}^n p_i^*- V_1^{\pi}(s_0; r_0) -\sum_{i = 1}^n p_i.
\end{split}
\end{equation}

\subsection{Offline Episodic RL with General Function Approximation}

We use offline RL in the general function approximation setting to minimize the aforementioned suboptimalities. Let $\cD$ be a precollected data set that contains $K$ trajectories, that is, $\cD = \{(x_h^\tau, a_h^\tau, \{\tilde{r}_{i, h}^\tau\}_{i = 1}^n, x_{h + 1}^\tau)\}_{h, \tau = 1}^{H, K}$. Following the setup in~\cite{xie2021bellman}, we consider the i.i.d.~data collection regime, where for all $h \in [H]$, $(x_h^\tau, a_h^\tau, x_{h + 1}^\tau)_{\tau = 1}^{K}$ is drawn from a distribution $\mu_h$ supported on $\cS \times \cA \times \cS$. The distribution $\mu$ over $\{\cS \times \cA \times \cS\}^H$ is induced by a behavioral policy used for data collection. We do not make any coverage assumption on $\mu$, similar to the existing literature on offline RL~\citep{jin2021pessimism, uehara2021pessimistic, zanette2021provable}.

Consider some general function class $\cF = \cF_1 \times \cF_2 \times \ldots \times \cF_H$. For each $h \in [H]$, we use some arbitrary yet bounded function class $\cF_h \subseteq \cS \times \cA \to [-(H - h + 1)R_{\textrm{max}}, (H - h + 1)R_{\textrm{max}}]$ to approximate $Q^\pi_h(\cdot, \cdot; r)$ for arbitrary $\pi$ and $r \in \tilde{\cR}$. For completeness, we let $\cF_{H + 1} = \{f: f(s, a) = 0\,\forall(s, a) \in \cS \times \cA\}$ be the singleton set containing only the degenerate function mapping all inputs to 0.

We make two common assumptions about the expressiveness of the function class $\cF$~\citep{antos2008learning,xie2021bellman}.
\begin{assumption}[Approximate Realizability]
\label{assumption:realizability}
  For any $r \in \Tilde{\cR}$ and $\pi \in \{\cS \to \Delta(\cA)\}^H$, there exists some $f_{r}^\pi \in \cF$ such that for all $h \in [H]$,
  \[
    \sup_{\pi' \in \{\cS \to \Delta(\cA)\}^H}\EE_{\pi'_h}\sbr{\|f_{h, r}^\pi(\cdot, \cdot; r) - Q_{h}^\pi(\cdot, \cdot; r)\|^2} \leq \epsilon_{\cF}.
  \] 
\end{assumption}
Intuitively, Assumption~\ref{assumption:realizability} dictates that for all reported reward functions $r$ and all policies $\pi$, there exists a function in $\cF$ that can approximate $Q^{\pi}_r$ sufficiently well.

\begin{assumption}[Approximate Completeness]
\label{assumption:completeness}
  For any $h \in [H], r \in \Tilde{\cR}$, and $\pi \in \{\cS \to \Delta(\cA)\}^H$, we have
  \[
    \sup_{f \in \cF_{h + 1}}\inf_{f' \in \cF_h} \EE_{\mu_h}[\|f' - \cT^{\pi}_{h, r}f\|^2] \leq \epsilon_{\cF, \cF}.
  \]
\end{assumption} 
Assumption~\ref{assumption:completeness} requires the function class $\cF$ to be approximately closed for all reported reward functions and policies. The assumption is prevalent in RL and can be omitted only in rare circumstances~\citep{xie2021batch}.

A fundamental problem in offline RL is the distribution shift, which occurs when the data generating distribution has only a partial coverage of the policy of interest~\citep{jin2021pessimism,zanette2021provable}. We address the issue with the help of distribution shift coefficient~\citep{xie2021bellman}.
\begin{definition}[Distribution Shift Coefficient]
\label{defn:distribution_shift_coef}
  Let $C^\pi(\nu)$ be the measure of distribution shift from an arbitrary distribution over $(\cS \times \cA)^H$, denoted $\nu$, to the data distribution $\mu$, when measured under the transition dynamics induced by a policy $\pi \in \{\cS \to \Delta(\cA)\}^H$. In particular,
  \[
    C^\pi(\nu) = \max_{f^1, f^2 \in \cF}\max_{h \in [H]}\max_{r \in \tilde{\cR}}\frac{\EE_{\nu_h}[\|f^1_h - \cT_{h, r}^{\pi}f^2_{h + 1}\|^2]}{\EE_{\mu_h}[\|f^1_h - \cT_{h, r}^{\pi}f^2_{h + 1}\|^2]}.
  \]
\end{definition} 
The coefficient controls how well the Bellman estimation error shifts from one distribution to another for any Bellman transition operator $\cT$. For a detailed discussion on how the coefficient generalizes previous measures of distribution shift, please refer to~\citet{xie2021bellman}. As a shorthand notation, when $\nu$ is the visitation measure induced by some policy $\pi'$, we let $C^{\pi}(\pi') = C^{\pi}(d_{\pi'}) = C^{\pi}(\nu)$.


In offline learning, with a finite data set, we can only hope to learn the desired mechanism up to certain statistical error. In particular, we state the  approximate versions of the desiderata for finite-sample analysis. 

\vspace{-10pt}
\begin{enumerate}[leftmargin=*,itemsep=0pt,parsep=0pt]
\itemsep0em 
    \item \emph{Asymptotic efficiency:} If all agents report truthfully, a mechanism is asymptotically efficient if $\textrm{SubOpt}(\pi; s_0) \in \cO(K^{-\alpha})$ for some $\alpha \in (0, 1)$.
    \item \emph{Asymptotic individual rationality:} Let $\pi$, $p_i$ be the policy and price chosen by the mechanism when the agent $i$ is truthful. A dynamic mechanism is asymptotically individually rational if $U_i^{\pi}(p_i) = -\cO(K^{-\alpha})$ for some $\alpha \in (0, 1)$, regardless of the truthfulness of other agents.
    \item \emph{Asymptotic truthfulness:} Let $\tilde{\pi}, \tilde{p}_i$ be the policy and price chosen by the mechanism when the agent $i$ is untruthful, and $\pi$, $p_i$ those chosen by the mechanism when the agent $i$ is truthful. We say a dynamic mechanism is asymptotically truthful if $U_i^{\tilde{\pi}}(\tilde{p}_i) - U_i^{\pi}(p_i) = \cO(K^{-\alpha})$ for some $\alpha \in (0, 1)$ regardless of the truthfulness of other agents.
\end{enumerate}
\vspace{-10pt}

As we will see in sequel, we propose a soft policy iteration algorithm that simultaneously satisfies all three criteria above with $\alpha = 1/3$ up to function approximation biases.



\section{Offline RL for VCG}\label{sec:alg}

We develop an algorithm that learns the dynamic VCG mechanism via offline RL. We begin by sketching out a basic outline of our algorithm. Recall the dynamic VCG mechanism given in Definition~\ref{defn:dynamic_vcg_mechanism}. At a high level, an algorithm that learns the dynamic VCG mechanism can be summarized as the following procedure.
\begin{enumerate}
    \item\label{meta_alg:learn_policy} Learn some policy $\check{\pi}$ such that the social welfare suboptimality ${\rm SubOpt}(\check{\pi}; s_0)$ is small.
    \item\label{meta_alg:learn_price} For all $i \in [n]$, estimate the VCG price $p_i$, defined in~\eqref{eqn:defn_p_star}, as $\hat{p}_i = G^{(1)}_{-i}(s_0) - G^{(2)}_{-i}(s_0)$, where $G^{(1)}_{-i}(s_0)$ estimates $V^*_1(s_0; \tilde{R}_{-i})$ and $G^{(2)}_{-i}(s_0)$ estimates $V^{\check{\pi}}_1(s_0; \tilde{R}_{-i})$.
\end{enumerate}
Step~\ref{meta_alg:learn_policy} simply minimizes the social welfare suboptimality using offline RL and has been extensively studied in prior literature~\citep{jin2021pessimism,zanette2021provable,xie2021bellman,uehara2021pessimistic}. 

A greater challenge lies in implementing Step~\ref{meta_alg:learn_price} and showing that the price estimates, $\{\hat{p}_i\}_{i = 1}^n$, satisfy all three approximate mechanism design desiderata. The estimate $G^{(2)}_{-i}(s_0)$ can be constructed by performing a policy evaluation of the learned policy, $\check{\pi}$. The construction of $G^{(1)}_{-i}(s_0)$ is more challenging, involving two separate steps: (1) learning a fictitious policy that approximately maximizes $V^{\pi}_1(s_0; \tilde{R}_{-i})$ over $\pi$ from offline data, and (2) performing a policy evaluation of the learned fictitious policy to obtain the estimate of the value function. Consequently, the policy evaluation and policy improvement subroutines are necessary for learning $G^{(1)}_{-i}(s_0)$ and implementing Step~\ref{meta_alg:learn_price}.

Our challenge is complicated by the fact that a combination of optimism and pessimism is needed for price estimation, whereas the typical offline RL literature only leverages pessimism~\citep{jin2021pessimism, uehara2021pessimistic, xie2021bellman}. For example, when $G^{(1)}_{-i}(s_0)$ is a pessimistic estimate of $V^{*}_1(s_0; \tilde{R}_{-i})$, the price estimate $\hat{p}_i$ is a ``lower bound,'' at least in the first term, of the actual price $p_i$ derived in~\eqref{eqn:defn_p_star}. A lower price estimate would be beneficial to the agent, but would increase the seller's suboptimality since, loosely speaking, the seller is ``paying for'' the uncertainty in the data set, and the reverse holds when $G^{(1)}_{-i}(s_0)$ is an optimistic estimate. The party burdened with the cost of uncertainty may be different in different settings. When allocating public goods, for instance, the cost of uncertainty should be the seller's burden to better benefit the public~\citep{bergemann2019dynamic}, whereas a company wishing to maximize their profit would prefer having the agents ``pay for" uncertainty~\citep{friedman2003pricing}.

To allow for such flexibility, we introduce hyperparameters $\zeta_1, \zeta_2 \in \{\mathtt{PES}, \mathtt{OPT}\}$, where $\zeta_1$ determines whether $G^{(1)}_{-i}(s_0)$ is a $\mathtt{PES}$simistic or $\mathtt{OPT}$imistic estimate and $\zeta_2$ does so for $G^{(2)}_{-i}(s_0)$.~To highlight the trade-off between agents' and seller's suboptimalities, we focus on the two extreme cases, $(\zeta_1, \zeta_2) = (\mathtt{PES}, \mathtt{OPT})$ and $(\zeta_1, \zeta_2) = (\mathtt{OPT}, \mathtt{PES})$, where the former favors the agents and the latter the seller. Depending on the goal of the mechanism designer, different~choices~of~$\zeta_1, \zeta_2$ may be selected to favor agents or the seller~\citep{maskin2008mechanism}.

With the crucial challenges identified, we introduce the specific algorithms that we use to implement Steps~\ref{meta_alg:learn_policy} and~\ref{meta_alg:learn_price}.

\subsection{Policy Evaluation and Soft Policy Iteration}
\label{subsec:policy_evaluation_and_soft_policy_iteration}

We use optimistic and pessimistic variants of soft policy iteration, commonly used for policy improvement~\citep{xie2021bellman, cai2020provably, zanette2021provable}. 
%
%
%
%
%
At a high level, each iteration of the soft policy iteration consists of two steps: policy evaluation and policy improvement. 

We begin by describing our policy evaluation algorithm. The Bellman error can be written as $f_h(s, a) - \cT_{h, r}^\pi f_{h + 1}(s, a)$ for any $(s, a) \in \cS \times \cA$, $h \in [H]$, and the estimate of the action value function $f \in \cF$ for policy $\pi$ and reward $r$. We construct an empirical estimate of the Bellman error as follows. For any $h \in [H]$, $f, f' \in \cF$ and $r \in \Tilde{\cR}$, we define $\cL_{h, r}(f_{h}, f'_{h + 1}, \pi; \cD)$ as
\begin{align*}
  &\cL_{h, r}(f_{h}, f'_{h + 1}, \pi; \cD) = \frac{1}{K}\sum_{\tau = 1}^K (f_{h}(s_h^\tau, a_h^\tau) - r_h(s_h^\tau, a_h^\tau) - f'_{h + 1}(s_{h + 1}^\tau, \pi_{h + 1}))^2,
\end{align*} 
where we slightly abuse the notation and let $r_h^\tau$ be the reported rewards $\tilde{r}_{i, h}^\tau$ summed over $i$ according to the chosen reported reward function $r \in \tilde{\cR}$. Recall that $\tilde{\cR} = \{\tilde{R}_{-i}\}_{i = 1}^n \cup \{\tilde{R}\}$ is the set of reported reward functions whose action-value functions need to be estimated.
The empirical estimate for Bellman error under policy $\pi$ at step $h$ is then constructed as
\begin{equation}\label{eqn:defn_cE}
\begin{split}
  \cE_{h, r}(f, \pi; \cD) = \cL_{h, r}(f_{h}, f_{h + 1}, \pi; \cD)-\min_{g \in \cF_{h}}\cL_{h, r}(g, f_{h + 1}, \pi; \cD).
\end{split}
\end{equation} 
The goal of the policy evaluation algorithm is to solve the following regularized optimization problems:
\begin{equation}\label{eqn:policy_evaluation_regularized_loss}
  \begin{aligned}
    \hat{Q}_{r}^{\pi} &= \argmin_{f \in \cF}-f_{1}(s_{0}, \pi) + \lambda\sum_{h = 1}^{H}\cE_{h, r}(f, \pi; \cD), \\
    \check{Q}_{r}^{\pi} & = \argmin_{f \in \cF}f_{1}(s_{0}, \pi) + \lambda\sum_{h = 1}^{H}\cE_{h, r}(f, \pi; \cD),
  \end{aligned}
\end{equation}
thereby obtaining optimistic and pessimistic estimates of $Q^\pi(\cdot, \cdot; r)$ for any policy $\pi$ and reward function $r$. We summarize the procedure in Algorithm~\ref{alg:policy_evaluation}.

\begin{algorithm}[ht]
\begin{algorithmic}[1]
\caption{Policy Evaluation}\label{alg:policy_evaluation}
\REQUIRE Reported reward $r \in \Tilde{\cR}$, regularization coefficient $\lambda$, dataset $\cD = \{(x_h^\tau, \omega_h^\tau, \{\tilde{r}_{i, h}^\tau\}_{i}^{n})\}_{h, \tau = 1}^{H, K}$, policy $\pi$.
\STATE For all $h, \tau$, calculate $r_h^\tau$ as the sum of $\tilde{r}_{i, h}^\tau$ over $i$ according to the reported reward function $r$.
\STATE Obtain the optimistic and pessimistic estimates of $Q^{\pi}_{r}$ using~\eqref{eqn:policy_evaluation_regularized_loss}
\STATE Return action-value function estimates $\hat{Q}_{r}^{\pi}, \check{Q}_{r}^{\pi}$.
\end{algorithmic}
\end{algorithm}

Next, we introduce the policy improvement procedure. At each step $t \in [T]$, we use the mirror descent with the Kullback-Leibler ($\mathrm{KL})$ divergence to update the policies for all $(s, a) \in \cS \times \cA, h \in [H]$. By direct computation, the update rule can be written as
\begin{align}
  \hat{\pi}^{(t + 1)}_{h, r}(a | s)\, \propto\, \hat{\pi}^{(t)}_{h, r}(a | s)\exp\rbr{\eta \hat{Q}_{h, r}^{(t)}(s, a)} \label{eqn:ospi_policy_update},\\
  \check{\pi}^{(t + 1)}_{h, r}(a | s)\, \propto\, \check{\pi}^{(t)}_{h, r}(a | s)\exp\rbr{\eta \check{Q}_{h, r}^{(t)}(s, a)}\label{eqn:pspi_policy_update},
\end{align} where $\hat{Q}_{h, r}$, $\check{Q}_{h, r}$ are the action-value function estimates obtained from~\eqref{eqn:policy_evaluation_regularized_loss}~\citep{bubeck2014convex,cai2020provably,xie2021bellman}.

For any set of $T$ policies $\{\pi^{(t)}\}_{t = 1}^T$, let $\mathrm{Unif}(\{\pi^{(t)}\}_{t = 1}^T)$ be the mixture policy formed by selecting one of $\{\pi^{(t)}\}_{t = 1}^T$ uniformly at random. The output of our policy improvement algorithm is then given by $\mathrm{Unif}(\{\hat{\pi}_r^{(t)}\}_{t = 1}^T)$ and $\mathrm{Unif}(\{\check{\pi}_r^{(t)}\}_{t = 1}^T)$, that is, the uniform mixture of optimistic and pessimistic policy estimates.
We summarize the soft policy iteration algorithm in the form of pseudocode in Algorithm~\ref{alg:spi}.

\begin{algorithm}[ht]
\begin{algorithmic}[1]
\caption{Soft Policy Iteration for Episodic MDPs}\label{alg:spi}
\REQUIRE Reported reward $r \in \Tilde{\cR}$, regularization coefficient $\lambda$, dataset $\cD = \{(x_h^\tau, \omega_h^\tau, \{\tilde{r}_{i, h}^\tau\}_{i}^{n})\}_{h, \tau = 1}^{H, K}$, number of iterations $T$, learning rate $\eta$.
\STATE Initialize optimistic and pessimistic polices, $\hat{\pi}_{r}^{(1)}$ and $\check{\pi}_{r}^{(1)}$, as the uniform policy.
\FOR{$t = 1, \ldots, T$}
\STATE Obtain the optimistic and pessimistic estimates of $Q^{\hat{\pi}_r^{(t)}}_{r}$ and $Q^{\check{\pi}_r^{(t)}}_{r}$ 
by  Algorithm~\ref{alg:policy_evaluation}.
\STATE Update policy estimates according to~\eqref{eqn:ospi_policy_update} and \eqref{eqn:pspi_policy_update}.
\ENDFOR
\STATE Let $\hat{\pi}_{r}^{\mathrm{out}} = \textrm{Unif}(\{\hat{\pi}_{r}^{(t)}\}_{t = 1}^T)$, $\check{\pi}_{r}^{\mathrm{out}} = \textrm{Unif}(\{\check{\pi}_{r}^{(t)}\}_{t = 1}^T)$.
\STATE Execute   Algorithm~\ref{alg:policy_evaluation} to construct optimistic action-value function $\hat{Q}_{r}^{\mathrm{out}}$  for  $\hat{\pi}_{r}^{\mathrm{out}}$ and pessimistic action-value function $\check{Q}_{r}^{\mathrm{out}}$  for $\check{\pi}_{r}^{\mathrm{out}}$, respectively.
\STATE \textbf{Return} $\{\hat{\pi}_{r}^{\mathrm{out}}, \hat{Q}_{r}^{\mathrm{out}}\}$ and $\{\check{\pi}_{r}^{\mathrm{out}}, \check{Q}_{r}^{\mathrm{out}}\}$.
\end{algorithmic}
\end{algorithm}

We defer the pseudocode of our main algorithm to Appendix~\ref{sec:pseudocode_of_offline_vcg_learn} in the form of Algorithm~\ref{alg:vcg_learn}, as its construction is apparent given the two key subroutines above.


\section{Main Results}
\label{subsec:main_results}

We begin by formally defining the policy class induced by the policy improvement algorithm, Algorithm~\ref{alg:spi}. It is a well-known result that policy iterates induced by mirror descent-style updates in~\eqref{eqn:ospi_policy_update} and~\eqref{eqn:pspi_policy_update} are in the natural policy class attained by soft policy iteration over $\cF$~\citep{cai2020provably, agarwal2021theory, xie2021bellman, zanette2021provable}, given by
\begin{equation*}
\begin{split}
  &\Pi_{\rm It} = \biggl\{\pi'_{h}(\cdot | s) \propto \exp\rbr{\eta \sum_{t = 1}^{T}f^{{t}}_{h}(s, \cdot)}: h \in [H], \{f^{(t)}_{h}\}_{t = 1}^{T} \subseteq \cF_{h}\biggr\}.
\end{split}
\end{equation*} 
Let $\Pi_{\mathrm{SPI}}$ denote the following set of policies
\begin{equation}\label{eqn:spi_policy_class}
\begin{split}
  \Pi_{\mathrm{SPI}} = &\Pi_{\rm It} \Bigl\{\pi: \pi = \mathrm{Unif}(\{\pi^{(t)}\}_{t = 1}^T), \{\pi^{(t)}\}_{t = 1}^T \subset \Pi_{\rm It}\Bigr\}.
\end{split}
\end{equation} 
Before stating the main result, we introduce an additional notation. 
The statistical error $\mathrm{Err}^{\mathrm{stat}}$ denotes
\begin{align*}
    &\mathrm{Err}^{\mathrm{stat}} = \Tilde{\cO}\rbr{H(HR_{\max})^{5/3}K^{-1/3} } + \Tilde{\cO}\biggl(H\rbr{(HR_{\max})^{1/3}\epsilon_{\cF}^{1/3} + \sqrt{\epsilon_{\cF} + \epsilon_{\cF, \cF}}}\biggr)\nonumber,
\end{align*} 
while the optimization error $\mathrm{Err}^{\mathrm{opt}}$ denotes 
\[
    \mathrm{Err}^{\mathrm{opt}} = \tilde{\cO}\rbr{H^2R_{\max}\sqrt{1/T}}.
\] 
To differentiate the policies learned under different truthfulness assumptions, let $\check{\pi} = \check{\pi}_R^{\rm out}$ be the policy chosen by the algorithm when all agents are truthful, let $\tilde{\pi} = \check{\pi}^{\rm out}_{r_i + \tilde{R}_{-i}}$ be the policy chosen when we only assume the agent $i$ is truthful, and let $\check{\pi}_{\tilde{R}} = \check{\pi}_{\tilde{R}}^{\rm out}$ be the policy chosen when no agent is truthful. Let $\check{\pi}^{(t)}, \tilde{\pi}^{(t)}, \check{\pi}_{\tilde{R}}^{(t)}$ be the iterates of Algorithm~\ref{alg:spi} when learning these policies. Denote the prices charged by $\{\hat{p}_i\}_{i = 1}^n, \{\tilde{p}_i\}_{i = 1}^n,$ and $\{\hat{p}_{ i, \tilde{R}}\}_{i = 1}^n$, respectively.

We then summarize the performance of our learned mechanism with asymptotic bounds in Theorem~\ref{thm:main_result_informal}. Theorem~\ref{thm:main_result} presented in Appendix~\ref{sec:proof_of_main_result} provides a more detailed result. 
\begin{theorem}[Informal]\label{thm:main_result_informal}
  With probability at least $1 - \delta$, with suitable choices of $\lambda, \delta$, under Assumptions~\ref{assumption:realizability} and~\ref{assumption:completeness}, the following claims hold simultaneously.
  \vspace{-10pt}
  \begin{enumerate}[leftmargin=*,itemsep=0pt,parsep=0pt]
    \itemsep0em 
    \item Algorithm~\ref{alg:vcg_learn} returns a mechanism that is asymptotically efficient. More specifically, assuming all agents report truthfully, we have
          \begin{align*}
            &\mathrm{SubOpt}(\check{\pi}; s_{0})\leq\mathrm{Err}^{\mathrm{opt}}  + \rbr{\frac{1}{T}\sum_{t = 1}^T\sqrt{C^{\check{\pi}^{(t)}}(\pi^*)}}\mathrm{Err}^{\mathrm{stat}}.
          \end{align*}
    \item Assuming all agents report truthfully, when $(\zeta_{1}, \zeta_{2}) = (\mathtt{PES}, \mathtt{OPT})$, we have
    \begin{align*}
        &\mathrm{SubOpt}_i(\check{\pi}, \{\hat{p}_i\}_{i = 1}^n; s_0)\leq\mathrm{Err}^{\mathrm{opt}} + \rbr{\frac{1}{T}\sum_{t = 1}^T\sqrt{C^{\check{\pi}^{(t)}}(\pi^*)}}\mathrm{Err}^{\mathrm{stat}}.
    \end{align*}
    When $(\zeta_{1}, \zeta_{2}) = (\mathtt{OPT}, \mathtt{PES})$, we have
    \begin{align*}
        &\mathrm{SubOpt}_i(\check{\pi}, \{\hat{p}_i\}_{i = 1}^n; s_0) \leq \mathrm{Err}^{\mathrm{opt}} + \mathrm{Err}^{\mathrm{stat}} \Biggl(\frac{1}{T}\sum_{t = 1}^T\sqrt{C^{\check{\pi}^{(t)}}(\pi^*)} + \sqrt{C^{\hat{\pi}_{-i}}(\hat{\pi}_{-i})} + \sqrt{C^{\check{\pi}}(\check{\pi})}\Biggr).
    \end{align*}
    \item Assuming all agents report truthfully, when $(\zeta_{1}, \zeta_{2}) = (\mathtt{PES}, \mathtt{OPT})$, we have
    \begin{align*}
        &\mathrm{SubOpt}_0(\check{\pi}, \{\hat{p}_i\}_{i = 1}^n; s_0)\\
        &\qquad \leq n \mathrm{Err}^{\mathrm{opt}} + \mathrm{Err}^{\mathrm{stat}} \Biggl(\sum_{i = 1}^n \sqrt{C^{\check{\pi}_{-i}}(\check{\pi}_{-i})}+ n\sqrt{C^{\check{\pi}}(\check{\pi})}  + \sum_{i = 1}^n\frac{1}{T}\sum_{t = 1}^T\sqrt{C^{\check{\pi}^{(t)}_{R_{-i}}}(\pi_{-i}^*)}\Biggr).
    \end{align*}
    and, when $(\zeta_1, \zeta_2) = (\mathtt{OPT}, \mathtt{PES})$, we have 
    \begin{align*}
        &\mathrm{SubOpt}_0(\check{\pi}, \{\hat{p}_i\}_{i = 1}^n; s_0)\\
        &\qquad \leq n \mathrm{Err}^{\mathrm{opt}}\mathrm{Err}^{\mathrm{stat}} \Biggl(\sum_{i = 1}^n\frac{1}{T} \sum_{t = 1}^T\sqrt{C^{\hat{\pi}_{R_{-i}}^{(t)}}(\hat{\pi}_{R_{-i}}^{(t)})}  + \sum_{i = 1}^n\frac{1}{T}\sum_{t = 1}^T\sqrt{C^{\hat{\pi}^{(t)}_{R_{-i}}}(\pi_{-i}^*)}\Biggr).
    \end{align*} 
    \item Algorithm~\ref{alg:vcg_learn} returns a mechanism that is asymptotically individually rational. More specifically, even when other agents are untruthful, when $(\zeta_1, \zeta_2) = (\mathtt{PES}, \mathtt{OPT})$ and the agent $i$ is truthful, their utility satisfies
    \begin{align*}
        &U_i^{\tilde{\pi}}(\tilde{p}_i) \geq - \mathrm{Err}^{\mathrm{opt}} - \mathrm{Err}^{\mathrm{stat}}\Biggl( \frac{1}{T}\sum_{t = 1}^T \sqrt{C^{\check{\pi}_{\tilde{R}_{-i}}^{(t)}}(\tilde{\pi}^*_{-i})} + \sqrt{C^{\check{\pi}_{\tilde{R}_{-i}}^{\rm out}}(\check{\pi}_{\tilde{R}_{-i}}^{\rm out})} +\frac{1}{T}\sum_{t = 1}^T\sqrt{C^{\tilde{\pi}^{(t)}}(\pi^*_{r_i + \tilde{R}_{-i}})} \Biggr).
    \end{align*} 
    and when $(\zeta_1, \zeta_2) = (\mathtt{OPT}, \mathtt{PES})$ and the agent $i$ is truthful, their utility satisfies
    \begin{align*}
        U_i^{\tilde{\pi}}(\tilde{p}_i) \geq &- \mathrm{Err}^{\mathrm{opt}}\\
        &- \mathrm{Err}^{\mathrm{stat}} \Biggl(\frac{1}{T}\sum_{t = 1}^T\sqrt{C^{\tilde{\pi}^{(t)}}(\pi^*_{r_i + \tilde{R}_{-i}}) }+ \sqrt{C^{\hat{\pi}_{\tilde{R}_{-i}}^{(t)}}(\tilde{\pi}^*_{-i})}+ \frac{1}{T}\sum_{t = 1}^T \sqrt{C^{\hat{\pi}^{(t)}_{\tilde{R}_{-i}}}(\hat{\pi}^{(t)}_{\tilde{R}_{-i}})} + \sqrt{C^{\tilde{\pi}}(\tilde{\pi})}\Biggr).
    \end{align*}
    \item Algorithm~\ref{alg:vcg_learn} returns a mechanism that is asymptotically truthful. More specifically, even when all the other agents are untruthful and irrespective of whether the agent $i$ is truthful or not, for all $i \in [n]$ when $\zeta_2 = \mathtt{OPT}$ the amount of utility gained by untruthful reporting is upper bounded as
    \begin{align*}
        &U_{i}^{\check{\pi}_{\tilde{R}}}(\hat{p}_{i, \tilde{R}}) - U_{i}^{\tilde{\pi}}(\tilde{p}_{i})\leq \mathrm{Err}^{\rm opt}+\mathrm{Err}^{\rm stat} \rbr{\frac{1}{T}\sum_{t = 1}^T \sqrt{C^{\tilde{\pi}^{(t)}}(\pi^*_{r_i + \tilde{R}_{-i}})} + \sqrt{C^{\check{\pi}_{\tilde{R}}}(\check{\pi}_{\tilde{R}}))}},
    \end{align*}
    and when $\zeta_2 = \mathtt{PES}$, the amount of utility gained by untruthful reporting is upper bounded as
    \begin{align*}
        &U_{i}^{\check{\pi}_{\tilde{R}}}(\hat{p}_{i, \tilde{R}}) - U_{i}^{\tilde{\pi}}(\tilde{p}_{i})\leq \mathrm{Err}^{\rm opt}+\mathrm{Err}^{\rm stat}\rbr{\frac{1}{T}\sum_{t = 1}^T \sqrt{C^{\tilde{\pi}^{(t)}}(\pi^*_{r_i + \tilde{R}_{-i}})} + \sqrt{C^{\tilde{\pi}}(\tilde{\pi}))}}. 
    \end{align*}
  \end{enumerate}
  \vspace{-10pt}
\end{theorem}
\begin{proof}
  See Appendix~\ref{sec:proof_of_main_result} for a detailed proof.
\end{proof}

We make a few remarks about Theorem~\ref{thm:main_result_informal}.

{\par \textbf{Dependence on the number of trajectories $K$.}} The only term that depends on the number of trajectories $K$ is the statistical error $\mathrm{Err}^{\rm stat}$ and it decays at the $\tilde{\cO}(K^{-1/3})$ rate, matching the sample complexity of the pessimistic soft policy iteration algorithm~\citep{xie2021bellman}. When data set has coverage of the optimal policy and no function approximation bias, our algorithm converges sublinearly to a mechanism with suboptimality $\cO(K^{1/3})$. Furthermore, when data set has sufficient coverage over all policies and the function class satisfies Assumptions~\ref{assumption:realizability} and~\ref{assumption:completeness} exactly, our algorithm is asymptotically individually rational and truthful at the same $\cO(K^{1/3})$ rate, a result that is not implied by the existing literature on offline RL~\citep{xie2021bellman,jin2021pessimism,zanette2021provable}. 

{\par \textbf{Dependence on $\zeta_1, \zeta_2$.}} Observe that $\zeta_1$ and $\zeta_2$ affect the bounds in Theorem~\ref{thm:main_result_informal} by changing the distribution shift coefficients involved for each suboptimality. The inclusion of optimism in offline RL for mechanism design is crucial, as the optimal individual suboptimality rate is attainable only when $\zeta_1 = \mathtt{OPT}$. Different from the existing work on offline RL which extensively uses pessimism, we demonstrate the importance and necessity of optimism when offline RL is used to help design dynamic mechanisms~\citep{xie2021bellman,jin2021bellman,zanette2021provable}.

{\par \textbf{Dependence on $\cF, \Pi_{\rm SPI}$.}} The statistical error term $\textrm{Err}^{\textrm{stat}}$ is the only term that depends on $\cF, \Pi_{\rm SPI}$ through the log covering numbers of $\cF$ and $\Pi_{\rm SPI}$. The covering numbers are formally defined in Appendix~\ref{sec:concentration_analysis} and the theorem's dependence on the covering number is made explicit in the non-asymptotic version, Theorem~\ref{thm:main_result}. We emphasize that our results are directly applicable to general, continuous function classes via a covering-based argument, improving over the results in~\citet{xie2021bellman}.

{\par \textbf{Comparison to related work.}} While deep RL algorithms such as conservative $Q$-learning~\citep{kumar2020conservative}, conservative offline model-based policy optimization~\citep{yu2021combo}, and decision transformer~\citep{chen2021decision} have achieved empirical success on popular offline RL benchmarks, such algorithms rarely have theoretical guarantees without strong coverage assumptions. Within a mechanism design context, such a lack of theoretical guarantees is particularly problematic, as we cannot ensure that the learned mechanism is individually rational or truthful, potentially leading to significant ethical issues when applied to real-world problems. When compared to~\citet{xie2021bellman}, our work features a streamlined, simplified theoretical analysis, which we sketch below, that is directly applicable when both $|\cF|$ and $|\Pi|$ are unbounded using a covering-based argument, whereas the convergence bounds in~\citet{xie2021bellman} grows linearly in the term $\sqrt{\frac{\log |\cF||\Pi|/\delta}{K}}$ in the general function approximation setting.


\section{Proof Sketch}
To prove the results in Theorem~\ref{thm:main_result_informal}, we need to first analyze the concentration properties of the empirical Bellman error estimate, $B_{h, r}(f, \pi; \cD)$. As the function approximation class $\cF$ and the policy class $\Pi$ often contains infinite elements, it is crucial that the tail bounds we obtain remain finite even when both $|\cF|$ and $|\Pi|$ are infinite.

We begin by sketching out the concentration bounds for $B_{h, r}(f, \pi, \cD)$. Consider some arbitrary and fixed $h \in [H]$ and $r \in \Tilde{\cR}$. Let $Z$ be the random vector $(s_{h}, a_{h}, r_{h}(s_h, a_h), s_{h + 1})$, where $(s_{h}, a_{h}, s_{h + 1}) \sim \mu_{h}$ and $Z_j$ its realization for any $j \in [K]$ drawn independently from $\cD_{h}$. For any $f, f' \in \cF$, and $\pi \in \Pi$, we further define the random variable
\begin{equation}
\label{eqn:defn_g}
\begin{split}
      g_{f, f'}^{\pi}(Z) = &(f_{h}(s_{h}, a_{h}) - r_{h} - f'_{h +  1}(s_{h + 1}, \pi_{h + 1}))^{2}\\
      &- (\cT_{h, r}^{\pi}f'_{h + 1}(s_{h}, a_{h}) - r_{h} - f'_{h +  1}(s_{h + 1}, \pi_{h + 1}))^{2},
\end{split}
\end{equation} and $g_{f, f'}^{\pi}(Z_{j})$ its empirical counterpart evaluated on $Z$'s realization, $Z_{j}$. Recalling the definition of the Bellman transition operator $\cT_{h, r}^\pi$, we can show that
\[
\EE_{Z \sim \mu_h}[g_{f, f'}^{\pi}(Z)] = \|f_h - \cT_{h, r}^\pi f'_{h + 1}\|_{2, \mu_h}^2.
\] The boundedness of functions in $\cF$ and reward functions $r \in \Tilde{\cR}$ ensure that
\[
\Var(g_{f, f'}^\pi(Z)) \leq 16H^2 R_{\max}^2 \|f_h - \cT_{h, r}^\pi f'_{h + 1}\|_{2, \mu_h}^2.
\] With both the expectation and variance bounded, we can derive a tail bound for the realizations $g_{f, f'}^\pi (Z_j)$, thereby ensuring $\frac{1}{K}\sum_{j = 1}^K g_{f, f'}^\pi(Z_j)$ is sufficiently close to $\|f_h - \cT_{h, r}^\pi f'_{h + 1}\|_{2, \mu_h}^2$ for a specific choice of $f, f' \in \cF$ and $\pi \in \Pi$. 

We then focus on the function $g_{f, f'}^\pi$ itself. Let $\cG_{\cF, \Pi} = \{g_{f_h, f_{h + 1}'}^\pi: f, f' \in \cF, \pi \in \Pi\}$. Examining the definition of $g_{f, f'}^\pi(Z)$ in~\eqref{eqn:defn_g}, we can directly control the covering number of $\cG_{\cF, \Pi}$ using covering numbers of $\cF, \Pi$, more formally introduced in Appendix~\ref{sec:pseudocode_of_offline_vcg_learn}. Using a standard covering argument, we obtain a tail bound for $g_{f, f'}^\pi(Z)$ for all possible choices of $f, f' \in \cF$ and $\pi \in \Pi$, even when both $\cF$ and $\Pi$ are infinite, via the covering numbers of $\cF$ and $\Pi$.

Finally, we notice that $\frac{1}{K}\sum_{j = 1}^K g_{f, f'}^\pi(Z_j)$ is close to $B_{h, r}(f, \pi; \cD)$ under Assumptions~\ref{assumption:realizability} and~\ref{assumption:completeness}, linking the concentration behavior of $\frac{1}{K}\sum_{j = 1}^K g_{f, f'}^\pi(Z_j)$ to the empirical losses $B_{h, r}(f, \pi; \cD)$ we observe.

\subsection{Seller Suboptimality}
We now sketch the proof for bounding the seller's optimality to provide some intuition on how to prove Theorem~\ref{thm:main_result_informal}. 
Equation~\eqref{eqn:seller_subopt_decomp}, given in the appendix,
bounds $\mathrm{SubOpt}_0(\check{\pi}, \{\hat{p}_i\}_{i = 1}^n; s_0)$
as 
\begin{align*}
    \mathrm{SubOpt}_0(\check{\pi}, \{\hat{p}_i\}_{i = 1}^n; s_0)
    & \leq  \sum_{i = 1}^n\rbr{ V_1^{\pi^{*}_{-i}}(s_0; R_{-i}) - G^{(1)}_{-i}(s_{0})}+ \sum_{i=1}^n \rbr{ G^{(2)}_{-i}(s_{0}) -V_{1}^{\check{\pi}}(s_{0}, R_{-i}) }.
\end{align*} 

The second term corresponds to the error bound of Algorithm~\ref{alg:policy_evaluation}. When $\zeta_2 = \mathtt{OPT}$, the term exactly corresponds to the classic function evaluation error of the upper confidence bound methods. As such, it can be bounded using a combination of the distribution shift coefficient $C^{\check{\pi}}(\check{\pi})$ and the fact that $\hat{Q}^{\check{\pi}}_{R_{-i}}$ minimizes~\eqref{eqn:policy_evaluation_regularized_loss}. When $\zeta_2 = \mathtt{PES}$, we bound the term using the fact that the output of our policy evaluation algorithm is approximately pessimistic, similar to Lemma C.6 of~\citet{xie2021bellman}.

Next, we focus on the first term $G_{-i}^{(1)}(s_0) - V_{1}^{\pi_{-i}^{*}}(s_0; R_{-i})$. When $\zeta_1 = \mathtt{OPT}$, we use the following decomposition
\begin{align*}
    G_{-i}^{(1)}(s_0) - V_{1}^{\pi_{-i}^{*}}(s_0; R_{-i}) = &V^{\pi_{-i}^*}_1(s_0; R_{-i}) - \frac{1}{T}\sum_{t =1 }^T\hat{Q}^{(t)}_{1, R_{-i}}(s_0, \hat{\pi}^{(t)}_{1, R_{-i}})\\
    &+ \frac{1}{T}\sum_{t = 1}^T\rbr{\hat{Q}^{(t)}_{1, R_{-i}}(s_0, \hat{\pi}^{(t)}_{1, R_{-i}}) - V_1^{\hat{\pi}_{R_{-i}}^{(t)}}(s_0; R_{-i})}\\
    &+ V_1^{\hat{\pi}_{-i}}(s_0; R_{-i}) - \hat{Q}^{\rm out}_{1, R_{-i}}(s_0, \hat{\pi}_{1, -i}).
\end{align*} 
The first term can be bounded using the properties of mirror descent~\citep{bubeck2014convex}. The latter two terms are function evaluation errors, which we can bound in a similar way as $G^{(2)}_{-i}(s_{0}) -V_{1}^{\check{\pi}}(s_{0}, R_{-i})$. 
The first term can be similarly bounded when $\zeta_1 = \mathtt{PES}$,
completing the proof sketch.

\section{Discussion}
Our work provides the first algorithm that can provably learn the dynamic VCG mechanism with no prior knowledge, where the learned mechanism is asymptotically efficient, individually rational, and truthful. For future work, we aim to study the performance of our algorithm when the training set is corrupted with untruthful reports.

\section{Acknowledgements}
Zhaoran Wang acknowledges National Science Foundation
(Awards 2048075, 2008827, 2015568, 1934931), Simons
Institute (Theory of Reinforcement Learning), Amazon, J.P.
Morgan, and Two Sigma for their supports. Mladen Kolar acknowledges
the William S. Fishman Faculty Research Fund
at the University of Chicago Booth School of
Business. Zhuoran Yang
acknowledges Simons Institute (Theory of Reinforcement
Learning).


\bibliographystyle{ims}
\bibliography{ref}

\newpage
\appendix
\section{Table of Notation} \label{sec:tab_notation}

The following table summarizes the notation used in the paper.


\begin{table}[!h]
\centering
\renewcommand*{\arraystretch}{1.5}
\begin{tabular}{ >{\centering\arraybackslash}m{2.5cm} | >{\centering\arraybackslash}m{12cm} } 
\hline\hline
Notation & Meaning \\ 
\hline
$r_{i, h} / \Tilde{r}_{i,h}$ & actual / reported reward function for agent $i$ at step $h \in [H]$\\
$R_{-i, h}/(\Tilde{R}_{-i, h})$ & actual / reported sum of reward function across all participants sans agent $i$\\
$R_h / (\Tilde{R}_h)$ & actual / reported sum of reward functions across all participants\\
$\cR / \Tilde{\cR}$ & actual / reported reward functions of interest. \\
$\pi_h$ & the policy taken by the seller at step $h \in [H]$\\
$\cT_{h, r}^{\pi}$ & policy specific Bellman transition operator\\
$C^{\pi}(\nu)$ & Distribution shift coefficient (see Definition~\ref{defn:distribution_shift_coef})\\
$C^{\pi_1}(\pi_2)$ & Shorthand notation for $C^{\pi_1}(d_{\pi_2})$\\
\hline
$\hat{\pi}^{(t)}_{h, r} / (\check{\pi}^{(t)}_{h, r})$ & optimistic / pessimistic policy estimate at the $t$-th iteration of Algorithm~\ref{alg:spi} with input $r \in \Tilde{\cR}$ \\
$\hat{Q}^{(t)}_{h, r} / (\check{Q}^{(t)}_{h, r})$ & optimistic / pessimistic action-value function estimate at the $t$-th iteration of Algorithm~\ref{alg:spi} with input $r \in \Tilde{\cR}$. Shorthand for $\hat{Q}^{\hat{\pi}^{(t)}_{h, r}}_{h, r}\, (\check{Q}^{\check{\pi}^{(t)}_{h, r}}_{h, r})$\\
$\hat{\pi}^{\rm out}_{h, r} / (\check{\pi}^{\rm out}_{h, r})$ & optimistic / pessimistic policy output of Algorithm~\ref{alg:spi} with input $r \in \Tilde{\cR}$ \\
$\hat{Q}^{\rm out}_{h, r} / (\check{Q}^{\rm out}_{h, r})$ & optimistic / pessimistic action-value function estimate output of Algorithm~\ref{alg:spi} with input $r \in \Tilde{\cR}$. Shorthand for $\hat{Q}^{\hat{\pi}^{\rm out}_{h, r}}_{h, r}\, (\check{Q}^{\check{\pi}^{\rm out}_{h, r}}_{h, r})$\\
\hline \hline
\end{tabular}
\label{tab:notation}
\end{table}
\section{Proof of Mechanism Design Desiderata (Proposition~\ref{prop:mech_design_desiderata})}
\label{app:proof_mech_design_desiderata}

Those familiar with the literature on mechanism design may quickly realize that our price function is derived using the Clarke pivot rule~\citep{nisan2007algorithmic}. The result is directly derived from the properties of the VCG mechanism~\citep{nisan2007algorithmic,parkes2007online, hartline2012bayesian}. We include a full proof for completeness.

With $\cP$ and $\{\tilde{r}_i\}_{i = 0}^n$ given, the state-value functions $V_h^\pi(s_0, r)$ can be explicitly calculated for all $h \in [H], r \in \tilde{\cR}$. We can then obtain exactly $\tilde{\pi}^*$ and directly calculate $p_i = V_1^{*}(s_0, \tilde{R}_{-i}) - V_1^{\tilde{\pi}^*}(s_0, \tilde{R}_{-i})$. Thus, the proposed mechanism is feasible when the rewards and transition kernel are known.

For convenience, let 
\[
\pi^{(1)} = \pi_{r_i + \tilde{R}_{-i}}^* = \argmax_{\pi \in \Pi} V_1^\pi(s_0; r_i + \tilde{R}_{-i})
\quad\text{and}\quad
\pi^{(2)} = \pi_{\tilde{R}}^* = \argmax_{\pi \in \Pi}V_1^\pi(s_0; \Tilde{R}),
\] 
denote the policies chosen by the mechanism when the agent $i$ is truthful and untruthful, respectively, without assumptions on the truthfulness of other agents.

We now show that the three desiderata are satisfied by the mechanism.
\begin{enumerate}
    \item Efficiency. When the agents report $\{r_i\}_{i = 1}^n$ truthfully, the chosen policy $\pi^*$ maximizes the social welfare and is efficient by definition.

    \item Individual rationality. The price charged from the agent $i$ is
    \[
        p_i = V_1^{*}(s_0; \Tilde{R}_{-i}) - V_1^{\pi^{(2)}}(s_0; \Tilde{R}_{-i}).
    \] 
    Our goal is to then show that $V_1^{\pi^{(2)}}(s_0; \tR_i) \geq p_i$.
    That is, the value function of the \emph{reported} reward is no less than the price charged. Observe that
    \begin{align*}
        V_1^{\pi^{(2)}}(s_0; \tR_i) - \tilde{p}_i = V_1^{\pi^{(2)}}(s_0; \Tilde{R}) - V_1^{*}(s_0; \Tilde{R}_{-i}).
    \end{align*} 
    Let $\pi^{(2)}_{-i} = \argmax_{\pi \in \Pi}V_1^{\pi}(s_0; \Tilde{R}_{-i})$. Then we know that
    \[
        V_1^{\pi^{(2)}}(s_0; \tR_i) - \tilde{p}_i \geq V_1^{\pi^{(2)}_{-i}}(s_0; \Tilde{R}) - V_1^{\pi^{(2)}_{-i}}(s_0; \Tilde{R}_{-i}) = V_1^{\pi^{(2)}_{-i}}(s_0; \tR_i) \geq 0.
    \]
    
    \item Truthfulness: 
    If $\Tilde{r}_i = r_i$, that is, the agent $i$ reports truthfully, they attain the following utility
    \[
        U_i^{\pi^{(1)}}(p_i) = V_1^{\pi^{(1)}}(s_0; r_i)  - V_1^{*}(s_0; \Tilde{R}_{-i}) + V_1^{\pi^{(1)}}(s_0; \Tilde{R}_{-i}) = V_1^{\pi^{(1)}}(s_0; r_i + \tilde{R}_{-i}) - V_1^{*}(s_0; \Tilde{R}_{-i}).
    \]
    When the agent reports some arbitrary $\tilde{r}_i$, the agent receives the following utility instead
    \[
        U_i^{\pi^{(2)}}(p_i) = V_1^{\pi^{(2)}}(s_0; r_i)  - V_1^{*}(s_0; \Tilde{R}_{-i}) + V_1^{\pi^{(2)}}(s_0; \Tilde{R}_{-i}) = V_1^{\pi^{(2)}}(s_0; r_i + \tilde{R}_{-i}) - V_1^{*}(s_0; \Tilde{R}_{-i}).
    \]
    Since $\pi^{(1)}$ maximizes $V_1^\pi(s_0; r_i + \tilde{R}_{-i})$, $u_i \geq \tilde{u}_i$ regardless of other agents' reported reward $\{\tR_j\}_{j \neq i}$ and the mechanism is truthful.
\end{enumerate}




\section{Pseudocode for Offline VCG Learn}
\label{sec:pseudocode_of_offline_vcg_learn}

Let $\cN_{\infty}(\epsilon, \cF)$ be the \(\epsilon\)-covering number of $\cF$ with respect to the $\ell_{\infty}$-norm, that is, the cardinality of the smallest set of functions $\{f^{l}\}_{l = 1}^{N_{L}}$ such that for all $f \in \cF$ there exists some $l \in [L]$ such that
\[
  \max_{h \in [H]}\sup_{s \in \cS, a \in \cA}|f^{l}_{h}(s, a) - f_{h}(s, a)| \leq \epsilon.
\] 
We also let $\cN_{\infty, 1}(\epsilon, \Pi)$ be the \(\epsilon\)-covering number of $\Pi$ with respect to the following norm:
\[
  \ell_{\infty, 1}(\pi - \pi') = \sup_{h \in [H], s\in \cS}\sum_{a \in \cA}|\pi_{h}(a | s) - \pi_{h}'(a | s)|.
\]
With the covering numbers defined, we introduce the main algorithm and the parameter choices for the algorithm, which depend on the covering numbers. For the main algorithm, we set
\begin{equation}\label{eqn:param_setting}
    \lambda = \rbr{\frac{R_{\max}}{H^2(\epsilon_{\rm S} + 3\epsilon_{\cF})^2}}^{1/3}, \quad \eta = \sqrt{\frac{\log |\cA|}{2H^2R_{\max}^2T}},
\end{equation}
where
\[
    \epsilon_{\rm S}= \frac{5136}{K}H^{4}R^{4}_{\max}\log \bigg (   56nH \cdot \cN_{\infty}\rbr{\frac{19H^{3}R^{3}_{\max}}{K}, \cF} \cdot \cN_{\infty, 1}\rbr{\frac{19H^{4}R^{4}_{\max}}{K}, \Pi_{\rm SPI}} \Big/ \delta \biggr ) .
\]
The pseudocode for our main algorithm can then be summarized as Algorithm~\ref{alg:vcg_learn}.

\begin{algorithm}[ht!]
\begin{algorithmic}[1]
\caption{Offline VCG Learn}\label{alg:vcg_learn}
\REQUIRE Hyperparameters $\zeta_{1}, \zeta_{2 }\in \{\tt OPT, PES\}$, regularization coefficient $\lambda$, number of iterations $T$, learning rate $\eta$.
\STATE Let $\check{\pi}_{\tilde{R}}^{\rm out}$ be the pessimistic policy output of Algorithm~\ref{alg:spi} with $r = \tilde{R}$, $T$, and $\lambda$, $\eta$ set according to~\eqref{eqn:param_setting}.\label{algline:learn_policy_for_welfare}
\FOR {Agent $i = 1, 2, \dots, n$}
\STATE Call Algorithm~\ref{alg:spi} with $r = \tilde{R}_{-i}$, $T$, and $\lambda$, $\eta$ set according to~\eqref{eqn:param_setting}.
\STATE If $\zeta_1 = \mathtt{OPT}$, let $G_{-i}^{(1)}(s_{0}) = \hat{Q}^{\mathrm{out}}_{1, \tilde{R}_{-i}}(s_{0}, \hat{\pi}_{1, \tilde{R}_{-i}}^{\mathrm{out}})$. Otherwise let $G_{-i}^{(1)}(s_{0}) = \check{Q}^{\mathrm{out}}_{1, \tilde{R}_{-i}}(s_{0},\check{\pi}_{1, \tilde{R}_{-i}}^{\mathrm{out}})$.
\STATE Call Algorithm~\ref{alg:policy_evaluation} with $r = \tilde{R}_{-i}$, $\pi = \check{\pi}_{\tilde{R}}^{\rm out}$, and $\lambda$ set according to~\eqref{eqn:param_setting}.
\STATE If $\zeta_2 = \mathtt{OPT}$, let $G_{-i}^{(2)}(s_{0}) = \hat{Q}^{\check{\pi}_{\tilde{R}}^{\rm out}}_{1, \tilde{R}_{-i}}(s_{0}, \check{\pi}_{1,\tilde{R}}^{\rm out})$.\\
Otherwise let $G_{-i}^{(2)}(s_{0}) = \check{Q}^{\check{\pi}_{\tilde{R}}^{\rm out}}_{1, \tilde{R}_{-i}}(s_{0}, \check{\pi}_{1,\tilde{R}}^{\rm out})$.
\STATE Set the estimated price $\hat{p}_i = G^{(1)}_{-i}(s_{0}) - G^{(2)}_{-i}(s_{0})$.\label{algline:vcg_pricing}
\ENDFOR
\STATE Return policy $\check{\pi}_{\tilde{R}}^{\rm out}$ and estimated prices $\{\hat{p}_i\}_{i = 1}^n$.
\end{algorithmic}
\end{algorithm}

\section{Proof of Theorem~\ref{thm:main_result_informal}}
\label{sec:proof_of_main_result}

We re-state Theorem~\ref{thm:main_result_informal} in a finite sample form.
\begin{theorem}[Theorem~\ref{thm:main_result_informal} restated]
\label{thm:main_result}
Suppose that $\lambda, \eta$ are set according to~\eqref{eqn:param_setting} and Assumptions~\ref{assumption:realizability} and~\ref{assumption:completeness} hold.
Then, with probability at least $1 - \delta$, the following holds simultaneously.
  \vspace{-10pt}
  \begin{enumerate}
  \itemsep0em
      \item Assuming all agents report truthfully, the suboptimality of the output policy $\check{\pi}$ is bounded as
        \begin{align*}
          &{\rm SubOpt}(\check{\pi}; s_0) \leq 2H^2R_{\max}\sqrt{\frac{2\log|\cA|}{T}} + \sqrt{\epsilon_{\cF}} + 2(HR_{\max})^{1/3}(\epsilon_{\rm S} + 3\epsilon_{\cF})^{1/3}\\
          &\qquad + H\rbr{\frac{1}{T}\sum_{t = 1}^T\sqrt{C^{\check{\pi}^{(t)}}({\pi^*})}}\rbr{2(HR_{\max})^{1/3}(\epsilon_{\rm S} + 3\epsilon_{\cF})^{1/3} +\sqrt{ 8\epsilon_{\rm S} + 12 \epsilon_{\cF} + 3\epsilon_{\cF, \cF}}}.
        \end{align*}
      \item Assuming all agents report truthfully, when $(\zeta_1, \zeta_2) = (\mathtt{PES}, \mathtt{OPT})$, the agent $i$'s suboptimality, for all $i \in [n]$, satisfies
      \begin{align*}
          &{\rm SubOpt}_i(\check{\pi}, \{\hat{p}_i\}_{i =1}^n; s_0) \leq 2H^2R_{\max}\sqrt{\frac{2\log|\cA|}{T}} + 3\sqrt{\epsilon_{\cF}} + 6(HR_{\max})^{1/3}(\epsilon_{\rm S} + 3\epsilon_{\cF})^{1/3}\\
          &\qquad + H\rbr{\frac{1}{T}\sum_{t = 1}^T\sqrt{C^{\check{\pi}^{(t)}}({\pi^*})}}\rbr{2(HR_{\max})^{1/3}(\epsilon_{\rm S} + 3\epsilon_{\cF})^{1/3} +\sqrt{ 8\epsilon_{\rm S} + 12 \epsilon_{\cF} + 3\epsilon_{\cF, \cF}}},
      \end{align*} 
      and when $(\zeta_1, \zeta_2) = (\mathtt{OPT}, \mathtt{PES})$,the agent $i$'s suboptimality, for all $i \in [n]$, satisfies 
      \begin{align*}
          &{\rm SubOpt}_i(\check{\pi}, \{\hat{p}_i\}_{i =1}^n; s_0) \leq  2H^2R_{\max}\sqrt{\frac{2\log|\cA|}{T}} + \sqrt{\epsilon_{\cF}} + 2(HR_{\max})^{1/3}(\epsilon_{\rm S} + 3\epsilon_{\cF})^{1/3}\\
          &\qquad + H\rbr{\frac{1}{T}\sum_{t = 1}^T\sqrt{C^{\check{\pi}^{(t)}}({\pi^*})} + \sqrt{C^{\hat{\pi}_{-i}}({\hat{\pi}_{-i}})} + \sqrt{C^{\check{\pi}}(\check{\pi})}}\\
          &\hspace{3em} \times \rbr{2(HR_{\max})^{1/3}(\epsilon_{\rm S} + 3\epsilon_{\cF})^{1/3} +\sqrt{ 8\epsilon_{\rm S} + 12 \epsilon_{\cF} + 3\epsilon_{\cF, \cF}}}.
      \end{align*}
      \item Assuming all agents report truthfully, when $(\zeta_1, \zeta_2) = (\mathtt{PES}, \mathtt{OPT})$, the seller's suboptimality satisfies 
      \begin{align*}
          &{\rm SubOpt}_0(\check{\pi}, \{\hat{p}_i\}_{i = 1}^n; s_0) \leq 2nH^{2}R_{\max}\sqrt{\frac{2\log |\cA|}{T}} + n\sqrt{\epsilon_{\cF}} + 2n(HR_{\max})^{1/3}(\epsilon_{\rm S} + 3\epsilon_{\cF})^{1/3}\\
          &\hspace{2em} + H\rbr{\sum_{i = 1}^n\rbr{\sqrt{C^{\check{\pi}_{-i}}(\check{\pi}_{-i})}  + \frac{1}{T}\sum_{t = 1}^T  \sqrt{C^{\check{\pi}^{(t)}_{R_{-i}}}(\pi^*_{-i})}} + n\sqrt{C^{\check{\pi}}(\check{\pi})}}\\
          &\hspace{3em}\times\rbr{2(HR_{\max})^{1/3}(\epsilon_{\rm S} + 3\epsilon_{\cF})^{1/3} + \sqrt{8\epsilon_{\rm S} + 12\epsilon_{\cF} + 3\epsilon_{\cF, \cF}}},
      \end{align*} 
      and when $(\zeta_1, \zeta_2) = (\mathtt{OPT}, \mathtt{PES})$, the seller's suboptimality satisfies
      \begin{align*}
          &{\rm SubOpt}_0(\check{\pi}, \{\hat{p}_i\}_{i = 1}^n; s_0) \leq 2nH^{2}R_{\max}\sqrt{\frac{2\log |\cA|}{T}} + 2n\sqrt{\epsilon_{\cF}} + 4n(HR_{\max})^{1/3}(\epsilon_{\rm S} + 3\epsilon_{\cF})^{1/3} \\
          &\qquad + H\rbr{\sum_{i = 1}^n\frac{1}{T}\sum_{t = 1}^T\rbr{\sqrt{C^{\hat{\pi}^{(t)}_{R_{-i}}}(\pi_{-i}^*)} + \sqrt{C^{\hat{\pi}^{(t)}_{R_{-i}}}(\hat{\pi}^{(t)}_{R_{-i}})}}}\\
          &\hspace{3em}\times\rbr{2(HR_{\max})^{1/3}(\epsilon_{\rm S} + 3\epsilon_{\cF})^{1/3} + \sqrt{8\epsilon_{\rm S} + 12\epsilon_{\cF} + 3\epsilon_{\cF, \cF}}}.
      \end{align*}
    \item (Asymptotic Individual Rationality) Even when other agents are untruthful, when $(\zeta_1, \zeta_2) = (\mathtt{PES}, \mathtt{OPT})$ and the agent $i$ is truthful, their utility is lower bounded by
      \begin{align*}
          U_{i}^{\tilde{\pi}}(\tilde{p}_{i}) &\geq -4H^{2}R_{\max}\sqrt{\frac{2\log |\cA|}{T}} - 3\sqrt{\epsilon_{\cF}} - 6(HR_{\max})^{1/3}(\epsilon_{\rm S} + 3\epsilon_{\cF})^{1/3} \\
          &\quad - H\rbr{\frac{1}{T}\sum_{t = 1}^T \rbr{\sqrt{C^{\tilde{\pi}^{(t)}}(\pi^*_{r_i + \tilde{R}_{-i}})} + \sqrt{C^{\check{\pi}^{(t)}_{\tilde{R}_{-i}}}(\tilde{\pi}_{-i}^*)}} + \sqrt{C^{\check{\pi}^{\rm out}_{\tilde{R}_{-i}}}(\check{\pi}^{\rm out}_{\tilde{R}_{-i}})}}\\
          &\qquad \times\rbr{2(HR_{\max})^{1/3}(\epsilon_{\rm S} + 3\epsilon_{\cF})^{1/3} + \sqrt{8\epsilon_{\rm S} + 12\epsilon_{\cF} + 3\epsilon_{\cF, \cF}}},
      \end{align*} 
      and when $(\zeta_1, \zeta_2) = (\mathtt{OPT}, \mathtt{PES})$, their utility is lower bounded by
      \begin{align*}
          U_{i}^{\tilde{\pi}}(\tilde{p}_{i}) &\geq -4H^{2}R_{\max}\sqrt{\frac{2\log |\cA|}{T}} -2 \sqrt{\epsilon_{\cF}} - 4(HR_{\max})^{1/3}(\epsilon_{\rm S} + 3\epsilon_{\cF})^{1/3}\\
          &\quad - H\rbr{\frac{1}{T}\sum_{t = 1}^T\rbr{\sqrt{C^{\tilde{\pi}^{(t)}}(\pi^*_{r_i + \tilde{R}_{-i}})} + \sqrt{C^{\hat{\pi}_{\tilde{R}_{-i}}^{(t)}}(\tilde{\pi}_{-i}^*)} + \sqrt{C^{\hat{\pi}_{\tilde{R}_{-i}}^{(t)}}(\hat{\pi}_{\tilde{R}_{-i}}^{(t)})}} + \sqrt{C^{\tilde{\pi}}(\tilde{\pi})}}\\
          &\qquad \times \rbr{2(HR_{\max})^{1/3}(\epsilon_{\rm S} + 3\epsilon_{\cF})^{1/3} + \sqrt{8\epsilon_{\rm S} + 12\epsilon_{\cF} + 3\epsilon_{\cF, \cF}}}.
      \end{align*}
    \item (Asymptotic Truthfulness)
    Even when all the other agents are untruthful and irrespective of whether the agent $i$ is truthful or not, when $\zeta_2 = \mathtt{OPT}$, the amount of utility gained by untruthful reporting is upper bounded by
      \begin{align*}
          &U_{i}^{\check{\pi}_{\tilde{R}}}(\hat{p}_{i, \tilde{R}}) - U_{i}^{\tilde{\pi}}(\tilde{p}_{i}) \leq 2H^{2}R_{\max}\sqrt{\frac{2\log|\cA|}{T}} +2\sqrt{\epsilon_{\cF}} + 4(HR_{\max})^{1/3}(\epsilon_{\rm S} + 3\epsilon_{\cF})^{1/3} \\
          &\hspace{2em} + H\rbr{\frac{1}{T}\sum_{t=1}^T\sqrt{C^{\tilde{\pi}^{(t)}}(\pi^*_{r_i + \tilde{R}_{-i}})}+\sqrt{C^{\check{\pi}_{\tilde{R}}}(\check{\pi}_{\tilde{R}})}}\\
          &\hspace{2em} \times \rbr{2(HR_{\max})^{1/3}(\epsilon_{\rm S} + 3\epsilon_{\cF})^{1/3} + \sqrt{8\epsilon_{\rm S} + 12\epsilon_{\cF} + 3\epsilon_{\cF, \cF}}},
      \end{align*} 
      and when $\zeta_2 = \mathtt{PES}$, the amount of utility gained by untruthful reporting is upper bounded by
      \begin{align*}
          &U_{i}^{\check{\pi}_{\tilde{R}}}(\hat{p}_{i, \tilde{R}}) - U_{i}^{\tilde{\pi}}(\tilde{p}_{i}) \leq 2H^{2}R_{\max}\sqrt{\frac{2\log|\cA|}{T}} +2\sqrt{\epsilon_{\cF}} + 4(HR_{\max})^{1/3}(\epsilon_{\rm S} + 3\epsilon_{\cF})^{1/3} \\
          &\hspace{2em} + H\rbr{\frac{1}{T}\sum_{t=1}^T\sqrt{C^{\tilde{\pi}^{(t)}}(\pi^*_{r_i + \tilde{R}_{-i}})}+\sqrt{C^{\tilde{\pi}}(\tilde{\pi})}}\\
          &\hspace{2em} \times \rbr{2(HR_{\max})^{1/3}(\epsilon_{\rm S} + 3\epsilon_{\cF})^{1/3} + \sqrt{8\epsilon_{\rm S} + 12\epsilon_{\cF} + 3\epsilon_{\cF, \cF}}}.
      \end{align*}
  \end{enumerate}
\end{theorem}

\begin{proof}[Proof of Theorem~\ref{thm:main_result}]

We will make use of the following concentration lemma.
\begin{lemma}
\label{lemma:gyorfi_variant_bellman_error_concentration}
  For any \emph{fixed} $h \in [H]$, $r \in \Tilde{\cR}$, and any policy class $\Pi \subset \{\cS \to \Delta(\cA)\}^H$ we have
  \begin{align*}
    &\Pr\Bigl(\exists f, f' \in \cF, \pi \in \Pi:\\
    &\hspace{6em}\abr{ \EE_{\mu_h}\sbr{\|f_{h} - \cT_{h, r}^{\pi}f'_{h + 1}\|^2} - \cL_{h, r}(f_{h}, f'_{h + 1}, \pi; \cD) + \cL_{h, r}(\cT_{h, r}^{\pi}f'_{h + 1}, f'_{h + 1}, \pi; \cD)}\\
    &\hspace{24em} \geq \epsilon\rbr{\alpha + \beta + \EE_{\mu_h}\sbr{\|f_{h} - \cT_{h, r}^{\pi}f'_{h + 1}\|^2}}\Bigr)\\
    &\qquad \leq 28\rbr{\cN_{\infty}\rbr{\frac{\epsilon\beta}{140HR_{\max}}, \cF}}^{2}\cN_{\infty, 1}\rbr{\frac{\epsilon\beta}{140H^{2}R^{2}_{\mathrm{max}}}, \Pi}\exp\rbr{-\frac{\epsilon^{2}(1 - \epsilon)\alpha K}{214(1 + \epsilon)H^{4}R_{\max}^{4}}}.
  \end{align*} for all $\alpha, \beta > 0$, $0 < \epsilon \leq 1/2$.
\end{lemma}
\begin{proof}
    See Section~\ref{subsec:proof_of_lemma_gyorfi_variant} for a detailed proof.
\end{proof}

Our proof hinges upon the occurrence of a ``good event" under which the difference between the empirical Bellman error estimator and the Bellman error can be bounded. We formalize the definition of the ``good event" below.

\begin{lemma}\label{corollary:good_event}
  For any policy class $\Pi \subset \{\cS \to \Delta(\cA)\}^H$, 
  let the ``good event'' $\cG(\Pi)$ be defined as
  \begin{equation}\label{eqn:defn_good_event}
    \begin{split}
      &\cG(\Pi) = \bigl\{\forall\, h \in [H], r \in \tilde{\cR}, \pi \in \Pi, f,f' \in \cF:\\
      &\hspace{6em}\abr{\EE_{\mu_h}[\|f_{h} - \cT_{h, r}^{\pi}f'_{h + 1}\|^2] - \cL_{h, r}(f_{h}, f'_{h + 1}, \pi; \cD) + \cL_{h, r}(\cT_{h, r}^{\pi}f'_{h + 1}, f'_{h + 1}, \pi; \cD)} \\
      &\hspace{25em} \leq\epsilon_{\rm S} + \frac{1}{2} \EE_{\mu_h}[\|f_{h} - \cT_{h, r}^{\pi}f'_{h + 1}\|^2]
      \bigr\},
    \end{split}
  \end{equation} where
  \begin{equation}\label{eqn:defn_eps_s}
    \epsilon_{\rm S}= \frac{5136}{K}H^{4}R^{4}_{\max}\log \bigg (   56nH \cdot \cN_{\infty}\rbr{\frac{19H^{3}R^{3}_{\max}}{K}, \cF} \cdot \cN_{\infty, 1}\rbr{\frac{19H^{4}R^{4}_{\max}}{K}, \Pi} \Big/ \delta \biggr ).
\end{equation}
  Then $\cG(\Pi)$ occurs with probability at least $1-\delta$.
\end{lemma}
\begin{proof}
    See Section~\ref{subsec:proofs_of_good_event} for a detailed proof.
\end{proof}

On the event $\cG(\Pi)$, the best approximations of action-value functions, defined according to Assumption~\ref{assumption:realizability}, have small empirical Bellman error estimates.
\begin{corollary}\label{lemma:q_star_small_cE}
Let $\Pi$ be any policy class. Conditioned on the event $\cG(\Pi)$, let $f_{r}^{\pi, *} \in \cF$ be the best estimate of $Q_{r}^{\pi}(\cdot, \cdot; r)$ as defined in Assumption~\ref{assumption:realizability},
$\pi \in \Pi$ and $r \in  \Tilde{\cR}$. Then, for all $ h \in [H]$, we have
  \[
    \cE_{h, r}(f_{r}^{\pi, *}, \pi; \cD) \leq 2\epsilon_{\rm S} + 6\epsilon_{\cF}.
  \] 
\end{corollary}
\begin{proof}
    See Section~\ref{subsec:proofs_of_good_event} for a detailed proof.
\end{proof}
We can also show that any function with sufficiently small empirical Bellman error estimate must also have small Bellman error conditioned on the good event.
\begin{corollary}
\label{lemma:small_cE_small_bellman_error}
  Let $\epsilon_{0} > 0$ be arbitrary and fixed. For any policy class $\Pi$, conditioned on the event $\cG(\Pi)$, for all $h \in [H]$, reported reward $r \in\tilde{\cR}, \pi \in \Pi, f\in \cF$, if $\cE_{h, r}(f, \pi; \cD) \leq \epsilon_{0}$, then
  \begin{align*}
    \EE_{\mu_h}\sbr{\|f_{h} - \cT_{h, r}^{\pi}f_{h + 1}\|^2} \leq 2\epsilon_{0} + 4\epsilon_{\rm S} + 3\epsilon_{\cF,\cF}.
  \end{align*} 
\end{corollary}
\begin{proof}
    See Section~\ref{subsec:proofs_of_good_event} for a detailed proof.
\end{proof}

We introduce the key properties of Algorithms~\ref{alg:policy_evaluation} and~\ref{alg:spi} that we will use. 
The following lemma states that the outputs of Algorithm~\ref{alg:policy_evaluation} are approximately optimistic and pessimistic.
\begin{lemma}\label{lemma:alg_policy_evaluation}
  For any $\pi = \{\pi_h\}_{h = 1}^H \in \Pi_{\rm SPI}$, reported reward $r \in \tilde{\cR}$, and $\lambda$, conditioned on the event $\cG(\Pi_{\rm SPI})$, the following holds simultaneously for optimistic and pessimistic outputs of Algorithm~\ref{alg:policy_evaluation}:
  \vspace{-5pt}
  \begin{enumerate}
  \itemsep0pt
    \item $\check{Q}^{\pi}_{1, r}(s_0, \pi_{1}) + \lambda\sum_{h = 1}^{H}\cE_{h,r}(\check{Q}_{r}^{\pi}, \pi; \cD) \leq Q^{\pi}_{1}(s_0, \pi_{1}; r) + \sqrt{\epsilon_{\cF}} + 2\lambda H\epsilon_{\rm S} + 6\lambda H\epsilon_{\cF};$
    \item $\hat{Q}^{\pi}_{1, r}(s_0, \pi_{1}) - \lambda\sum_{h = 1}^{H}\cE_{h,r}(\hat{Q}_{r}^{\pi}, \pi; \cD) \geq Q^{\pi}_{1}(s_0, \pi_{1}; r) - \sqrt{\epsilon_{\cF}} - 2\lambda H\epsilon_{\rm S} - 6\lambda H\epsilon_{\cF}.$
  \end{enumerate}
\end{lemma}
\begin{proof}
    See Section~\ref{sec:proofs_of_alg_policy_evaluation} for a detailed proof.
\end{proof}
Additionally, the estimates given by Algorithm~\ref{alg:policy_evaluation} are sufficiently good estimates of the ground truth action-value functions.
\begin{lemma}\label{thm:alg_policy_evaluation_performance}
  For any input $\pi = \{\pi_h\}_{h = 1}^H \in \Pi_{\rm SPI}$, reported reward $r \in \tilde{\cR}$, when $\lambda = \rbr{\frac{R_{\max}}{H^{2}(\epsilon_{\rm S} + 3\epsilon_{\cF})^{2}}}^{1/3}$ and the event $\cG(\Pi_{\rm SPI})$ holds, the outputs of Algorithm~\ref{alg:policy_evaluation} satisfy:
  \vspace{-5pt}
  \begin{enumerate}
  \itemsep0em 
      \item $Q_{1}^{\pi}(s_{0}, \pi_{1}; r) - \check{Q}_{1, r}^{\pi}(s_{0}, \pi_{1}) \leq H\sqrt{C^{\pi}(\pi)}\rbr{2(HR_{\max})^{1/3}(\epsilon_{\rm S} + 3\epsilon_{\cF})^{1/3} +\sqrt{ 8\epsilon_{\rm S} + 12 \epsilon_{\cF} + 3\epsilon_{\cF, \cF}}}$;
      \item $\hat{Q}_{1, r}^{\pi}(s_{0}, \pi_{1}) - Q_{1}^{\pi}(s_{0}, \pi_{1}; r) \leq H\sqrt{C^\pi(\pi)}\rbr{2(HR_{\max})^{1/3}(\epsilon_{\rm S} + 3\epsilon_{\cF})^{1/3} +\sqrt{ 8\epsilon_{\rm S} + 12 \epsilon_{\cF} + 3\epsilon_{\cF, \cF}}}$.
  \end{enumerate}
\end{lemma}
\begin{proof}
    See Section~\ref{sec:proofs_of_alg_policy_evaluation} for a detailed proof.
\end{proof}

Finally, we bound the difference between outputs of Algorithm~\ref{alg:spi} and the true values. More precisely, we characterize the performance of the output policy with respect to \emph{any} comparator policy, not necessarily in the induced policy class $\Pi_{\rm SPI}$, and bound the difference between the estimated value function and the true value function of the output policy.
\begin{lemma}\label{thm:alg_spi}
  For any comparator policy $\pi$ (not necessarily in $\Pi_{\rm SPI}$), any reported reward function $r \in \tilde{\cR}$, with $\eta$ set to $\sqrt{\frac{\log|\cA|}{2H^{2}R_{\max}^2T}}$ and $\lambda$ set to $\rbr{\frac{R_{\max}}{H^{2}(\epsilon_{\rm S} + 3\epsilon_{\cF})^{2}}}^{1/3}$ in Algorithm~\ref{alg:spi}, the following claims hold conditioned on the event $\cG(\Pi_{\rm SPI})$:
  \begin{enumerate}
    \item Let $\check{Q}_{1, r}^{(t)}$ and $\check{\pi}^{(t)}_{r}$ be the pessimistic value function estimate and policy estimate. Then 
    \begin{align*}
      &V_{1}^{\pi}(s_0; r) - \frac{1}{T}\sum_{t =1}^T\check{Q}_{1, r}^{(t)}(s_0, \check{\pi}^{(t)}_{1, r}) \leq 2H^2R_{\max}\sqrt{\frac{2\log|\cA|}{T}}\\
    &\qquad  + H\rbr{\frac{1}{T}\sum_{t = 1}^T \sqrt{C^{\check{\pi}^{(t)}_r}({\pi})}}\rbr{2(HR_{\max})^{1/3}(\epsilon_{\rm S} + 3\epsilon_{\cF})^{1/3} +\sqrt{ 8\epsilon_{\rm S} + 12\epsilon_{\cF} + 3\epsilon_{\cF, \cF}}}.
    \end{align*}
    \item Let $\hat{Q}_{1, r}^{(t)}$ and $\hat{\pi}^{(t)}_{r}$ be the optimistic value function estimate and policy estimate.
    Then
    \begin{align*}
      &V_{1}^{\pi}(s_0; r) - \frac{1}{T}\sum_{t =1}^T\hat{Q}_{1, r}^{(t)}(s_0, \hat{\pi}^{(t)}_{1, r}) \leq 2H^2R_{\max}\sqrt{\frac{2\log|\cA|}{T}}\\
    &\qquad  + H\rbr{\frac{1}{T}\sum_{t = 1}^T \sqrt{C^{\hat{\pi}^{(t)}_r}({\pi})}}\rbr{2(HR_{\max})^{1/3}(\epsilon_{\rm S} + 3\epsilon_{\cF})^{1/3} +\sqrt{ 8\epsilon_{\rm S} + 12\epsilon_{\cF} + 3\epsilon_{\cF, \cF}}}.
    \end{align*}
  \end{enumerate}
\end{lemma}
\begin{proof}
    See Section~\ref{sec:proofs_of_alg_spi} for a detailed proof.
\end{proof}

We then proceed with the proof as follows. We start by bounding the suboptimality of the output policy, defined according to equation \eqref{eqn:defn_subopt}. We then bound the regret of each individual agent and the seller. We follow up with showing that our output asymptotically satisfies individual rationality. Finally, we prove that our output also asymptotically satisfies truthfulness.

We use the following notation to differentiate the policies and prices learned under different truthfulness assumptions. Let $\check{\pi} = \check{\pi}_R^{\rm out}$ be the policy chosen by the algorithm when all agents are truthful, let $\tilde{\pi} = \check{\pi}^{\rm out}_{r_i + \tilde{R}_{-i}}$ be the policy chosen when we only assume the agent $i$ is truthful, and finally let $\check{\pi}_{\tilde{R}} = \check{\pi}_{\tilde{R}}^{\rm out}$ be the policy chosen when none of the agents are truthful. Let the prices charged by the algorithm be $\{\hat{p}_i\}_{i = 1}^n, \{\tilde{p}_i\}_{i = 1}^n,$ and $\{\hat{p}_{ i, \tilde{R}}\}_{i = 1}^n$, respectively.

\paragraph{Social Welfare Suboptimality}
\label{subsec:welfare_subopt}
Assuming all agents are truthful, we have $\tilde{r}_{i} = r_{i}$ for all $i$. Let $\pi^{*}$ be the maximizer of $V^{\pi}_{1}(s_0; R)$ over $\pi$ and let $\check{\pi}^{(t)}_{R}$ be the pessimistic policy iterate of Algorithm~\ref{alg:spi}. We know that the social welfare suboptimality of $\check{\pi}$ is
\begin{align*}
  {\rm SubOpt}(\check{\pi}; s_0) &= V_{1}^{\pi^{*}}(s_0; R) - V_{1}^{\check{\pi}}(s_0; R) = V_{1}^{\pi^{*}}(s_0; R) - \frac{1}{T}\sum_{t = 1}^T  V_{1}^{\check{\pi}^{(t)}_{R}}(s_0; R) \\
  &= \frac{1}{T}\sum_{t = 1}^{T}\rbr{ V_{1}^{\pi^{*}}(s_0; R) - Q^{\check{\pi}^{(t)}}_{1}(s_0, \check{\pi}^{(t)}_{1, R}; R)},
\end{align*} 
as we recall that $\check{\pi}$ is the uniform mixture of policies $\{\check{\pi}_{R}^{(t)}\}_{t\in [T]}$. By Lemma~\ref{lemma:alg_policy_evaluation}, we have
\begin{align}\label{eqn:welfare_regret_decomposition}
  {\rm SubOpt}(\check{\pi}; s_0) \leq \frac{1}{T}\sum_{t = 1}^T \rbr{V_{1}^{\pi^{*}}(s_0; R) - \check{Q}_{1,R}^{(t)}(s_0, \check{\pi}_{1, R}^{(t)}; R)} + \sqrt{\epsilon_{\cF}} + 2\lambda H\epsilon_{\rm S} + 6\lambda H\epsilon_{\cF},
\end{align} 
where $\check{Q}_{R}^{(t)}$ is the pessimistic estimate of $Q(\cdot, \cdot; R)$ at the $t$-th iteration of Algorithm~\ref{alg:spi}. When $\lambda = \rbr{\frac{R_{\max}}{H^{2}(\epsilon_{\rm S} + 3\epsilon_{\cF})^{2}}}^{1/3}$ and $\eta = \sqrt{\frac{\log|\cA|}{2H^{2}R_{\max}^2T}}$,
we apply Lemma~\ref{thm:alg_spi} to obtain
\begin{align*}
  &{\rm SubOpt}(\check{\pi}; s_0) \leq 2H^2R_{\max}\sqrt{\frac{2\log|\cA|}{T}} + \sqrt{\epsilon_{\cF}} + 2(HR_{\max})^{1/3}(\epsilon_{\rm S} + 3\epsilon_{\cF})^{1/3}\\
  &\qquad + H\rbr{\frac{1}{T}\sum_{t = 1}^T\sqrt{C^{\check{\pi}_R^{(t)}}({\pi^*})}}\rbr{2(HR_{\max})^{1/3}(\epsilon_{\rm S} + 3\epsilon_{\cF})^{1/3} +\sqrt{ 8\epsilon_{\rm S} + 12 \epsilon_{\cF} + 3\epsilon_{\cF, \cF}}}.
\end{align*}

\paragraph{Individual Suboptimality}
\label{subsec:individual_subopt}
Let $\pi^*_{-i}$ be the maximizer of $V^{\pi}(s_0; R_{-i})$ over $\pi$.
By Algorithm~\ref{alg:vcg_learn}, the price $\hat{p}_{i}$ is constructed as
\[
    \hat{p}_{i} = G_{-i}^{(1)}(s_{0}) - G_{-i}^{(2)}(s_{0}),
\] 
where $G_{-i}^{(1)}(s_{0})$ is an estimate of $V^{\pi_{-i}^{*}}(s_{0}; R_{-i})$ obtained using Algorithm~\ref{alg:spi} and $G_{-i}^{(2)}(s_{0})$ is an estimate of $V^{\check{\pi}}(s_{0}; R_{-i})$ for Algorithm~\ref{alg:vcg_learn}'s output policy, $\check{\pi}$. This observation will be extensively used in the remainder of the proof.

Assuming all agents are truthful, we have $\tilde{r}_{i} = r_{i}$ for all $i$. Recalling the construction of $\hat{p}_i$ in Algorithm~\ref{alg:vcg_learn} line~\ref{algline:vcg_pricing} and the definition of $\{p_i^*\}_{i = 1}^n$ (see~\eqref{eqn:defn_p_star}), we have
\begin{equation*}
\begin{split}
  &{\rm SubOpt}_{i}(\check{\pi}, \{\hat{p}_i\}_{i = 1}^n; s_0) \\
  &\quad= V_{1}^{\pi^{*}}(s_0; r_{i}) + V_{1}^{\pi^{*}}(s_0; R_{-i}) -V_{1}^{\pi_{-i}^{*}}(s_0; R_{-i}) - V_{1}^{\check{\pi}}(s_0; r_{i}) + G^{(1)}_{-i}(s_0) - G^{(2)}_{-i}(s_0)\\
  &\quad = V_{1}^{\pi^{*}}(s_0; R) - V_{1}^{\pi_{-i}^{*}}(s_0; R_{-i}) - V_{1}^{\check{\pi}}(s_0; r_{i}) + G^{(1)}_{-i}(s_0) - G^{(2)}_{-i}(s_0)\\
  &\quad \leq V_{1}^{\pi^{*}}(s_0; R) - V_{1}^{\check{\pi}}(s_0; R) + \rbr{G_{-i}^{(1)}(s_0) - V_{1}^{\pi_{-i}^{*}}(s_0; R_{-i})} + \rbr{V_{1}^{\check{\pi}}(s_0; R_{-i}) - G^{(2)}_{-i}(s_0)}\\
  &\quad = {\rm SubOpt}(\check{\pi}; s_0) + \rbr{G_{-i}^{(1)}(s_0) - V_{1}^{\pi_{-i}^{*}}(s_0; R_{-i})} + \rbr{V_{1}^{\check{\pi}}(s_0; R_{-i}) - G^{(2)}_{-i}(s_0)}.
\end{split}
\end{equation*} 
We have already bounded the first term and now focus on the two latter terms. 

We begin by examining $G_{-i}^{(1)}(s_0) - V_{1}^{\pi_{-i}^{*}}(s_0; R_{-i})$.
\begin{itemize}
  \item Suppose $\zeta_{1} = {\tt OPT}$. Since $\pi_{-i}^{*}$ maximizes $ V_{1}^{\pi_{-i}^{*}}(s_0; R_{-i})$ over $\pi$, we have
        \begin{align*}
          G_{-i}^{(1)}(s_0) - V_{1}^{\pi_{-i}^{*}}(s_0; R_{-i}) &\leq G_{-i}^{(1)}(s_0) - V_{1}^{\hat{\pi}_{-i}}(s_0; R_{-i}).
        \end{align*} 
        Recall that $\hat{Q}^{\rm out}_{R_{-i}}$ is the optimistic function estimate from the output of Algorithm~\ref{alg:spi}, which is exactly the output of Algorithm~\ref{alg:policy_evaluation} called on the policy returned by Algorithm~\ref{alg:spi}, $\hat{\pi}_{-i}$. By Lemma~\ref{thm:alg_policy_evaluation_performance}, we know that
        \begin{multline*}
            G_{-i}^{(i)}(s_0) - V_1^{\hat{\pi}_{-i}}(s_0; R_{-i}) \\
            \leq H\sqrt{C^{\hat{\pi}_{-i}}({\hat{\pi}_{-i}})}\rbr{2(HR_{\max})^{1/3}(\epsilon_{\rm S} + 3\epsilon_{\cF})^{1/3} +\sqrt{ 8\epsilon_{\rm S} + 12 \epsilon_{\cF} + 3\epsilon_{\cF, \cF}}}.
        \end{multline*}
  \item Suppose $\zeta_{1} = {\tt PES}$. Since $\pi_{-i}^{*}$ maximizes $ V_{1}^{\pi}(s_0; R_{-i})$ over $\pi$, we have
        \begin{align*}
          G_{-i}^{(1)}(s_0) - V_{1}^{\pi_{-i}^{*}}(s_0; R_{-i}) &\leq G_{-i}^{(1)}(s_0) - V_{1}^{\check{\pi}_{-i}}(s_0; R_{-i}).
        \end{align*}
        Recall that $G_{-i}^{(1)}(s_0) = \check{Q}^{\rm out}_{1, R_{-i}}(s_0, \check{\pi}_{1, -i})$.
        When $\lambda = \rbr{\frac{R_{\max}}{H^{2}(\epsilon_{\rm S} + 3\epsilon_{\cF})^{2}}}^{1/3}$, by Lemma~\ref{lemma:alg_policy_evaluation} we have
        \begin{align*}
          G_{-i}^{(1)}(s_0) - V_{1}^{\pi_{-i}^{*}}(s_0; R_{-i}) \leq \sqrt{\epsilon_{\cF}} + 2(HR_{\max})^{1/3}(\epsilon_{\rm S} + 3\epsilon_{\cF})^{1/3}.
        \end{align*}
\end{itemize}
 
We perform a similar analysis for $V_{1}^{\check{\pi}}(s_0; R_{-i}) - G^{(2)}_{-i}(s_0)$ and when $\lambda = \rbr{\frac{R_{\max}}{H^{2}(\epsilon_{\rm S} + 3\epsilon_{\cF})^{2}}}^{1/3}$.
\begin{itemize}
  \item When $\zeta_{2} = {\tt OPT}$, $V_{1}^{\check{\pi}}(s_0; R_{-i}) - G^{(2)}_{-i}(s_0) \leq \sqrt{\epsilon_{\cF}} + 2(HR_{\max})^{1/3}(\epsilon_{\rm S} + 3\epsilon_{\cF})^{1/3}$ by Lemma~\ref{lemma:alg_policy_evaluation}.
  \item When $\zeta_{2} = {\tt PES}$, let $\check{Q}^{\check{\pi}}_{R_{-i}}$ be the pessimistic output of Algorithm~\ref{alg:policy_evaluation} called on $\check{\pi}$. By Lemma~\ref{thm:alg_policy_evaluation_performance}, we have
  \begin{align*}
    &V_{1}^{\check{\pi}}(s_0; R_{-i}) - G_{-i}^{(2)}(s_0) \leq  H\sqrt{C^{\check{\pi}}({\check{\pi}})}\rbr{2(HR_{\max})^{1/3}(\epsilon_{\rm S} + 3\epsilon_{\cF})^{1/3} +\sqrt{ 8\epsilon_{\rm S} + 12 \epsilon_{\cF} + 3\epsilon_{\cF, \cF}}}.
  \end{align*}
\end{itemize}

\paragraph{Seller Suboptimality}
\label{subsec:seller_suboptimality}
We now turn our attention to the sellers' suboptimality. Assuming all agents are truthful, we have $\tilde{r}_{i} = r_{i}$ for all $i$. Recalling the definition of $\{p_i^*\}_{i = 1}^n$ in~\eqref{eqn:defn_p_star}, we have
\begin{equation}\label{eqn:seller_subopt_decomp}
\begin{split}
  &\textrm{SubOpt}_0(\check{\pi}, \{\hat{p}_i\}_{i = 1}^n; s_0) \\
  & \qquad = V_1^{\pi^*}(s_0; r_0) - V_1^{\check{\pi}}(s_0; r_0)+ \sum_{i = 1}^n\rbr{\max_{\pi' \in \Pi} V_1^{\pi'}(s_0; R_{-i}) - V_1^{\pi^*}(s_0; R_{-i})} - \sum_{i = 1}^n \hat{p}_i\\
  &\qquad = \sum_{i = 1}^n\max_{\pi' \in \Pi} V_1^{\pi'}(s_0; R_{-i}) - (n - 1)V_{1}^{\pi^{*}}(s_{0}; R) - V_1^{\check{\pi}}(s_0; r_0) - \sum_{i = 1}^{n}G^{(1)}_{-i}(s_{0}) + \sum_{i=1}^n G^{(2)}_{-i}(s_{0})\\
  &\qquad = \sum_{i = 1}^n\rbr{\max_{\pi' \in \Pi} V_1^{\pi'}(s_0; R_{-i}) - G^{(1)}_{-i}(s_{0})} - (n - 1)V_{1}^{\pi^{*}}(s_{0}; R) - V_1^{\check{\pi}}(s_0; r_0)  + \sum_{i=1}^n G^{(2)}_{-i}(s_{0})\\
  &\qquad = \sum_{i = 1}^n\rbr{ V_1^{\pi^{*}_{-i}}(s_0; R_{-i}) - G^{(1)}_{-i}(s_{0})} + (n - 1)(V_{1}^{\check{\pi}}(s_{0}; R) - V_{1}^{\pi^{*}}(s_{0}; R)) \\
  &\hspace{3em}+ \sum_{i=1}^n \rbr{ G^{(2)}_{-i}(s_{0}) -V_{1}^{\check{\pi}}(s_{0}, R_{-i}) }\\
  &\qquad \leq \sum_{i = 1}^n\rbr{ V_1^{\pi^{*}_{-i}}(s_0; R_{-i}) - G^{(1)}_{-i}(s_{0})} + \sum_{i=1}^n \rbr{ G^{(2)}_{-i}(s_{0}) -V_{1}^{\check{\pi}}(s_{0}, R_{-i}) },
\end{split}
\end{equation} 
where the last inequality comes from the fact that $\pi^*$ is the social welfare-maximizing policy. The two terms can be bounded similarly to bounding the agents' suboptimality. We discuss the exact bounds for different choices of $\zeta_{1}, \zeta_{2}$ and $\lambda = \rbr{\frac{R_{\max}}{H^{2}(\epsilon_{\rm S} + 3\epsilon_{\cF})^{2}}}^{1/3}, \eta = \sqrt{\frac{\log |\cA|}{2H^2R_{\rm max}^2T}}$.
\begin{itemize}
  \item When $\zeta_{1} = \mathtt{OPT}$, by Algorithm~\ref{alg:vcg_learn} line~\ref{algline:vcg_pricing}, we know that for any $i \in [n]$,
        \begin{align*}
          &V_{1}^{\pi^{*}_{-i}}(s_{0}; R_{-i}) - G_{-i}^{(1)}(s_{0}) = V_{1}^{\pi^{*}_{-i}}(s_{0}; R_{-i}) -\hat{Q}^{\rm out}_{1, R_{-i}}(s_{0}, \hat{\pi}_{1, -i}).
        \end{align*} 
        By Lemma~\ref{thm:alg_spi}, we know that
        \begin{align*}
          &V_1^{\pi^{*}_{-i}}(s_0; R_{-i}) - \frac{1}{T}\sum_{t = 1}^T \hat{Q}^{(t)}_{1, R_{-i}}(s_{0}, \hat{\pi}^{(t)}_{1, R_{-i}})\leq 2H^{2}R_{\max}\sqrt{\frac{2\log |\cA|}{T}}\\
          &\qquad + H\rbr{\frac{1}{T}\sum_{t = 1}^T\sqrt{C^{\hat{\pi}_{R_{-i}}^{(t)}}({\pi_{-i}^*})}}\rbr{2(HR_{\max})^{1/3}(\epsilon_{\rm S} + 3\epsilon_{\cF})^{1/3} + \sqrt{8\epsilon_{\rm S} + 12\epsilon_{\cF} + 3\epsilon_{\cF, \cF}}}.
        \end{align*}
        By Lemma~\ref{thm:alg_policy_evaluation_performance} and recalling that $\hat{\pi}_{-i}$ is the uniform mixture of $\{\hat{\pi}_{R_{-i}}^{(t)}\}_{t \in [T]}$, we know that
        \begin{align*}
          \frac{1}{T}&\sum_{t = 1}^T \hat{Q}^{(t)}_{1, R_{-i}}(s_{0}, \hat{\pi}^{(t)}_{1, R_{-i}}) - V_{1}^{\hat{\pi}_{-i}}(s_{0}; R_{-i}) \\
          &=\frac{1}{T}\sum_{t = 1}^T\rbr{ \hat{Q}^{(t)}_{1, R_{-i}}(s_{0}, \hat{\pi}^{(t)}_{1, R_{-i}}) - V_{1}^{\hat{\pi}_{R_{-i}}^{(t)}}(s_{0}; R_{-i})} \\
          &\leq H\rbr{\frac{1}{T}\sum_{t = 1}^T\sqrt{C^{\hat{\pi}^{(t)}_{R_{-i}}}({\hat{\pi}^{(t)}_{R_{-i}}})}}\rbr{2(HR_{\max})^{1/3}(\epsilon_{\rm S} + 3\epsilon_{\cF})^{1/3} + \sqrt{8\epsilon_{\rm S} + 12\epsilon_{\cF} + 3\epsilon_{\cF, \cF}}}.
        \end{align*} Lastly, by Lemma~\ref{lemma:alg_policy_evaluation}, we also know that
        \begin{align*}
          V_{1}^{\hat{\pi}_{-i}}(s_{0}; R_{-i}) -\hat{Q}^{\rm out}_{1, R_{-i}}(s_{0}, \hat{\pi}_{1, -i}) \leq \sqrt{\epsilon_{\cF}} + 2(HR_{\max})^{1/3}(\epsilon_{\rm S} + 3\epsilon_{\cF})^{1/3}.
        \end{align*} 
        Summing the three parts tells us that, for all $i \in [n]$, we have
        \begin{equation}\label{eqn:g_i_bound_opt}
        \begin{split}
          V_{1}^{\pi^{*}_{-i}}(s_{0}; R_{-i}) &- G_{-i}^{(1)}(s_{0})\\
          &= V_{1}^{\pi^{*}_{-i}}(s_{0}; R_{-i}) -\hat{Q}^{\rm out}_{1, R_{-i}}(s_{0}, \hat{\pi}_{1, -i})\\
          &\leq 2H^{2}R_{\max}\sqrt{\frac{2\log |\cA|}{T}} + \sqrt{\epsilon_{\cF}} + 2(HR_{\max})^{1/3}(\epsilon_{\rm S} + 3\epsilon_{\cF})^{1/3} \\
          &\quad + H\rbr{\frac{1}{T}\sum_{t = 1}^T\rbr{\sqrt{C^{\hat{\pi}_{R_{-i}}^{(t)}}({\pi_{-i}^*})} + \sqrt{C^{\hat{\pi}^{(t)}_{R_{-i}}}(\hat{\pi}^{(t)}_{R_{-i}})}}}\\
          &\qquad \times \rbr{2(HR_{\max})^{1/3}(\epsilon_{\rm S} + 3\epsilon_{\cF})^{1/3} + \sqrt{8\epsilon_{\rm S} + 12\epsilon_{\cF} + 3\epsilon_{\cF, \cF}}}
        \end{split}
        \end{equation} 
        and
        \begin{align*}
          \sum_{i = 1}^n &\rbr{ V_{1}^{\pi^{*}_{-i}}(s_{0}; R_{-i}) - G_{-i}^{(1)}(s_{0}) }\\
          & \leq 2nH^{2}R_{\max}\sqrt{\frac{2\log |\cA|}{T}} + n\sqrt{\epsilon_{\cF}} + 2n(HR_{\max})^{1/3}(\epsilon_{\rm S} + 3\epsilon_{\cF})^{1/3} \\
          &\quad + H\rbr{\sum_{i = 1}^n\frac{1}{T}\sum_{t = 1}^T\rbr{\sqrt{C^{\hat{\pi}_{R_{-i}}^{(t)}}({\pi_{-i}^*})} + \sqrt{C^{\hat{\pi}^{(t)}_{R_{-i}}}(\hat{\pi}^{(t)}_{R_{-i}})}}}\\
          &\qquad  \times \rbr{2(HR_{\max})^{1/3}(\epsilon_{\rm S} + 3\epsilon_{\cF})^{1/3} + \sqrt{8\epsilon_{\rm S} + 12\epsilon_{\cF} + 3\epsilon_{\cF, \cF}}}.
        \end{align*}
  \item When $\zeta_{1} = \mathtt{PES}$, by Algorithm~\ref{alg:vcg_learn} we know that for any $i \in [n]$,
        \begin{align*}
          &V_{1}^{\pi^{*}_{-i}}(s_{0}; R_{-i}) - G_{-i}^{(1)}(s_{0}) = V_{1}^{\pi^{*}_{-i}}(s_{0}; R_{-i}) -\check{Q}^{\rm out}_{1, R_{-i}}(s_{0}, \check{\pi}_{1, -i}).
        \end{align*} 
        By Lemma~\ref{thm:alg_spi}, we know that
        \begin{align*}
          &V_1^{\pi^{*}_{-i}}(s_0; R_{-i}) - \frac{1}{T}\sum_{t = 1}^T \check{Q}^{(t)}_{1, R_{-i}}(s_{0}, \check{\pi}^{(t)}_{1, R_{-i}})\leq 2H^{2}R_{\max}\sqrt{\frac{2\log |\cA|}{T}}\\
          &\qquad + H\rbr{\frac{1}{T}\sum_{t = 1}^T\sqrt{C^{\check{\pi}_{R_{-i}}^{(t)}}(\pi^*_{-i})}}\rbr{2(HR_{\max})^{1/3}(\epsilon_{\rm S} + 3\epsilon_{\cF})^{1/3} + \sqrt{8\epsilon_{\rm S} + 12\epsilon_{\cF} + 3\epsilon_{\cF, \cF}}}.
        \end{align*} 
        By Lemma~\ref{lemma:alg_policy_evaluation}, we know that
        \begin{align*}
          \frac{1}{T}\sum_{t = 1}^T \check{Q}^{(t)}_{1, R_{-i}}(s_{0}, \check{\pi}^{(t)}_{1, R_{-i}}) - V_{1}^{\check{\pi}_{-i}}(s_{0}; R_{-i})\leq \sqrt{\epsilon_{\cF}} + 2(HR_{\max})^{1/3}(\epsilon_{\rm S} + 3\epsilon_{\cF})^{1/3}.
        \end{align*} 
        By Lemma~\ref{thm:alg_policy_evaluation_performance}, we further know that
        \begin{multline*}
          V_{1}^{\check{\pi}_{-i}}(s_{0}; R_{-i}) - \check{Q}^{\rm out}_{1, R_{-i}}(s_{0}, \check{\pi}_{1, -i}) \\
          \leq H\sqrt{C^{\check{\pi}_{-i}}(\check{\pi}_{-i})}\rbr{2(HR_{\max})^{1/3}(\epsilon_{\rm S} + 3\epsilon_{\cF})^{1/3} + \sqrt{8\epsilon_{\rm S} + 12\epsilon_{\cF} + 3\epsilon_{\cF, \cF}}}.
        \end{multline*}
        Summing the three parts together tells us that, for all $i \in [n]$ and any $C \geq 1$, we have
        \begin{equation}\label{eqn:g_i_bound_pes}
        \begin{split}
          &V_{1}^{\pi^{*}_{-i}}(s_{0}; R_{-i}) - G_{-i}^{(1)}(s_{0}) = V_{1}^{\pi^{*}_{-i}}(s_{0}; R_{-i}) -\check{Q}^{\rm out}_{1, R_{-i}}(s_{0}, \check{\pi}_{1, -i}) \\
          &\qquad \leq 2H^{2}R_{\max}\sqrt{\frac{2\log |\cA|}{T}} + \sqrt{\epsilon_{\cF}} + 2(HR_{\max})^{1/3}(\epsilon_{\rm S} + 3\epsilon_{\cF})^{1/3}\\
          &\hspace{3em} + H\rbr{\sqrt{C^{\check{\pi}_{-i}}(\check{\pi}_{-i})} + \frac{1}{T}\sum_{t = 1}^T\sqrt{C^{\check{\pi}_{R_{-i}}^{(t)}}(\pi^*_{-i})}}\\
          &\hspace{4em} \times \rbr{2(HR_{\max})^{1/3}(\epsilon_{\rm S} + 3\epsilon_{\cF})^{1/3} + \sqrt{8\epsilon_{\rm S} + 12\epsilon_{\cF} + 3\epsilon_{\cF, \cF}}}
        \end{split}
        \end{equation}
        and
        \begin{align*}
          \sum_{i = 1}^n &\rbr{V_1^{\pi^{*}_{-i}}(s_0; R_{-i}) - G^{(1)}_{-i}(s_{0})  } \\
          &\leq 2nH^{2}R_{\max}\sqrt{\frac{2\log |\cA|}{T}} + n\sqrt{\epsilon_{\cF}} + 2n(HR_{\max})^{1/3}(\epsilon_{\rm S} + 3\epsilon_{\cF})^{1/3}\\
          &\quad + H\rbr{\sum_{i = 1}^n\sqrt{C^{\check{\pi}_{-i}}(\check{\pi}_{-i})} + \sum_{i = 1}^n\frac{1}{T}\sum_{t = 1}^T\sqrt{C^{\check{\pi}_{R_{-i}}^{(t)}}(\pi^*_{-i})}}\\
          &\qquad \times \rbr{2(HR_{\max})^{1/3}(\epsilon_{\rm S} + 3\epsilon_{\cF})^{1/3} + \sqrt{8\epsilon_{\rm S} + 12\epsilon_{\cF} + 3\epsilon_{\cF, \cF}}}.
        \end{align*}
  \item When $\zeta_{2} = {\tt OPT}$, for all $i \in [n]$, let $\check{Q}^{\check{\pi}}_{R_{-i}}$ be the pessimistic estimate of $Q^{\check{\pi}}(\cdot, \cdot; R_{-i})$ returned by Algorithm~\ref{alg:policy_evaluation}. By Lemma~\ref{thm:alg_policy_evaluation_performance}, we know
  \begin{multline*}
    \sum_{i=1}^n \rbr{ G^{(2)}_{-i}(s_{0}) -V_{1}^{\check{\pi}}(s_{0}, R_{-i}) } \\
    \leq nH\sqrt{C^{\check{\pi}}(\check{\pi})}\rbr{2(HR_{\max})^{1/3}(\epsilon_{\rm S} + 3\epsilon_{\cF})^{1/3} + \sqrt{8\epsilon_{\rm S} + 12\epsilon_{\cF} + 3\epsilon_{\cF, \cF}}}.
  \end{multline*}
  \item When $\zeta_{2} = {\tt PES}$, $\sum_{i=1}^n \rbr{ G^{(2)}_{-i}(s_{0}) -V_{1}^{\check{\pi}}(s_{0}, R_{-i}) } \leq n\sqrt{\epsilon_{\cF}} + 2n(HR_{\max})^{1/3}(\epsilon_{\rm S} + 3\epsilon_{\cF})^{1/3}$ by Lemma~\ref{lemma:alg_policy_evaluation}.
\end{itemize} Plugging in the bound for ${\rm SubOpt}(\check{\pi}; s_{0})$ completes the proof.

\paragraph{Individual Rationality}\label{subsec:individual_rationality}
We show that the utility of any agent $i$ is bounded below.
First, assume for convenience that all other agents are truthful and report their true $r_{i', h}$ for $i' \in [n] \setminus {i}$. Recall that for any price $p_{i}$, the agents' expected utility under the chosen policy $\check{\pi}$ can be written as
\[
  \EE_{d_{\check{\pi}}}[u_{i}] = V^{\check{\pi}}_{1}(s_{0}; r_{i}) - p_{i}.
\] 
According to Algorithm~\ref{alg:vcg_learn}, we have
\begin{equation}
  \label{eqn:individual_rationality_decomp}
  \begin{split}
    &\EE_{{\check{\pi}}}[u_{i}] = V^{\check{\pi}}_{1}(s_{0}; r_{i}) - G_{-i}^{(1)}(s_{0}) + G_{-i}^{(2)}(s_{0})\\
    &\qquad = V^{\check{\pi}}_{1}(s_{0}; r_{i}) +  G_{-i}^{(2)}(s_{0}) - V^{\pi^{*}_{-i}}(s_{0}; R_{-i}) + V^{\pi^{*}_{-i}}(s_{0}; R_{-i}) - G_{-i}^{(1)}(s_{0})\\
    &\qquad = (V^{\pi^{*}}(s_{0}; R) - V^{\pi^{*}_{-i}}(s_{0}; R_{-i})) + V^{\check{\pi}}(s_{0}; r_{i}) + G_{-i}^{(2)}(s_{0}) - V^{\pi^{*}}(s_{0}; R) \\
    &\hspace{3em} +  V^{\pi^{*}_{-i}}(s_{0}; R_{-i}) - G_{-i}^{(1)}(s_{0})\\
    &\qquad \geq V^{\check{\pi}}(s_{0}; r_{i}) + G_{-i}^{(2)}(s_{0}) - V^{\pi^{*}}(s_{0}; R) +  V^{\pi^{*}_{-i}}(s_{0}; R_{-i}) - G_{-i}^{(1)}(s_{0})\\
    &\qquad = G_{-i}^{(2)}(s_{0}) - V^{\check{\pi}}(s_{0}; R_{-i}) + V^{\check{\pi}}(s_{0}; R) - V^{\pi^{*}}(s_{0}; R) +V^{\pi^{*}_{-i}}(s_{0}; R_{-i}) - G_{-i}^{(1)}(s_{0}),
  \end{split}
\end{equation} 
where the inequality comes from the fact that  
\[
(V^{\pi^{*}}(s_{0}; R) - V^{\pi^{*}_{-i}}(s_{0}; R_{-i})) \geq (V^{\pi_{-i}^{*}}(s_{0}; R) - V^{\pi^{*}_{-i}}(s_{0}; R_{-i})) = V^{\pi^{*}_{-i}}(s_{0}; r_{i})) \geq 0,
\]
as $r_{i, h} \in [0, 1]$ for all $i, h$. We already know the lower bounds for $V^{\pi^{*}_{-i}}(s_{0}; R_{-i}) - G_{-i}^{(1)}(s_{0})$ and $G_{-i}^{(2)}(s_{0}) - V^{\check{\pi}}(s_{0}; R_{-i})$ , respectively, when bounding the individual suboptimalities for the  agents. Also note that $V^{\check{\pi}}(s_{0; R}) - V^{\pi^{*}}(s_{0}; R)  = -{\rm SubOpt}(\check{\pi}; s_{0})$ has been bounded when bounding social welfare suboptimality.

Similar to the previous sections, we now discuss the bounds for the different terms under difference choices of $\zeta_{1}, \zeta_{2}$.
\begin{itemize}
  \item When $\zeta_{1} = \mathtt{OPT}$, by equation~\eqref{eqn:g_i_bound_opt} we know that
        \begin{equation*}
        \begin{split}
          & G_{-i}^{(1)}(s_{0}) - V_{1}^{\pi^{*}_{-i}}(s_{0}; R_{-i}) \geq -2H^{2}R_{\max}\sqrt{\frac{2\log |\cA|}{T}} - \sqrt{\epsilon_{\cF}} - 2(HR_{\max})^{1/3}(\epsilon_{\rm S} + 3\epsilon_{\cF})^{1/3} \\
          &\hspace{3em} - H\rbr{\frac{1}{T}\sum_{t = 1}^T\rbr{\sqrt{C^{\hat{\pi}_{R_{-i}}^{(t)}}({\pi_{-i}^*})} + \sqrt{C^{\hat{\pi}^{(t)}_{R_{-i}}}(\hat{\pi}^{(t)}_{R_{-i}})}}}\\
          &\hspace{3em}\times\rbr{2(HR_{\max})^{1/3}(\epsilon_{\rm S} + 3\epsilon_{\cF})^{1/3} + \sqrt{8\epsilon_{\rm S} + 12\epsilon_{\cF} + 3\epsilon_{\cF, \cF}}}.
        \end{split}
        \end{equation*}
  \item When $\zeta_{1} = \mathtt{PES}$, by equation~\eqref{eqn:g_i_bound_pes} we know that
        \begin{equation*}
        \begin{split}
          & G_{-i}^{(1)}(s_{0}) - V_{1}^{\pi^{*}_{-i}}(s_{0}; R_{-i}) \geq -2H^{2}R_{\max}\sqrt{\frac{2\log |\cA|}{T}} - \sqrt{\epsilon_{\cF}} - 2(HR_{\max})^{1/3}(\epsilon_{\rm S} + 3\epsilon_{\cF})^{1/3}\\
          &\hspace{3em} - H\rbr{\sqrt{C^{\check{\pi}_{-i}}(\check{\pi}_{-i})} + \frac{1}{T}\sum_{t = 1}^T\sqrt{C^{\check{\pi}_{R_{-i}}^{(t)}}(\pi^*_{-i})}}\\ &\hspace{3em} \times \rbr{2(HR_{\max})^{1/3}(\epsilon_{\rm S} + 3\epsilon_{\cF})^{1/3} + \sqrt{8\epsilon_{\rm S} + 12\epsilon_{\cF} + 3\epsilon_{\cF, \cF}}}.
        \end{split}
        \end{equation*}
  \item When $\zeta_{2} = \mathtt{OPT}$, by Lemma~\ref{lemma:alg_policy_evaluation}, we know that
        \begin{align*}
          G_{-i}^{(2)}(s_{0}) - V^{\check{\pi}}(s_{0}; R_{-i}) \geq-\sqrt{\epsilon_{\cF}} - 2(HR_{\max})^{1/3}(\epsilon_{\rm S} + 3\epsilon_{\cF})^{1/3}.
        \end{align*}
  \item When $\zeta_{2} = \mathtt{PES}$, by Lemma~\ref{thm:alg_policy_evaluation_performance}
  \begin{align*}
    & G_{-i}^{(2)}(s_0) - V_{1}^{\check{\pi}}(s_0; R_{-i})\geq - H\sqrt{C^{\check{\pi}}(\check{\pi})}\rbr{2(HR_{\max})^{1/3}(\epsilon_{\rm S} + 3\epsilon_{\cF})^{1/3} +\sqrt{ 8\epsilon_{\rm S} + 12 \epsilon_{\cF} + 3\epsilon_{\cF, \cF}}}.
  \end{align*}
\end{itemize}

We now argue that our analysis holds even when the other agents are not truthful. Recall that~$\tilde{\pi}$ is the output policy selected by Algorithm~\ref{alg:vcg_learn} when other agents report $\tilde{r}_{i'}$ and the agent $i$ reports truthfully. Observe that here the decomposition in equation~\eqref{eqn:individual_rationality_decomp} can be written as
\begin{equation*}
  \begin{split}
    \EE_{{\tilde{\pi}}}[u_{i}] &\geq \tilde{G}_{-i}^{(2)}(s_{0}) - V^{\tilde{\pi}}(s_{0}; \tilde{R}_{-i}) + V^{\tilde{\pi}}(s_{0}; r_{i} + \tilde{R}_{-i}) - V^{{\pi}^*_{r_i + \tilde{R}_{-i}}}(s_{0}; r_i + \tilde{R}_{-i}) \\
    &\quad+V^{\tilde{\pi}^*_{-i}}(s_{0}; \tilde{R}_{-i}) - \tilde{G}_{-i}^{(1)}(s_{0}),
  \end{split}
\end{equation*} 
where we recall that $\tilde{R}_{-i} = \sum_{i' \neq i}\tilde{r}_{i'}$, and ${\pi}^*_{r_i + \tilde{R}_{-i}}$ and $\tilde{\pi}^*_{-i}$ maximize $V^{\pi}_{1}(s_{0}; r_i + \tilde{R}_{-i})$ and $V^{\pi}_{1}(s_{0}; \tilde{R}_{-i})$ over $\pi$, respectively. We also let $\tilde{G}_{-i}^{(1)}, \tilde{G}_{-i}^{(2)}$ be the estimates used in Algorithm~\ref{alg:vcg_learn} line~\ref{algline:vcg_pricing} when other agents are reporting untruthfully.

Similar to the previous sections, we bound different terms under difference choices of $\zeta_{1}, \zeta_{2}$.
\begin{itemize}
  \item When $\zeta_{1} = \mathtt{OPT}$, similar to equation~\eqref{eqn:g_i_bound_opt}, we have
        \begin{equation*}
        \begin{split}
          & \tilde{G}_{-i}^{(1)}(s_{0}) - V_{1}^{\tilde{\pi}_{-i}^*}(s_{0}; \tilde{R}_{-i}) \geq -2H^{2}R_{\max}\sqrt{\frac{2\log |\cA|}{T}} - \sqrt{\epsilon_{\cF}} - 2(HR_{\max})^{1/3}(\epsilon_{\rm S} + 3\epsilon_{\cF})^{1/3} \\
          &\hspace{3em} - H\rbr{\frac{1}{T}\sum_{t = 1}^T\rbr{\sqrt{C^{\hat{\pi}_{\tilde{R}_{-i}}^{(t)}}(\tilde{\pi}_{-i}^*)} + \sqrt{C^{\hat{\pi}_{\tilde{R}_{-i}}^{(t)}}(\hat{\pi}_{\tilde{R}_{-i}}^{(t)})}}}
          \\&\hspace{3em}\times \rbr{2(HR_{\max})^{1/3}(\epsilon_{\rm S} + 3\epsilon_{\cF})^{1/3} + \sqrt{8\epsilon_{\rm S} + 12\epsilon_{\cF} + 3\epsilon_{\cF, \cF}}}.
        \end{split}
        \end{equation*}
  \item When $\zeta_{1} = \mathtt{PES}$, similar to equation~\eqref{eqn:g_i_bound_pes}, we have
        \begin{equation*}
        \begin{split}
          & \tilde{G}_{-i}^{(1)}(s_{0}) - V_{1}^{\tilde{\pi}^*_{-i}}(s_{0}; \tilde{R}_{-i}) \geq -2H^{2}R_{\max}\sqrt{\frac{2\log |\cA|}{T}} - \sqrt{\epsilon_{\cF}} - 2(HR_{\max})^{1/3}(\epsilon_{\rm S} + 3\epsilon_{\cF})^{1/3}\\
          &\hspace{3em} - H\rbr{\sqrt{C^{\check{\pi}^{\rm out}_{\tilde{R}_{-i}}}(\check{\pi}_{\tilde{R}_{-i}}^{\rm out})} + \frac{1}{T}\sum_{t = 1}^T \sqrt{C^{\check{\pi}_{\tilde{R}_{-i}}^{(t)}}(\tilde{\pi}_{-i}^*)}}
          \\&\hspace{3em}\times \rbr{2(HR_{\max})^{1/3}(\epsilon_{\rm S} + 3\epsilon_{\cF})^{1/3} + \sqrt{8\epsilon_{\rm S} + 12\epsilon_{\cF} + 3\epsilon_{\cF, \cF}}}.
        \end{split}
        \end{equation*}
  \item When $\zeta_{2} = \mathtt{OPT}$, by Lemma~\ref{lemma:alg_policy_evaluation}, we know
        \begin{align*}
          \tilde{G}_{-i}^{(2)}(s_{0}) - V^{\tilde{\pi}}(s_{0}; \tilde{R}_{-i}) \geq-\sqrt{\epsilon_{\cF}} - 2(HR_{\max})^{1/3}(\epsilon_{\rm S} + 3\epsilon_{\cF})^{1/3}.
        \end{align*}
  \item When $\zeta_{2} = \mathtt{PES}$, by Lemma~\ref{thm:alg_policy_evaluation_performance}
  \begin{align*}
    & \tilde{G}_{-i}^{(2)}(s_0) - V_{1}^{\tilde{\pi}}(s_0; \tilde{R}_{-i})\geq - H\sqrt{C^{\tilde{\pi}}(\tilde{\pi})}\rbr{2(HR_{\max})^{1/3}(\epsilon_{\rm S} + 3\epsilon_{\cF})^{1/3} +\sqrt{ 8\epsilon_{\rm S} + 12 \epsilon_{\cF} + 3\epsilon_{\cF, \cF}}},
  \end{align*} 
  where $\tilde{\pi}$ is the policy that the seller chooses when agent $i$ reports truthfully and the other agents do not.
\end{itemize}

We finally focus on lower bounding $V^{\tilde{\pi}}(s_0; r_i + \tilde{R}_{-i}) - V^{{\pi}^*_{r_i + \tilde{R}_{-i}}}(s_0; r_i + \tilde{R}_{-i})$. Since $\tilde{\pi}$ is the uniform mixture of $\{\tilde{\pi}^{(t)}\}_{t \in [T]}$, we have
\begin{align*}
   &V^{{\pi}^*_{r_i + \tilde{R}_{-i}}}_1(s_0; r_{i} + \tilde{R}_{-i}) -  V^{\tilde{\pi}}_{1}(s_{0}; r_i + \tilde{R}_{-i})\\
   & \qquad = \frac{1}{T}\sum_{t = 1}^T\rbr{ V^{{\pi}^*_{r_i + \tilde{R}_{-i}}}_{1}(s_{0}; r_i + \tilde{R}_{-i}) - V^{\tilde{\pi}^{(t)}}_{1}(s_{0}; r_i + \tilde{R}_{-i})}\\
   & \qquad \leq \frac{1}{T}\sum_{t = 1}^T\rbr{ V^{{\pi}^*_{r_i + \tilde{R}_{-i}}}_{1}(s_{0}; r_i + \tilde{R}_{-i}) - \check{Q}^{(t)}_{1, r_i + \tilde{R}_{-i}}(s_{0}, \tilde{\pi}^{(t)}_1)} + \sqrt{\epsilon_{\cF}} + 2(HR_{\max})^{1/3}(\epsilon_{\rm S} + 3\epsilon_{\cF})^{1/3}
\end{align*} by Lemma~\ref{lemma:alg_policy_evaluation}. By Lemma~\ref{thm:alg_spi}, we know that
\begin{align*}
  &\frac{1}{T}\sum_{t = 1}^T\rbr{ V^{{\pi}^*_{r_i + \tilde{R}_{-i}}}_{1}(s_{0}; r_i + \tilde{R}_{-i}) - \check{Q}^{(t)}_{1, r_i + \tilde{R}_{-i}}(s_{0}, \tilde{\pi}^{(t)}_1)}\leq 2H^{2}R_{\max}\sqrt{\frac{2\log|\cA|}{T}}\\
  &\qquad + H\rbr{\frac{1}{T}\sum_{t = 1}^T\sqrt{C^{\tilde{\pi}^{(t)}}(\pi^*_{r_i + \tilde{R}_{-i}})}}\rbr{2(HR_{\max})^{1/3}(\epsilon_{\rm S} + 3\epsilon_{\cF})^{1/3} + \sqrt{8\epsilon_{\rm S} + 12\epsilon_{\cF} + 3\epsilon_{\cF, \cF}}}.
\end{align*} 
Therefore, we have
\begin{equation}\label{eqn:tilde_subopt_bound}
\begin{split}
  &V^{{\pi}^*_{r_i + \tilde{R}_{-i}}}_1(s_0; r_{i} + \tilde{R}_{-i}) -  V^{\tilde{\pi}}_{1}(s_{0}; r_i + \tilde{R}_{-i})\\
  &\qquad\leq 2H^{2}R_{\max}\sqrt{\frac{2\log|\cA|}{T}} +\sqrt{\epsilon_{\cF}} + 2(HR_{\max})^{1/3}(\epsilon_{\rm S} + 3\epsilon_{\cF})^{1/3} \\
  &\hspace{3em} + H\rbr{\frac{1}{T}\sum_{t = 1}^T\sqrt{C^{\tilde{\pi}^{(t)}}(\pi^*_{r_i + \tilde{R}_{-i}})}}\rbr{2(HR_{\max})^{1/3}(\epsilon_{\rm S} + 3\epsilon_{\cF})^{1/3} + \sqrt{8\epsilon_{\rm S} + 12\epsilon_{\cF} + 3\epsilon_{\cF, \cF}}}.
\end{split}
\end{equation} Flipping the signs yields the final bound.

\paragraph{Truthfulness}\label{subsec:truthfulness}
Similar to above and let $\tilde{r}_{i'}$ be the potentially untruthful reward functions reported by other agents and let $\tilde{r}_{i}$ be the untruthful reward function that the agent $i$ may report. Furthermore, let $\tilde{R}_{-i} = \sum_{i' \neq i} \tilde{r}_{i'}$ and $\tilde{R} = \sum_{i= 1}^n \tilde{r}_{i}$.

Let $\tilde{\pi}$ be the policy chosen by the seller when the agent $i$ is truthful and other agents are possibly non-truthful and $\check{\pi}_{\tilde{R}}$ the policy chosen by Algorithm~\ref{alg:vcg_learn} when both the agent $i$ and other agents are non-truthful. The agents' expected utilities for the two cases are
\begin{align*}
  \EE_{{\tilde{\pi}}}[u_{i}] & = V^{\tilde{\pi}}_{1}(s_{0}; r_{i}) +  \tilde{G}_{-i}^{(2)}(s_{0}) -  \tilde{G}_{-i}^{(1)}(s_{0}),\\
  \EE_{d_{\check{\pi}_{\tilde{R}}}}[u_{i}] & = V^{\check{\pi}_{\tilde{R}}}_{1}(s_{0}; r_{i}) + \tilde{G}_{-i}^{(2), \prime}(s_{0}) -  \tilde{G}_{-i}^{(1), \prime}(s_{0}),
\end{align*} where $\tilde{G}_{-i}^{(2)}(s_{0})$ estimates $V^{\tilde{\pi}}(s_0; \tilde{R}_{-i})$ and $\tilde{G}_{-i}^{(2), \prime}(s_{0})$ estimates $V^{\check{\pi}_{\tilde{R}}}(s_0; \tilde{R}_{-i})$.

Observe that both $\tilde{G}_{-i}^{(1)}(s_{0})$ and $ \tilde{G}_{-i}^{(1), \prime}(s_{0})$ approximate $V^{\tilde{\pi}^*_{-i}}_{1}(s_{0}; \tilde{R}_{-i}) $ using the same algorithm, Algorithm~\ref{alg:spi}. As the algorithm itself does not contain randomness and $\tilde{G}_{-i}^{(1)}(s_{0})$ and $ \tilde{G}_{-i}^{(1), \prime}(s_{0})$ are constructed using the same parameters, the two terms must be equal. Then we have
\begin{align*}
  &\EE_{{\check{\pi}_{\tilde{R}}}}[u_{i}] - \EE_{\tilde{\pi}}[u_{i}] =  V^{\check{\pi}_{\tilde{R}}}_{1}(s_{0}; r_{i}) + \tilde{G}_{-i}^{(2), \prime}(s_{0}) - \rbr{V^{\tilde{\pi}}_{1}(s_{0}; r_{i}) +  \tilde{G}_{-i}^{(2)}(s_{0})}\\
  &\qquad = V^{\check{\pi}_{\tilde{R}}}_{1}(s_{0}; r_{i} + \tilde{R}_{-i}) + \tilde{G}_{-i}^{(2), \prime}(s_{0}) - V^{\check{\pi}_{\tilde{R}}}_{1}(s_{0}; \tilde{R}_{-i}) - \rbr{ V^{\tilde{\pi}}_{1}(s_{0}; r_i + \tilde{R}_{-i}) +  \tilde{G}_{-i}^{(2)}(s_{0}) - V^{\tilde{\pi}}_{1}(s_{0}; \tilde{R}_{-i})}\\
  &\qquad = V^{\check{\pi}_{\tilde{R}}}_{1}(s_{0}; r_{i} + \tilde{R}_{-i}) - V^{{\pi}^*_{r_i + \tilde{R}_{-i}}}_1(s_0; r_{i} + \tilde{R}_{-i}) + \tilde{G}_{-i}^{(2), \prime}(s_{0}) - V^{\check{\pi}_{\tilde{R}}}_{1}(s_{0}; \tilde{R}_{-i})\\
  &\hspace{3em} + V^{{\pi}^*_{r_i + \tilde{R}_{-i}}}_1(s_0; r_{i} + \tilde{R}_{-i}) -  V^{\tilde{\pi}}_{1}(s_{0}; r_i + \tilde{R}_{-i}) +  V^{\tilde{\pi}}_{1}(s_{0}; \tilde{R}_{-i}) -  \tilde{G}_{-i}^{(2)}(s_{0}),
\end{align*} 
where we recall that ${\pi}^*_{r_i + \tilde{R}_{-i}}$ is the maximizer of $V^{{\pi}}_{1}(s_{0}; r_i + \tilde{R}_{-i})$ over $\pi$ (the social welfare maximizing policy when agent $i$ reports truthfully). We then know that 
\[
V^{\check{\pi}_{\tilde{R}}}_{1}(s_{0}; r_{i} + \tilde{R}_{-i}) - V^{{\pi}^*_{r_i + \tilde{R}_{-i}}}_1(s_0; r_{i} + \tilde{R}_{-i})\leq 0
\]and 
\begin{align*}
  &\EE_{{\check{\pi}_{\tilde{R}}}}[u_{i}] - \EE_{{\tilde{\pi}}}[u_{i}]\\
  &\qquad \leq \rbr{\tilde{G}_{-i}^{(2), \prime}(s_{0}) - V^{\check{\pi}_{\tilde{R}}}_{1}(s_{0}; \tilde{R}_{-i})} + \rbr{V^{{\pi}^*_{r_i + \tilde{R}_{-i}}}_1(s_0; r_{i} + \tilde{R}_{-i}) -  V^{\tilde{\pi}}_{1}(s_{0}; r_i + \tilde{R}_{-i})} \\
  &\hspace{3em}+  \rbr{V^{\tilde{\pi}}_{1}(s_{0}; \tilde{R}_{-i}) -  \tilde{G}_{-i}^{(2)}(s_{0})}.
\end{align*}

Let us focus on the middle term first. 
By~\eqref{eqn:tilde_subopt_bound}, we have
\begin{align*}
  &V^{{\pi}^*_{r_i + \tilde{R}_{-i}}}_1(s_0; r_{i} + \tilde{R}_{-i}) -  V^{\tilde{\pi}}_{1}(s_{0}; r_i + \tilde{R}_{-i})\\
  &\qquad\leq 2H^{2}R_{\max}\sqrt{\frac{2\log|\cA|}{T}} +\sqrt{\epsilon_{\cF}} + 2(HR_{\max})^{1/3}(\epsilon_{\rm S} + 3\epsilon_{\cF})^{1/3} \\
  &\hspace{3em} + H\rbr{\frac{1}{T}\sum_{t = 1}^T\sqrt{C^{\tilde{\pi}^{(t)}}(\pi^*_{r_i + \tilde{R}_{-i}})}}\rbr{2(HR_{\max})^{1/3}(\epsilon_{\rm S} + 3\epsilon_{\cF})^{1/3} + \sqrt{8\epsilon_{\rm S} + 12\epsilon_{\cF} + 3\epsilon_{\cF, \cF}}}.
\end{align*}
We state the results conditioned on different values of $\zeta_{2}$ as the bound no longer depends on $\zeta_{1}$.
\begin{itemize}
  \item When $\zeta_{2} = \mathtt{OPT}$, by Lemma~\ref{lemma:alg_policy_evaluation}, we have
        \begin{align*}
          V^{\tilde{\pi}}_{1}(s_{0}; \tilde{R}_{-i}) -  \tilde{G}_{-i}^{(2)}(s_{0}) \leq \sqrt{\epsilon_{\cF}} + 2(HR_{\max})^{1/3}(\epsilon_{\rm S} + 3\epsilon_{\cF})^{1/3},
        \end{align*} and by Lemma~\ref{thm:alg_policy_evaluation_performance},
        \begin{multline*}
          \tilde{G}_{-i}^{(2), \prime}(s_{0}) - V^{\check{\pi}_{\tilde{R}}}_{1}(s_{0}; \tilde{R}_{-i}) \\
          \leq H\sqrt{C^{\check{\pi}_{\tilde{R}}}(\check{\pi}_{\tilde{R}})}\rbr{2(HR_{\max})^{1/3}(\epsilon_{\rm S} + 3\epsilon_{\cF})^{1/3} + \sqrt{8\epsilon_{\rm S} + 12\epsilon_{\cF} + 3\epsilon_{\cF, \cF}}}.
        \end{multline*}
  \item When $\zeta_{2} = \mathtt{PES}$, by Lemma~\ref{thm:alg_policy_evaluation_performance},
        \begin{align*}
          &V^{\tilde{\pi}}_{1}(s_{0}; \tilde{R}_{-i}) -  \tilde{G}_{-i}^{(2)}(s_{0}) \leq H\sqrt{C^{\tilde{\pi}}(\tilde{\pi})}\rbr{2(HR_{\max})^{1/3}(\epsilon_{\rm S} + 3\epsilon_{\cF})^{1/3} + \sqrt{8\epsilon_{\rm S} + 12\epsilon_{\cF} + 3\epsilon_{\cF, \cF}}},
        \end{align*} and by Lemma~\ref{lemma:alg_policy_evaluation},
        \begin{align*}
          \tilde{G}_{-i}^{(2), \prime}(s_{0}) - V^{\check{\pi}_{\tilde{R}}}_{1}(s_{0}; \tilde{R}_{-i}) \leq \sqrt{\epsilon_{\cF}} + 2(HR_{\max})^{1/3}(\epsilon_{\rm S} + 3\epsilon_{\cF})^{1/3}.
        \end{align*}
\end{itemize}
Combining the terms completes the proof.
\end{proof}



\section{Supporting Lemmas}\label{sec:supporting_lemmas}
In this section, we provide detailed proofs of supporting lemmas used in Section~\ref{sec:proof_of_main_result}.

\subsection{Proofs for Algorithm~\ref{alg:policy_evaluation}} \label{sec:proofs_of_alg_policy_evaluation}
Previous work has shown that the estimate of the value function $f^{\pi}$ is the exact value function of an induced MDP that shares the same state space, action space, and transition kernel as $\cM$, only with slightly perturbed reward functions~\citep{cai2020provably,uehara2021pessimistic, xie2021bellman, zanette2021provable}. More precisely, let $r$ be the input reward for Algorithm~\ref{alg:policy_evaluation}, $\pi$ the input policy, and $f^{\pi}$ the output. Let $\cM_{f^{\pi}}$ be the induced MDP. We formally state the result below.
\begin{lemma}\label{lemma:policy_evaluation_induced_mdp}
  For any input policy $\pi$ (not necessarily in $\Pi_{\rm SPI}$) and input reward function $r$, Algorithm~\ref{alg:policy_evaluation} returns a function $f^{\pi}$ such that $f^{\pi}$ is the $Q$-function of the policy $\pi$ under the induced MDP $\cM_{f^{\pi}}$, given by
  \begin{equation}\label{eqn:induced_MDP}
    \cM_{f^{\pi}} = (\cS, \cA, H, \cP, r_{f^{\pi}}),
  \end{equation} 
  where $r_{f^{\pi}, h} = r_{h} + f_{h}^{\pi} - \cT_{h, r}^{\pi}f_{h + 1}^{\pi}$. In other words, $f^{\pi}(\cdot, \cdot) = Q^{\pi}(\cdot, \cdot; r_{f^{\pi}})$.
\end{lemma}
\begin{proof}
  See Section C.1 in~\citet{zanette2021provable} for a detailed proof.
\end{proof}
We immediately have the following corollary.
\begin{corollary}\label{corollary:policy_evaluation_error}
  Let $f^{\pi}$ be any one of the two functions returned by Algorithm~\ref{alg:policy_evaluation} for any input policy $\pi$ (not necessarily in $\Pi_{\rm SPI}$) and any input reward function $r$. 
  Then, for all $h \in [H]$, we have
  \[
    \abr{f_{h}^{\pi}(s, a) - Q_{h}^{\pi}(s, a; r)} \leq \sum_{h' = h}^{H}\EE_{(S_{h'}, A_{h'}) \sim \pi | (s, a)}\sbr{\abr{f_{h}^{\pi} - \cT_{h, r}^{\pi}f_{h + 1}^{\pi}}}.
  \]
\end{corollary}
\begin{proof}
  By definition of the $Q$-function, we have
  \[
  \begin{aligned}
    f_{h}^{\pi}(s, a) - Q_{h}^{\pi}(s, a; r) 
    &= Q_{h}^{\pi}(s, a; r_{f^{\pi}}) - Q_{h}^{\pi}(s, a; r) \\
    &= \sum_{h' = h}^H \EE_{(S_{h'}, A_{h'}) \sim \pi | (s, a)}[r_{h}(S_{h'}, A_{h'}) - r_{f^{\pi}, h}(S_{h'}, A_{h'})].
    \end{aligned}    
  \] 
  Recalling the definition of $r_{f^{\pi}}$ in equation \eqref{eqn:induced_MDP} and using Jensen's inequality concludes the proof.
\end{proof}

We proceed to show that Algorithm~\ref{alg:policy_evaluation} is approximately optimistic/pessimistic and bounding the estimation error of its outputs. We begin with the proof of Lemma~\ref{lemma:alg_policy_evaluation}.
\begin{proof}[Proof of Lemma~\ref{lemma:alg_policy_evaluation}]
  We start by upper bounding two auxiliary terms. Let $f^{\pi, *}_{r} \in \cF$ be the best approximation of $Q^{\pi}(\cdot, \cdot; r)$, as defined in Assumption~\ref{assumption:realizability}. By Jensen's inequality, we have
  \begin{align*}
    |f^{\pi, *}_{1, r}(s_0, \pi_{1}) - Q_{1}^{\pi}(s_0, \pi_{1}; r)| &\leq \EE_{a \sim \pi_{1}(\cdot | s_0)}[|f^{\pi, *}_{1, r}(s_0, \pi_{1}) - Q_{1}^{\pi}(s_0, \pi_{1}; r)|] \leq \sqrt{\epsilon_{\cF}}.
  \end{align*} 
  Additionally, using Lemma~\ref{lemma:q_star_small_cE} we know that, conditioned on the event $\cG(\Pi_{\rm SPI})$, for all $h \in [H]$ we have $\cE_{h, r}(f^{\pi, *}_{r}, \pi; \cD) \leq 2\epsilon_{\rm S} + 6\epsilon_{\cF}$.

  We then consider $\check{Q}_{r}^{\pi}$. By~\eqref{eqn:policy_evaluation_regularized_loss}, we know that
  \begin{align*}
    \check{Q}^{\pi}_{1, r}(s_0, \pi) + \lambda\sum_{h = 1}^{H}\cE_{h, r}(\check{Q}_{r}^{\pi}, \pi; \cD) &\leq f^{\pi, *}_{1, r}(s_0, \pi) + \lambda \sum_{h = 1}^{H}\cE_{h, r}(f^{\pi, *}_{r}, \pi; \cD) \\
    &\leq Q^{\pi}_{1}(s_0, \pi; r) + |f^{\pi, *}_{1, r}(s_0, \pi_{1}) - Q_{1}^{\pi}(s_0, \pi_{1}; r)| +2\lambda H\epsilon_{\rm S} + 6\lambda H\epsilon_{\cF}\\
    &\leq Q^{\pi}_{1}(s_0, \pi_{1}; r) + \sqrt{\epsilon_{\cF}} + 2\lambda H\epsilon_{\rm S} + 6\lambda H\epsilon_{\cF}.
  \end{align*} 
  Similarly for $\hat{Q}_{r}^{\pi}$, by~\eqref{eqn:policy_evaluation_regularized_loss}, we have
    \begin{align*}
    \hat{Q}^{\pi}_{1, r}(s_0, \pi) - \lambda\sum_{h = 1}^{H}\cE_{h, r}(\hat{Q}_{r}^{\pi}, \pi; \cD) &\geq f^{\pi, *}_{1, r}(s_0, \pi) - \lambda \sum_{h = 1}^{H}\cE_{h, r}(f^{\pi, *}_{r}, \pi; \cD) \\
    &\geq Q^{\pi}_{1}(s_0, \pi; r) - |f^{\pi, *}_{1, r}(s_0, \pi_{1}) - Q_{1}^{\pi}(s_0, \pi_{1}; r)| -2\lambda H\epsilon_{\rm S} - 6\lambda H\epsilon_{\cF}\\
    &\geq Q^{\pi}_{1}(s_0, \pi_{1}; r) - \sqrt{\epsilon_{\cF}} - 2\lambda H\epsilon_{\rm S} - 6\lambda H\epsilon_{\cF},
  \end{align*} 
  thus completing the proof.
\end{proof}

We prove that the action-value functions returned by Algorithm~\ref{alg:policy_evaluation} are sufficiently good estimates.
\begin{proof}[Proof of Lemma~\ref{thm:alg_policy_evaluation_performance}]
  By Corollary~\ref{corollary:policy_evaluation_error}, we have
  \begin{align*}
    & \hat{Q}_{1, r}^{\pi}(s_{0}, \pi_{1})  - Q_{1}^{\pi}(s_{0}, \pi_{1}; r)\leq \abr{\sum_{h = 1}^H \EE_{\pi}\sbr{\hat{Q}_{h, r}^{\pi} - \cT_{h, r}^{\pi}\hat{Q}_{h + 1, r}^{\pi}}}, \\
    &  Q_{1}^{\pi}(s_{0}, \pi_{1}; r)  - \check{Q}_{1, r}^{\pi}(s_{0}, \pi_{1})\leq \abr{\sum_{h = 1}^H \EE_{\pi}\sbr{\check{Q}_{h, r}^{\pi} - \cT_{h, r}^{\pi}\check{Q}_{h + 1, r}^{\pi}}}.
  \end{align*} 
  Since the differences share similar forms, we can without loss of generality only consider $\hat{Q}_{r}^{\pi}$. Recall the definition of $C^\pi(\nu)$, given in Definition~\ref{defn:distribution_shift_coef}.
  We have
    \begin{equation}\label{eqn:thm_alg_policy_evaluation_performance_decomp}
    \begin{split}
      \abr{\sum_{h = 1}^H \EE_{\pi}\sbr{\check{Q}_{h, r}^{\pi} - \cT_{h, r}^{\pi}\check{Q}_{h + 1, r}^{\pi}}} &\leq \sum_{h = 1}^H \EE_{\pi}\sbr{\nbr{\check{Q}_{h, r}^{\pi} - \cT_{h, r}^{\pi}\check{Q}_{h + 1, r}^{\pi}}}\\
      & \leq \sqrt{C^\pi(\pi)}\sum_{h = 1}^H\EE_{\mu_h}\sbr{\nbr{\check{Q}_{h, r}^{\pi} - \cT_{h, r}^{\pi}\check{Q}_{h + 1, r}^{\pi}}},
    \end{split}
    \end{equation} 
    where the first inequality is by Cauchy-Schwarz, the second inequality by the definition of $C^\pi({\pi})$, which is the shorthand notation for $C^{\pi}(d_{\pi})$. Similar to the proof of Lemma~\ref{lemma:alg_policy_evaluation}, let $f^{\pi, *}_{r}$ be the best approximation of $Q^{\pi}(\cdot, \cdot; r)$ as defined in Assumption~\ref{assumption:realizability}. Then
    \begin{align*}
        \lambda \sum_{h = 1}^H\cE_{h, r}(\check{Q}_r^\pi, \pi; \cD) \leq f^{\pi, *}_{1, r}(s_{0}, \pi_1) - \check{Q}_{1, r}^\pi(s_0, \pi_1) + 2\lambda H\epsilon_{\rm S} + 6\lambda H \epsilon_{\cF}.
    \end{align*}
    Since $f^{\pi, *}_{r}, \check{Q}_{1, r}^{\pi} \in \cF$, we have $f^{\pi, *}_{r}, \check{Q}_{1, r}^\pi \in [-HR_{\max}, HR_{\max}]$ and thus
    \begin{align*}
        &\sum_{h = 1}^H\cE_{h, r}(\check{Q}_r^\pi, \pi; \cD) \leq \frac{2HR_{\max}}{\lambda} +  2H\epsilon_{\rm S} + 6 H \epsilon_{\cF}.
    \end{align*}
    By Corollary~\ref{lemma:small_cE_small_bellman_error}, conditioned on $\cG(\Pi_{\rm SPI})$, we have 
    \begin{align*}
        \sum_{h = 1}^H\EE_{\mu_h}\sbr{\|\check{Q}_{h, r}^\pi - \cT_{h, r}^\pi \check{Q}_{h + 1, r}^\pi\|^{2}} &\leq 2\sum_{h = 1}^H\cE_{h, r}(\check{Q}_r^\pi, \pi; \cD) + 4H\epsilon_{\rm S} + 3H\epsilon_{\cF, \cF}\\
        &\leq \frac{4HR_{\max}}{\lambda} + 8H\epsilon_{\rm S} + 12 H \epsilon_{\cF} + 3H\epsilon_{\cF, \cF}.
    \end{align*} Plugging the bound back into~\eqref{eqn:thm_alg_policy_evaluation_performance_decomp} and applying Cauchy-Schwarz inequality gives us
    \begin{align*}
        \abr{\sum_{h = 1}^H \EE_{\pi}\sbr{\check{Q}_{h, r}^{\pi} - \cT_{h, r}^{\pi}\check{Q}_{h + 1, r}^{\pi}}} &\leq \sqrt{H}\sqrt{C^\pi(\pi)}\sqrt{\frac{4HR_{\max}}{\lambda} + 8H\epsilon_{\rm S} + 12 H \epsilon_{\cF} + 3H\epsilon_{\cF, \cF}}\\
        &=H \sqrt{C^\pi(\pi)}\sqrt{\frac{4R_{\max}}{\lambda} + 8\epsilon_{\rm S} + 12 \epsilon_{\cF} + 3\epsilon_{\cF, \cF}}.
    \end{align*} 
    Setting $\lambda = \rbr{\frac{R_{\max}}{H^{2}(\epsilon_{\rm S} + 3\epsilon_{\cF})^{2}}}^{1/3}$ and using $\sqrt{a + b} \leq \sqrt{a} + \sqrt{b}$ for $a, b \in \RR_{\geq 0}$ completes the proof.
\end{proof}

\subsection{Proofs for Algorithm~\ref{alg:spi}}
\label{sec:proofs_of_alg_spi}
We now turn to analyzing the policies selected in Algorithm~\ref{alg:spi}. In particular, we focus on the mirror descent-style updates given in \eqref{eqn:ospi_policy_update} and \eqref{eqn:pspi_policy_update}. We start by defining an abstract version of the procedure in Algorithm~\ref{alg:spi}.
\begin{definition}\label{defn:mirror_descent_regret}
  Consider the following procedure. For any $t \in [T]$:
  \begin{enumerate}
    \item Let $f^{(t)}\in \cF$ be an arbitrary function in the function class.
    \item Let $\pi^{(t + 1)}_{h}(a | s)\, \propto\,\pi^{(t)}_{h}(a | s)\exp\rbr{\eta f_{h}^{(t)}(s, a)}$ for all $(s, a) \in \cS \times \cA,\,h \in [H]$.
  \end{enumerate}
\end{definition}
Recall that $\EE_{a \in \cA}\sbr{\log\pi_{h}(a | s)} = \sum_{a \in \cA} \pi_{h}(a | s)\log\pi_{h}(a|s)$ for all $\pi,h,$ and $s$. We continue with a standard analysis of the regret of actor-critic algorithms.

\begin{lemma}\label{lemma:mirror_descent_bound}
    For any $\pi$ (not necessarily in $\Pi_{\rm SPI}$), for all $h\in[H]$ and $s \in \cS$, setting $\eta = \sqrt{\frac{\log|\cA|}{2H^{2}R_{\max}^2T}}$ in the procedure defined in~\ref{defn:mirror_descent_regret} ensures that
    \begin{equation*}
        \sum_{t = 1}^T\langle \pi_h(\cdot | s) - \pi_h^{(t)}(\cdot | s), f_h^{(t)}(s, \cdot) \rangle \leq 2HR_{\max}\sqrt{2T\log|\cA|}.
    \end{equation*}
\end{lemma}
\begin{proof}
By a direct application of Lemma C.3 of~\citet{xie2021bellman}, we know that even for policies not in $\Pi_{\rm SPI}$ (as we are effectively performing mirror descent over the probability simplex with the $\mathrm{KL}$ penalty) we have
  \begin{align*}
    \sum_{t = 1}^T \langle \pi_h(\cdot | s) - \pi^{(t)}_h(\cdot | s), f_h^{(t)}(s, \cdot)\rangle \leq \sum_{t = 1}^T \langle \pi_{h}^{(t + 1)} - \pi_h^{(t)}(\cdot | s), f_h^{(t)}(s, \cdot)\rangle - \frac{1}{\eta}\EE_{a \sim \pi_h^{(1)}}\sbr{\log \pi_h^{(1)}(a | s)},
  \end{align*} 
  where $\eta$ is the stepsize. From the proof of Lemma C.4 in~\citet{xie2021bellman}, we further note that for any $\pi \in \pi$, $h \in [H]$, $s \in \cS$, and $t\in[T]$ we have
  \begin{align*}
    \langle \pi_h(\cdot | s) - \pi_h^{(t)}(\cdot | s), f_h^{(t)}(s, \cdot) \rangle &\leq \|f_{h}^{(t)}(s, \cdot)\|_{\infty}\sqrt{2\eta\langle \pi_h(\cdot | s) - \pi_h^{(t)}(\cdot | s), f_h^{(t)}(s, \cdot) \rangle }.
  \end{align*} 
  Recalling that all $f_{h} \in \cF_{h}$ are bounded by $HR_{\max}$, we know that $\langle \pi_h(\cdot | s) - \pi_h^{(t)}(\cdot | s), f_h^{(t)}(s, \cdot) \rangle \leq 2\eta H^{2}R_{\max}^{2}$. Following the proof in Section C.1 in~\citet{xie2021bellman} completes our proof.
\end{proof}

With the observations above, we proceed with proving Lemma~\ref{thm:alg_spi}.
\begin{proof}[Proof of Lemma~\ref{thm:alg_spi}]
We analyze the pessimistic estimate and note that the analysis is similar for the other part. Let $\check{\pi}_r^{(t)}$ be the policy iterate of Algorithm~\ref{alg:spi} and $\check{Q}_r^{(t)}$ the corresponding value function estimate. We know that
  \begin{align*}
      &V_{1}^{\pi}(s_0; r) - \frac{1}{T}\sum_{t =1}^T\check{Q}_{1, r}^{(t)}(s_0, \check{\pi}^{(t)}_{1, r}) = \frac{1}{T}\sum_{t =1}^T\rbr{Q_{1}^{\pi}(s_0, \pi_{1}; r) - \check{Q}_{1, r}^{(t)}(s_0, \check{\pi}^{(t)}_{1, r})}\\
      &\qquad \leq \frac{1}{T}\sum_{t =1}^T\sum_{h = 1}^H\EE_{\pi}\sbr{\langle \check{Q}^{(t)}_{h, r}(s_h, \cdot), \pi_h(\cdot | s_h) - \check{\pi}^{(t)}_{h, r}(\cdot | s_h)\rangle} + \abr{\frac{1}{T}\sum_{t = 1}^T\sum_{h = 1}^H \EE_{\pi}\sbr{\check{Q}^{(t)}_{h, r} - \cT_{h, r}^{\check{\pi}_{r}^{(t)}} \check{Q}_{h + 1, r}^{(t)}}},
  \end{align*} 
  where the inequality is by a standard argument in episodic reinforcement learning (see, for example, Lemma A.1 in~\citet{jin2021pessimism} or Section B.1 in~\citet{cai2020provably}). By Lemma~\ref{lemma:mirror_descent_bound}, we know that when $\eta = \sqrt{\frac{\log |\cA|}{2H^2R_{\max}^2 T}}$, we have
  \begin{align*}
      &\frac{1}{T}\sum_{t =1}^T\sum_{h = 1}^H\EE_{\pi}\sbr{\langle \check{Q}^{(t)}_{h, r}(s_h, \cdot), \pi_h(\cdot | s_h) - \check{\pi}^{(t)}_{h, r}(\cdot | s_h)\rangle} \leq 2H^2R_{\max}\sqrt{\frac{2\log | \cA |}{T}}.
  \end{align*} 
  For all $t \in [T]$, similar to the proof of Lemma~\ref{thm:alg_policy_evaluation_performance}, when $\lambda = \rbr{\frac{R_{\max}}{H^{2}(\epsilon_{\rm S} + 3\epsilon_{\cF})^{2}}}^{1/3}$, we have
  \begin{align*}
    \abr{\sum_{h = 1}^H \EE_{\pi}\sbr{\check{Q}^{(t)}_{h, r} - \cT_{h, r}^{\check{\pi}_{r}^{(t)}} \check{Q}_{h + 1, r}^{(t)}}}&\leq H\sqrt{C^{\check{\pi}_r^{(t)}}(\pi)}\rbr{2(HR_{\max})^{1/3}(\epsilon_{\rm S} + 3\epsilon_{\cF})^{1/3} +\sqrt{8\epsilon_{\rm S} + 12\epsilon_{\cF} + 3\epsilon_{\cF, \cF}}}.
  \end{align*} 
  Notice that the distribution shift coefficient is changed from $C^\pi(\pi)$ to $C^{\check{\pi}_r^{(t)}}(\pi)$, as the policy specific Bellman operator $\cT$ is now induced by policy $\check{\pi}_r^{(t)}$ rather than $\pi$.
  Taking the average over $t$ and applying the triangle inequality give us
  \begin{align*}
    &\abr{\frac{1}{T}\sum_{t = 1}^T\sum_{h = 1}^H \EE_{\pi}\sbr{\check{Q}^{(t)}_{h, r} - \cT_{h, r}^{\check{\pi}_{r}^{(t)}} \check{Q}_{h + 1, r}^{(t)}}}\\
    &\qquad\leq H\rbr{\frac{1}{T}\sum_{t = 1}^T \sqrt{C^{\check{\pi}^{(t)}_r}({\pi})}}\rbr{2(HR_{\max})^{1/3}(\epsilon_{\rm S} + 3\epsilon_{\cF})^{1/3} +\sqrt{ 8\epsilon_{\rm S} + 12\epsilon_{\cF} + 3\epsilon_{\cF, \cF}}}.
  \end{align*} 
  Combining the bounds, we have
  \begin{align*}
    &V_{1}^{\pi}(s_0; r) - \frac{1}{T}\sum_{t =1}^T\check{Q}_{1, r}^{(t)}(s_0, \check{\pi}^{(t)}_{1, r})\leq 2H^2R_{\max}\sqrt{\frac{2\log|\cA|}{T}}\\
    &\qquad  + H\rbr{\frac{1}{T}\sum_{t = 1}^T \sqrt{C^{\check{\pi}^{(t)}_r}({\pi})}}\rbr{2(HR_{\max})^{1/3}(\epsilon_{\rm S} + 3\epsilon_{\cF})^{1/3} +\sqrt{ 8\epsilon_{\rm S} + 12\epsilon_{\cF} + 3\epsilon_{\cF, \cF}}},
  \end{align*}
  which completes the proof.
\end{proof}




\section{Concentration Analysis} \label{sec:concentration_analysis}

In this section, we prove the concentration lemmas used in Section~\ref{sec:proof_of_main_result}.

\subsection{Proof of Lemma~\ref{lemma:gyorfi_variant_bellman_error_concentration}}
\label{subsec:proof_of_lemma_gyorfi_variant}

We start by including a minor adaptation of a useful result from~\citet{gyorfi2002distribution}.
\begin{theorem}[Adaptation of Theorem 11.6 from~\citet{gyorfi2002distribution}]
\label{thm:11_6_gyorfi}
Let $B \geq 1$ and let $\cG$ be a class of functions $g: \RR^{d} \to [0, B]$. Let $Z_{1}, Z_{2}, \ldots, Z_{K}$ be i.i.d.~$\RR^{d}$-valued random variables. Assume $\alpha > 0$, $0 < \epsilon < 1$, and $K \geq 1$. Then
\[
  \Pr\rbr{\sup_{g \in \cG}\frac{\frac{1}{K}\sum_{j = 1}^{K}g(Z_{j}) - \EE[Z_{j}]}{\alpha + \frac{1}{K}\sum_{j = 1}^{K}g(Z_{j}) + \EE[Z_{j}]} > \epsilon} \leq 4\cN_{\infty}\rbr{\frac{\alpha\epsilon}{5}, \cG}\exp\rbr{-\frac{3\epsilon^{2}\alpha K}{40B}}.
\]
\end{theorem}
\begin{proof}
  By Theorem 11.6 from~\citet{gyorfi2002distribution}, we know that
  \begin{align*}
    \Pr\rbr{\sup_{g \in \cG}\frac{\frac{1}{K}\sum_{j = 1}^{K}g(Z_{j}) - \EE[Z_{j}]}{\alpha + \frac{1}{K}\sum_{j = 1}^{K}g(Z_{j}) + \EE[Z_{j}]} > \epsilon} \leq 4\EE\sbr{\cN_{1}\rbr{\frac{\alpha\epsilon}{5}, \cG, \{Z_{j}\}_{j = 1}^{K}}}\exp\rbr{-\frac{3\epsilon^{2}\alpha K}{40B}},
  \end{align*} where $\cN_{1}\rbr{\frac{\alpha\epsilon}{5}, \cG, \{Z_{j}\}_{j = 1}^{K}}$ is the cardinality of the smallest set of functions $\{g^{l}\}_{l = 1}^{L}$ such that for all $g \in \cG$ there exists some $l \in [L]$ where
  \[
    \frac{1}{K}\sum_{j = 1}^K \left|g(Z_{j}) - g^{l}(Z_{j})\right| \leq \frac{\alpha \epsilon}{5}.
  \]
  See Section 11.4 from~\citet{gyorfi2002distribution} for a detailed proof of the statement above. We then show that for any $\{Z_{j}\}_{j=1 }^{K}$, $\cN_{1}\rbr{\frac{\alpha\epsilon}{5}, \cG, \{Z_{j}\}_{j = 1}^{K}} \leq\cN_{\infty}\rbr{\frac{\alpha\epsilon}{5}, \cG}$. Let $\{\tilde{g}^{l}\}_{l = 1}^{L}$ be an $\frac{\alpha\epsilon}{5}$-covering of $\cG$ with respect to the $\ell_{\infty}$-norm. We then know that for any $g \in \cG$, there exists some $l \in [L]$ such that
  \begin{align*}
    \frac{1}{K}\sum_{j = 1}^{K}|g(Z_{j}) - \Tilde{g}^{l}(Z_{j})| \leq \frac{1}{K}\sum_{j = 1}^{K}\frac{\alpha\epsilon}{5} = \frac{\alpha\epsilon}{5}.
  \end{align*} Therefore $\{\tilde{g}^{l}\}_{l = 1}^{L}$ satisfies the requirement above, concluding our proof.
\end{proof}
  Let $h \in [H], r \in \Tilde{\cR}$ be arbitrary and fixed.
  First, we show
  \begin{align*}
    &\Pr\Bigl(\exists f, f' \in \cF, \pi \in \Pi: \EE_{\mu_h}\sbr{\|f_{h} - \cT_{h, r}^{\pi}f'_{h + 1}\|^2} - \cL_{h, r}(f_{h}, f'_{h + 1}, \pi; \cD) + \\
    &\hspace{12em}\cL_{h, r}(\cT_{h, r}^{\pi}f'_{h + 1}, f'_{h + 1}, \pi; \cD) \geq \epsilon\bigl(\alpha + \beta + \EE_{\mu_h}\sbr{\|f_{h} - \cT_{h, r}^{\pi}f'_{h + 1}\|^2}\bigr)\Bigr)\\
    &\qquad \leq 14\rbr{\cN_{\infty}\rbr{\frac{\epsilon\beta}{140HR_{\max}}, \cF}}^{2}\cN_{\infty, 1}\rbr{\frac{\epsilon\beta}{140H^{2}R^{2}_{\mathrm{max}}}, \Pi}\exp\rbr{-\frac{\epsilon^{2}(1 - \epsilon)\alpha K}{214(1 + \epsilon)H^{4}R_{\max}^{4}}}.
  \end{align*} for all $\alpha, \beta > 0$, $0 < \epsilon \leq 1/2$.

  Let $Z$ be the random vector $(s_{h}, a_{h}, r_{h}(s_h, a_h), s_{h + 1})$ where $(s_{h}, a_{h}, s_{h + 1}) \sim \mu_{h}$. Let $Z_j$ be its realization for any $j \in [K]$ drawn independently from $\cD_{h}$. For any $f, f' \in \cF$, and $\pi \in \Pi$, we further define the random variable
  \begin{equation*}
          g_{f, f'}^{\pi}(Z) = (f_{h}(s_{h}, a_{h}) - r_{h} - f'_{h +  1}(s_{h + 1}, \pi_{h + 1}))^{2} -
          (\cT_{h, r}^{\pi}f'_{h + 1}(s_{h}, a_{h}) - r_{h} - f'_{h +  1}(s_{h + 1}, \pi_{h + 1}))^{2},
  \end{equation*} 
  and $g_{f, f'}^{\pi}(Z_{j})$ its empirical counterpart evaluated on $Z$'s realization, $Z_{j}$. We begin by showing some basic properties of the random variable $g_{f, f'}^{\pi}(Z)$. Recall that by definition of the Bellman evaluation operator
  \begin{equation}
    \cT_{h, r}^{\pi}f'_{h + 1}(s_{h}, a_{h}) = \EE_{\cP}\sbr{r_{h} + f'_{h +  1}(s_{h + 1}, \pi_{h + 1}) | s_{h}, a_{h}}.
  \end{equation} Since $\cT_{h, r}^{\pi}f_{h + 1}(s_{h}, a_{h}) = \EE_{\mu_{h}}\sbr{r_{h} + f'_{h +  1}(s_{h + 1}, \pi_{h + 1}) | s_{h}, a_{h}}$, by the law of total probability
  \begin{align*}
    \EE_{Z \sim \mu_h}&[g_{f, f'}^{\pi}(Z)] \\
    & = {\EE_{s_h, a_h \sim \mu_h}}\Big[{\EE_{s_{h + 1} \sim \mu_h | s_h, a_h}} [(f_{h}(s_{h}, a_{h}) - r_{h} - f'_{h +  1}(s_{h + 1}, \pi_{h + 1}))^{2} -\\
    & \hspace{9em}(\cT_{h, r}^{\pi}f'_{h + 1}(s_{h}, a_{h}) - r_{h} - f'_{h +  1}(s_{h + 1}, \pi_{h + 1}))^{2} | s_{h}, a_{h}]\Big]\\
    & =  {\EE_{\mu_h}}\Big[{\EE}_{s_{h + 1} \sim \mu_h | s_h, a_h}[(f_{h}(s_{h}, a_{h}) +\cT_{h, r}^{\pi}f'_{h + 1}(s_{h}, a_{h}) - 2( r_{h} + f'_{h +  1}(s_{h + 1}, \pi_{h + 1}))) \times \\
    & \hspace{21em}(f_{h}(s_{h}, a_{h}) - \cT_{h, r}^{\pi}f'_{h + 1}(s_{h}, a_{h})) | s_{h}, a_{h}]\Big]\\
    & = \EE_{\mu_h}\sbr{\|f_{h}(s_{h}, a_{h}) - \cT_{h, r}^{\pi}f'_{h + 1}(s_{h}, a_{h})\|^{2}}.
  \end{align*}
  Additionally, recalling that $r_{h} \in [-R_{\max}, R_{\max}]$, $f'_{h + 1} \in [-(H - h)R_{\max}, (H - h)R_{\max}]$, $f_{h} \in [-(H - h + 1)R_{\max}, (H - h + 1)R_{\max}]$, we know that $g_{f, f'}^{\pi}(Z) \in [- 16H^2R_{\max}^2, 16H^2R_{\max}^2]$. Lastly, notice that
  \begin{equation}\label{eqn:g_var_bound}
    \begin{split}
      \Var(g_{f, f'}^{\pi}(Z)) &\leq \EE[(g_{f, f'}^{\pi}(Z))^{2}]\\
      &=  {\EE}\Big[{\EE}[(f_{h}(s_{h}, a_{h}) +\cT_{h, r}^{\pi}f'_{h + 1}(s_{h}, a_{h}) - 2( r_{h} + f'_{h +  1}(s_{h + 1}, \pi_{h + 1})))^{2} \times \\
      & \hspace{16em}(f_{h}(s_{h}, a_{h}) - \cT_{h, r}^{\pi}f'_{h + 1}(s_{h}, a_{h}))^{2} | s_{h}, a_{h}]\Big]\\
      &\leq \EE[16H^{2}R_{\max}^{2}(f_{h}(s_{h}, a_{h}) - \cT_{h, r}^{\pi}f'_{h + 1}(s_{h}, a_{h}))^{2}] = 16H^{2}R_{\max}^{2}\EE[g_{f, f'}^{\pi}(Z)],
    \end{split}
  \end{equation} where for the last inequality we noticed that $f_{h}(s_{h}, a_{h}) +\cT_{h, r}^{\pi}f'_{h + 1}(s_{h}, a_{h}) - 2( r_{h} + f'_{h +  1}(s_{h + 1}, \pi_{h + 1}))$ is bounded by $[-4HR_{\max}, 4HR_{\max}]$.

  Our ensuing proof largely follows the structure of Section 11.5 of~\citet{gyorfi2002distribution} and we reproduce the proof below for completeness. Let $\alpha, \beta > 0$ and $0 < \epsilon \leq \frac{1}{2}$ be arbitrary and fixed constants. We now proceed with the proof.

  {\noindent \bf Symmetrization by Ghost Sample.} Consider some $(f_{n}, f'_{n}, \pi_{n}) \in \cF \times \cF \times \Pi$ depending on $\{Z_{j}\}_{j = 1}^{K}$ such that
  \[
    \EE[g_{f_{n}, f'_{n}}^{\pi_{n}}(Z) | \{Z_{j}\}_{j = 1}^{K}] - \frac{1}{K}\sum_{j = 1}^{K}g_{f_{n}, f'_{n}}^{\pi_{n}}(Z_{j}) \geq \epsilon(\alpha + \beta + \EE[g_{f_{n}, f'_{n}}^{\pi_{n}}(Z) | \{Z_{j}\}_{\tau = 1}^{K}]),
  \] if such $(f_{n}, f'_{n}, \pi_{n})$ exists. If not, choose some arbitrary $(f_{n}, f'_{n}, \pi_{n})$. As a shorthand notation, let $g_{n} = g_{f_{n}, f'_{n}}^{\pi_{n}}$. Finally, introduce ghost samples $\{Z_{j}'\}_{j = 1}^{K} \sim \mu_{h}$, drawn i.i.d.~from the same distribution as $\{Z_{j}\}_{j= 1}^{K}$. Recalling that the variance of $g_{n}$ is bounded by $16\EE[g_{n}(Z)]$, by Chebyshev's inequality we have
  \begin{align*}
    &\Pr\biggl(\EE[g_{n}(Z) | \{Z_{j}\}_{j = 1}^{K}] - \frac{1}{K}\sum_{j = 1}^{K}g_{n}({Z_{j}'}) \geq \frac{\epsilon}{2}(\alpha + \beta) + \frac{\epsilon}{2}\EE[g_{n}(Z) | \{Z_{j}\}_{j = 1}^{K}] | \{Z_{j}\}_{j = 1}^{K}\biggr)\\
    &\qquad \leq  \frac{\Var(g_{n}(Z) | \{Z_j\}_{j = 1}^{K})}{K(\frac{\epsilon}{2}(\alpha + \beta) + \frac{\epsilon}{2}\EE[g_{n}(Z) | \{Z_j\}_{j = 1}^{K}] )^{2}}\\
    &\qquad \leq \frac{16H^{2}R_{\max}^{2}\EE[g_{n}(Z) | \{Z_j\}_{j = 1}^{K}] }{K(\frac{\epsilon}{2}(\alpha + \beta) + \frac{\epsilon}{2}\EE[g_{n}(Z) | \{Z_j\}_{j = 1}^{K}] )^{2}}\\
    & \qquad \leq \frac{16H^{2}R_{\max}^{2}}{\epsilon^{2}(\alpha + \beta)K},
  \end{align*} where the last inequality comes from the fact that $\frac{s_0}{(a + s_0)^{2}} \leq \frac{1}{4a}$ for all $s_0 \geq 0$ and $a > 0$. Thus, for all $K \geq \frac{128H^{2}R_{\max}^{2}}{\epsilon^{2}(\alpha + \beta)}$,
  \[
    \Pr\biggl(\EE[g_{n}(Z) | \{Z_j\}_{j = 1}^{K}] - \frac{1}{K}\sum_{j = 1}^{K}g_{n}({Z'}_{j}) \geq \frac{\epsilon}{2}(\alpha + \beta) + \frac{\epsilon}{2}\EE[g_{n}(Z) | \{Z_j\}_{j = 1}^{K}] | \{Z_{j}\}_{j = 1}^{K}\biggr) \leq \frac{7}{8}.
  \] We then know that
  \begin{align*}
    &\Pr\biggl(\exists f, f'\in \cF, \pi \in \Pi:
      \frac{1}{K}\sum_{j = 1}^{K}g_{f_{h}, f'_{h + 1}}^{\pi}(Z_i')- \frac{1}{K}\sum_{j = 1}^{K}g_{f_{h}, f'_{h + 1}}^{\pi}(Z_j) \geq \frac{\epsilon}{2}(\alpha + \beta) + \frac{\epsilon}{2}\EE[g_{f_{h}, f'_{h + 1}}^{\pi}(Z)]
      \biggr)\\
    &\ \ \geq \Pr\biggl(\frac{1}{K}\sum_{j = 1}^{K}g_{n}(Z_i')- \frac{1}{K}\sum_{j = 1}^{K}g_{n}(Z_j) \geq \frac{\epsilon}{2}(\alpha + \beta) + \frac{\epsilon}{2}\EE[g_{n}(Z) | \{Z_j\}_{j = 1}^{K}]\biggr)\\
    &\ \  \geq \Pr\biggl(\EE[g_{n}(Z) | \{Z_j\}_{j = 1}^{K}] - \frac{1}{K}\sum_{j = 1}^{K}g_{n}(Z_j) \geq \epsilon(\alpha + \beta) + \epsilon\EE[g_{n}(Z) | \{Z_j\}_{j = 1}^{K}]  \\
    &\hspace{8em}\EE[g_{n}(Z) | \{Z_j\}_{j = 1}^{K}] - \frac{1}{K}\sum_{j = 1}^{K}g_{n}(Z_i') \geq \epsilon(\alpha + \beta) + \epsilon\EE[g_{n}(Z) | \{Z_j\}_{j = 1}^{K}]\biggr)\\
    &\ \  = \EE\Biggl(\ind\biggl\{\EE[g_{n}(Z) | \{Z_j\}_{j = 1}^{K}] - \frac{1}{K}\sum_{j = 1}^{K}g_{n}(Z_j) \geq \epsilon(\alpha + \beta) + \epsilon\EE[g_{n}(Z) | \{Z_j\}_{j = 1}^{K}]\biggr\}\\
    &\hspace{8em}\Pr\biggl(\EE[g_{n}(Z) | \{Z_j\}_{j = 1}^{K}] - \frac{1}{K}\sum_{j = 1}^{K}g_{n}(Z_i') \geq \epsilon(\alpha + \beta) + \epsilon\EE[g_{n}(Z) | \{Z_j\}_{j = 1}^{K}]\biggr)\Biggr)\\
    &\ \  \geq \frac{7}{8}\Pr\biggl(\EE[g_{n}(Z) | \{Z_j\}_{j = 1}^{K}] - \frac{1}{K}\sum_{j = 1}^{K}g_{n}(Z_j) \geq \epsilon(\alpha + \beta) + \epsilon\EE[g_{n}(Z) | \{Z_j\}_{j = 1}^{K}]\biggr)\\
    &\ \  = \frac{7}{8}\Pr\biggl(\exists f, f' \in \cF, \pi \in \Pi: \EE[g_{f_{h}, f'_{h + 1}}^{\pi}(Z)] - \frac{1}{K}\sum_{j = 1}^{K}g_{f_{h}, f'_{h + 1}}^{\pi}(Z_j) \geq \epsilon(\alpha + \beta) + \epsilon\EE[g_{f_{h}, f'_{h + 1}}^{\pi}(Z)]\biggr).\\
  \end{align*}
  In other words, for $K \geq \frac{128H^{2}R_{\max}^{2}}{\epsilon^{2}(\alpha + \beta)}$,
  \begin{multline}\label{eqn:symmetrization}
    \Pr\biggl(\exists f, f' \in \cF, \pi \in \Pi: \EE[g_{f_{h}, f'_{h + 1}}^{\pi}(Z)] - \frac{1}{K}\sum_{j = 1}^{K}g_{f_{h}, f'_{h + 1}}^{\pi}(Z_j) \geq \epsilon(\alpha + \beta) + \epsilon\EE[g_{f_{h}, f'_{h + 1}}^{\pi}(Z)]\biggr) \\
    \leq \frac{8}{7} \Pr\biggl(\exists f, f'\in \cF, \pi \in \Pi: \frac{1}{K}\sum_{j = 1}^{K}g_{f_{h}, f'_{h + 1}}^{\pi}(Z_{j}')\\
    - \frac{1}{K}\sum_{j = 1}^{K}g_{f_{h}, f'_{h + 1}}^{\pi}(Z_j) \geq  \frac{\epsilon}{2}(\alpha + \beta) + \frac{\epsilon}{2}\EE[g_{f_{h}, f'_{h + 1}}^{\pi}(Z)]
    \biggr).
  \end{multline}
  
  {\noindent \bf Replacement of Expectation by Empirical Mean of Ghost Sample}
  We begin by noticing
  \begin{equation}\label{eqn:ghost_sample}
    \begin{split}
      &\Pr\biggl(\exists f, f'\in \cF, \pi \in \Pi:
      \frac{1}{K}\sum_{j = 1}^{K}g_{f_{h}, f'_{h + 1}}^{\pi}(Z_i')
      -\frac{1}{K}\sum_{j = 1}^{K}g_{f_{h}, f'_{h + 1}}^{\pi}(Z_j) \geq  \frac{\epsilon}{2}(\alpha + \beta) + \frac{\epsilon}{2}\EE[g_{f_{h}, f'_{h + 1}}^{\pi}(Z)]
      \biggr)\\
      &\qquad \leq \Pr\biggl(\exists f, f'\in \cF, \pi \in \Pi:\\
      &\hspace{5em}
      \frac{1}{K}\sum_{j = 1}^{K}g_{f_{h}, f'_{h + 1}}^{\pi}(Z_i')
      -\frac{1}{K}\sum_{j = 1}^{K}g_{f_{h}, f'_{h + 1}}^{\pi}(Z_j) \geq  \frac{\epsilon}{2}(\alpha + \beta) + \frac{\epsilon}{2}\EE[g_{f_{h}, f'_{h + 1}}^{\pi}(Z)],\\
      &\hspace{5em}  \frac{1}{K}\sum_{j = 1}^{K}(g^{\pi}_{f_{h}, f'_{h + 1}})^{2}(Z_i') - \EE[(g^{\pi}_{f_{h}, f'_{h + 1}})^{2}(Z)]
      \leq\\
      &\hspace{10em}\epsilon\Bigl(\alpha + \beta + \frac{1}{K}\sum_{j = 1}^{K}(g^{\pi}_{f_{h}, f'_{h + 1}})^{2}(Z_j) + \EE[(g^{\pi}_{f_{h}, f'_{h + 1}})^{2}(Z)]
      \Bigr) ,\\
      &\hspace{5em}\frac{1}{K}\sum_{j = 1}^{K}(g^{\pi}_{f_{h}, f'_{h + 1}})^{2}(Z_i') - \EE[(g^{\pi}_{f_{h}, f'_{h + 1}})^{2}(Z)]
      \leq \\
      &\hspace{15em}\epsilon\Bigl(\alpha + \beta + \frac{1}{K}\sum_{j = 1}^{K}(g^{\pi}_{f_{h}, f'_{h + 1}})^{2}(Z_i') + \EE[(g^{\pi}_{f_{h}, f'_{h + 1}})^{2}(Z)] \Bigr)\biggr)\\
      &\hspace{3em} + 2\Pr\rbr{\exists f, f'\in \cF, \pi \in \Pi: \frac{\frac{1}{K}\sum_{j = 1}^{K}(g^{\pi}_{f_{h}, f'_{h + 1}})^{2}(Z_j) - \EE[(g^{\pi}_{f_{h}, f'_{h + 1}})^{2}(Z)]}{\rbr{\alpha + \beta + \frac{1}{K}\sum_{j = 1}^{K}(g^{\pi}_{f_{h}, f'_{h + 1}})^{2}(Z_j) + \EE[(g^{\pi}_{f_{h}, f'_{h + 1}})^{2}(Z)] }}}.
    \end{split}
  \end{equation}
  Citing Theorem~\ref{thm:11_6_gyorfi}, we may bound the second probability term on the right hand side as
  \begin{align*}
    &\Pr\rbr{\exists f, f'\in \cF, \pi \in \Pi: \frac{\frac{1}{K}\sum_{j = 1}^{K}(g^{\pi}_{f_{h}, f'_{h + 1}})^{2}(Z_{j}) - \EE[(g^{\pi}_{f_{h}, f'_{h + 1}})^{2}(Z)]}{\rbr{\alpha + \beta + \frac{1}{K}\sum_{j = 1}^{K}(g^{\pi}_{f_{h}, f'_{h + 1}})^{2}(Z_{j}) + \EE[(g^{\pi}_{f_{h}, f'_{h + 1}})^{2}(Z)] }}}\\
    &\qquad \leq 4\cN_{\infty}\rbr{\frac{(\alpha + \beta)\epsilon}{5}, \{g_{f_{h}, f'_{h + 1}}^{\pi}: f, f' \in \cF, \pi \in \Pi\}}\exp\rbr{-\frac{3\epsilon^{2}(\alpha + \beta)K}{40(16H^{2}R_{\max}^{2})}}.
  \end{align*} 
  For the first probability term, notice that the second event in the conjunction implies
  \[
    (1 + \epsilon)\EE[(g^{\pi}_{f_{h}, f'_{h + 1}})^{2}(Z)] \geq (1 - \epsilon)\frac{1}{K}\sum_{j = 1}^{K}(g^{\pi}_{f_{h}, f'_{h + 1}})^{2}(Z_j) - \epsilon(\alpha + \beta),
  \] 
  which is equivalent to
  \[
    \frac{1}{32H^{2}R_{\max}^{2}}\EE[(g^{\pi}_{f_{h}, f'_{h + 1}})^{2}(Z)] \geq \frac{1 - \epsilon}{32H^{2}R_{\max}^{2}(1 + \epsilon)}\frac{1}{K}\sum_{j = 1}^{K}(g^{\pi}_{f_{h}, f'_{h + 1}})^{2}(Z_j) - \epsilon \frac{(\alpha + \beta)}{32H^{2}R_{\max}^{2}(1 + \epsilon)}.
  \] 
  A similar bound may be obtained for the term involving $Z_i'$. Noticing that by equation~\eqref{eqn:g_var_bound}, we have $\EE[g^{\pi}_{f_{h}, f'_{h + 1}}(Z)] \geq \frac{1}{16H^{2}R_{\max}^{2}}\EE[(g^{\pi}_{f_{h}, f'_{h + 1}})^{2}(Z)]$, and we know the first probability term in~\eqref{eqn:ghost_sample} can be bounded by
  \begin{align*}\label{eqn:ghost_sample_first_term}
    &\Pr\biggl(\exists f, f'\in \cF, \pi \in \Pi:\frac{1}{K}\sum_{j = 1}^{K}g_{f_{h}, f'_{h + 1}}^{\pi}(Z_i')- \frac{1}{K}\sum_{j = 1}^{K}g_{f_{h}, f'_{h + 1}}^{\pi}(Z_j) \geq\frac{\epsilon}{2}(\alpha + \beta) +\\
    &\hspace{6em}\frac{\epsilon}{2}\Bigl(\frac{1 - \epsilon}{32H^{2}R_{\max}^{2}(1 + \epsilon)}\frac{1}{K}\sum_{j = 1}^{K}(g^{\pi}_{f_{h}, f'_{h + 1}})^{2}(Z_j) - \frac{\epsilon(\alpha + \beta)}{32H^{2}R_{\max}^{2}} +\\
    &\hspace{12em}\frac{1 - \epsilon}{32H^{2}R_{\max}^{2}(1 + \epsilon)}\frac{1}{K}\sum_{j = 1}^{K}(g^{\pi}_{f_{h}, f'_{h + 1}})^{2}(Z_j) - \frac{\epsilon(\alpha + \beta)}{32H^{2}R_{\max}^{2}}\Bigr)\biggr)\\
    & \qquad = \Pr\biggl(\exists f, f'\in \cF, \pi \in \Pi:\frac{1}{K}\sum_{j = 1}^{K}g_{f_{h}, f'_{h + 1}}^{\pi}(Z_i')- \frac{1}{K}\sum_{j = 1}^{K}g_{f_{h}, f'_{h + 1}}^{\pi}(Z_j) \geq\frac{\epsilon}{2}(\alpha + \beta) -\\
    &\hspace{6em}\frac{\epsilon^{2}(\alpha + \beta)}{32H^{2}R_{\max}^{2}(1 + \epsilon)} + \frac{\epsilon(1 - \epsilon)}{64H^{2}R_{\max}^{2}(1 + \epsilon)}\rbr{\frac{1}{K}\sum_{j = 1}^{K}((g^{\pi}_{f_{h}, f'_{h + 1}})^{2}(Z'_j) +(g^{\pi}_{f_{h}, f'_{h + 1}})^{2}(Z_j))}\biggr).
  \end{align*}
  
  {\noindent \bf Additional Randomization by Random Signs} 
  Let $\{U_j\}_{j = 1}^{K}$ be i.i.d.~Rademacher random variables drawn independently from $\{Z_j\}_{j = 1}^{K}$ and $\{Z'_j\}_{j = 1}^{K}$. Because $\{Z_j\}_{j = 1}^{K}$ and $\{Z'_j\}_{j=1}^{K}$ are i.i.d., we know that
  \begin{equation}\label{eqn:random_signs}
    \begin{split}
      &\Pr\biggl(\exists f, f'\in \cF, \pi \in \Pi:\frac{1}{K}\sum_{j = 1}^{K}g_{f_{h}, f'_{h + 1}}^{\pi}(Z'_{j})- \frac{1}{K}\sum_{j = 1}^{K}g_{f_{h}, f'_{h + 1}}^{\pi}(Z_{j}) \geq\frac{\epsilon}{2}(\alpha + \beta) -\\
      &\hspace{6em}\frac{\epsilon^{2}(\alpha + \beta)}{32H^{2}R_{\max}^{2}(1 + \epsilon)} + \frac{\epsilon(1 - \epsilon)}{64H^{2}R_{\max}^{2}(1 + \epsilon)}\rbr{\frac{1}{K}\sum_{j = 1}^{K}((g^{\pi}_{f_{h}, f'_{h + 1}})^{2}(Z_i') +(g^{\pi}_{f_{h}, f'_{h + 1}})^{2}(Z_j))}\biggr)\\
      &\qquad = \Pr\biggl(\exists f, f'\in \cF, \pi \in \Pi:\frac{1}{K}\sum_{j = 1}^{K}U_{j}\bigl(g_{f_{h}, f'_{h + 1}}^{\pi}(Z'_{j})-g_{f_{h}, f'_{h + 1}}^{\pi}(Z_j) \bigr) \geq\frac{\epsilon}{2}(\alpha + \beta) -\\
      &\hspace{6em}\frac{\epsilon^{2}(\alpha + \beta)}{32H^{2}R_{\max}^{2}(1 + \epsilon)} + \frac{\epsilon(1 - \epsilon)}{64H^{2}R_{\max}^{2}(1 + \epsilon)}\rbr{\frac{1}{K}\sum_{j = 1}^{K}((g^{\pi}_{f_{h}, f'_{h + 1}})^{2}(Z_i') +(g^{\pi}_{f_{h}, f'_{h + 1}})^{2}(Z_j))}\biggr)\\
      &\qquad \leq 2  \Pr\biggl(\exists f, f'\in \cF, \pi \in \Pi:\frac{1}{K}\sum_{j = 1}^{K}\abr{U_{j}g_{f_{h}, f'_{h + 1}}^{\pi}(Z_j)}\geq\frac{\epsilon}{4}(\alpha + \beta) -\\
      &\hspace{6em}\frac{\epsilon^{2}(\alpha + \beta)}{64H^{2}R_{\max}^{2}(1 + \epsilon)} + \frac{\epsilon(1 - \epsilon)}{64H^{2}R_{\max}^{2}(1 + \epsilon)}\frac{1}{K}\sum_{j = 1}^{K}((g^{\pi}_{f_{h}, f'_{h + 1}})^{2}(Z_j))\biggr).
    \end{split}
  \end{equation}
  
  {\noindent \bf Conditioning and Covering} We then condition the probability on $\{Z_{j}\}_{j = 1}^{K}$. Fix some $z_{1}, \ldots, z_{K}$ and we consider instead
  \begin{align*}
    &\Pr\Biggl\{\exists f, f'\in\cF, \pi\in\Pi: \abr{\frac{1}{K}\sum_{j = 1}^{K}U_{j}g_{f_{h}, f'_{h + 1}}^{\pi}(z_{j})}\geq\\
    &\hspace{8em}\frac{\epsilon(\alpha + \beta)}{4} - \frac{\epsilon^{2}(\alpha + \beta)}{64H^{2}R_{\max}^{2}(1 + \epsilon)} + \frac{\epsilon(1 -\epsilon)}{64H^{2}R_{\max}^{2}(1 + \epsilon)}\frac{1}{K}\sum_{j = 1}^{K}(g_{f_{h}, f'_{h + 1}}^{\pi})^{2}(z_{j}) \Biggr\}.
  \end{align*} Let $\delta > 0$ and let $\cG_{\delta}$ be an $\ell_{\infty}$ $\delta$-cover of $\cG_{\cF, \Pi} = \{g_{f_{h}, f'_{h + 1}}^{\pi} : f, f'\in F, \pi\in\Pi\}$. Fix some $(f, f', \pi)\in \cF \times \cF \times \Pi$ and there exists some $g \in \cG_{\delta}$ such that $\sup_{z}|g(z) - g_{f_{h}, f'_{h + 1}}^{\pi}(z)| < \delta$. We then know that
  \begin{align*}
    \abr{\frac{1}{K}\sum_{j = 1}^{K}U_{j}g_{f_{h}, f'_{h + 1}}^{\pi}(z_{j})} \leq \abr{\frac{1}{K}\sum_{j = 1}^{K}U_{j}g(z_{j})} + \frac{1}{K}\sum_{j = 1}^{K}\abr{g_{f_{h}, f'_{h + 1}}^{\pi}(z_{j}) - g(z_{j})} \leq \abr{\frac{1}{K}\sum_{j = 1}^{K}U_{j}g(z_{j})} + \delta
  \end{align*} and
  \begin{align*}
    \frac{1}{K}\sum_{j = 1}^{K}(g_{f_{h}, f'_{h + 1}}^{\pi})^{2}(z_{j}) &= \frac{1}{K}\sum_{j = 1}^{K}g^{2}(z_{j}) + \frac{1}{K}\sum_{j = 1}^{K}((g_{f_{h}, f'_{h + 1}}^{\pi})^{2}(z_{j}) - g^{2}(z_{j}))\\
    &=  \frac{1}{K}\sum_{j = 1}^{K}g^{2}(z_{j}) + \frac{1}{K}\sum_{j = 1}^{K}(g_{f_{h}, f'_{h + 1}}^{\pi}(z_{j}) - g(z_{j}))(g_{f_{h}, f'_{h + 1}}^{\pi}(z_{j}) + g(z_{j}))\\
    &\geq \frac{1}{K}\sum_{j = 1}^{K}g^{2}(z_{j}) - 8H^{2}R_{\max}^{2}\frac{1}{K}\sum_{j = 1}^{K}|g_{f_{h}, f'_{h + 1}}^{\pi}(z_{j}) - g(z_{j})|\\
    &\geq \frac{1}{K}\sum_{j = 1}^{K}g^{2}(z_{j}) - 8H^{2}R_{\max}^{2}\delta.\\
  \end{align*} Set $\delta = \frac{\beta\epsilon}{enumerate5}$. Notice that as $HR_{\max} \geq 1$, $0 < \epsilon \leq \frac{1}{2}$, we have
  \begin{align*}
    \frac{\epsilon\beta}{4} - \frac{\epsilon^{2}\beta}{64H^{2}R_{\max}^{2}(1 + \epsilon)} - \delta - \delta\frac{\epsilon(1 - \epsilon)}{8(1 + \epsilon)} = \frac{\epsilon\beta}{2} - \frac{\epsilon^{2}\beta}{64H^{2}R_{\max}^{2}(1 + \epsilon)}  - \frac{\epsilon^{2}(1 - \epsilon)\beta}{40(1 + \epsilon)} \geq 0.
  \end{align*}
  Therefore we have
  \begin{align}
    &\Pr\Biggl\{\exists f, f'\in\cF, \pi\in\Pi: \abr{\frac{1}{K}\sum_{j = 1}^{K}U_{j}g_{f_{h}, f'_{h + 1}}^{\pi}(z_{j})}\geq \nonumber\\
    &\hspace{8em}\frac{\epsilon(\alpha + \beta)}{4} - \frac{\epsilon^{2}(\alpha + \beta)}{64H^{2}R_{\max}^{2}(1 + \epsilon)} + \frac{\epsilon(1 -\epsilon)}{64H^{2}R_{\max}^{2}(1 + \epsilon)}\frac{1}{K}\sum_{j = 1}^{K}(g_{f_{h}, f'_{h + 1}}^{\pi})^{2}(z_{j}) \Biggr\}\nonumber\\
    & \quad \leq |\cG_{{\epsilon\beta}/{5}}|\max_{g \in \cG_{{\epsilon\beta}/{5}}}\Pr\Biggl\{\abr{\frac{1}{K}\sum_{j = 1}^{K}U_{j}g(z_{j})}\geq \frac{\epsilon\alpha}{4} -\frac{\epsilon^{2}\alpha}{64H^{2}R_{\max}^{2}(1 + \epsilon)} + \nonumber\\
    &\hspace{18em}\frac{\epsilon(1 -\epsilon)}{64H^{2}R_{\max}^{2}(1 + \epsilon)}\frac{1}{K}\sum_{j = 1}^{K}g^{2}(z_{j}) \Biggr\}\label{eqn:covering_bound_intermediate}.
  \end{align}
  We then apply Bernstein's inequality to bound
  \[
    \Pr\Biggl\{\abr{\frac{1}{K}\sum_{j = 1}^{K}U_{j}g(z_{j})}\geq \frac{\epsilon\alpha}{4} -\frac{\epsilon^{2}\alpha}{64H^{2}R_{\max}^{2}(1 + \epsilon)} + \frac{\epsilon(1 -\epsilon)}{64H^{2}R_{\max}^{2}(1 + \epsilon)}\frac{1}{K}\sum_{j = 1}^{K}g^{2}(z_{j}) \Biggr\}
  \] for any $g \in \cG_{\epsilon\beta/5}$. We begin by relating the variance of $U_{j}g(z_{j})$ with $\frac{1}{K}\sum_{j = 1}^{k}g^{2}(z_{j})$. Notice that as $U_{j}$ is i.i.d. Rademacher,
  \[
    \frac{1}{K}\sum_{j = 1}^{K}\Var(U_{j}g(z_{j})) = \frac{1}{K}\sum_{j = 1}^{k}g^{2}(z_{j}) \Var(U_{i}) =\frac{1}{K}\sum_{j = 1}^{k}g^{2}(z_{j}).
  \] Perform a simple change of variable and let $V_{j} = g(z_{j})U_{j}$. As $g(z_{j}) \in [-4H^{2}R_{\max}^{2},4H^{2}R_{\max}^{2}]$ for all $z_{j}$, we know $|V_{j}| \leq 4H^{2}R_{\max}^{2}$. For convenience, further let $A_{1} = \frac{\epsilon\alpha}{4} - \frac{\epsilon^{2}\alpha}{64H^{2}R_{\max}^{2}(1 + \epsilon)}, A_{2} = \frac{\epsilon(1 - \epsilon)}{64H^{2}R_{\max}^{2}(1 + \epsilon)},$ and $\sigma^{2} = \frac{1}{K}\sum_{j = 1}^{K}\Var(U_{j}g(z_{j})) = \frac{1}{K}\sum_{j = 1}^{k}g^{2}(z_{j})$. We then have for any $g \in \cG_{\epsilon\beta/5}$
  \begin{align*}
    &\Pr\Biggl\{\abr{\frac{1}{K}\sum_{j = 1}^{K}U_{j}g(z_{j})}\geq \frac{\epsilon\alpha}{4} -\frac{\epsilon^{2}\alpha}{64H^{2}R_{\max}^{2}(1 + \epsilon)} + \frac{\epsilon(1 -\epsilon)}{64H^{2}R_{\max}^{2}(1 + \epsilon)}\frac{1}{K}\sum_{j = 1}^{K}g^{2}(z_{j}) \Biggr\}\\
    &\qquad = \Pr\rbr{\abr{\frac{1}{K}\sum_{j = 1}^{k}V_{j}} \ge A_{1} + A_{2}\sigma^{2}}\\
    &\qquad \leq 2\exp\rbr{-\frac{K(A_{1} + A_{2}\sigma^{2})^{2}}{2\sigma^{2} + 2(A_{1} + A_{2}\sigma^{2})\frac{8H^{2}R^{2}}{3}}}\\
    &\qquad = 2\exp\rbr{-\frac{3KA_{2}}{16H^{2}R_{\max}^{2}}\frac{\rbr{\frac{A_{1}}{A_{2}} + \sigma^{2}}^{2}}{\frac{A_{1}}{A_{2}} + \rbr{1 + \frac{3}{8H^{2}R_{\max}^{2}A_{2}}}\sigma^{2}}}\\
    &\qquad \leq 2\exp\rbr{-\frac{\epsilon^{2}(1 - \epsilon)\alpha K}{140H^{2}R_{\max}^{2}(1 + \epsilon)}},
  \end{align*} where the last inequality follows a series of manipulations discussed in greater detail in page 218 of~\citet{gyorfi2002distribution} that we omit here for brevity. Plugging the result back into equations~\eqref{eqn:random_signs} and~\eqref{eqn:covering_bound_intermediate} gives us
  \begin{align*}
    &\Pr\biggl(\exists f, f'\in \cF, \pi \in \Pi:\frac{1}{K}\sum_{j = 1}^{K}g_{f_{h}, f'_{h + 1}}^{\pi}(Z'_{j})- \frac{1}{K}\sum_{j = 1}^{K}g_{f_{h}, f'_{h + 1}}^{\pi}(Z_{j}) \geq\frac{\epsilon}{2}(\alpha + \beta) -\\
    &\hspace{6em}\frac{\epsilon^{2}(\alpha + \beta)}{32H^{2}R_{\max}^{2}(1 + \epsilon)} + \frac{\epsilon(1 - \epsilon)}{64H^{2}R_{\max}^{2}(1 + \epsilon)}\rbr{\frac{1}{K}\sum_{j = 1}^{K}((g^{\pi}_{f_{h}, f'_{h + 1}})^{2}(Z_i') +(g^{\pi}_{f_{h}, f'_{h + 1}})^{2}(Z_j))}\biggr)\\
    &\qquad \leq 2\cN_{\infty}\rbr{\frac{\epsilon\beta}{5}, \{g_{f_{h}, f'_{h + 1}}^{\pi} : f, f'\in F, \pi\in\Pi\}}\exp\rbr{-\frac{\epsilon^{2}(1 - \epsilon)\alpha K}{140H^{2}R_{\max}^{2}(1 + \epsilon)}} .
  \end{align*} Recalling equations~\eqref{eqn:ghost_sample} and~\eqref{eqn:random_signs}, we have
  \begin{align*}
    &\Pr\biggl(\exists f, f'\in \cF, \pi \in \Pi:
      \frac{1}{K}\sum_{j = 1}^{K}g_{f_{h}, f'_{h + 1}}^{\pi}(Z_i')
      -\frac{1}{K}\sum_{j = 1}^{K}g_{f_{h}, f'_{h + 1}}^{\pi}(Z_j) \geq  \frac{\epsilon}{2}(\alpha + \beta) + \frac{\epsilon}{2}\EE[g_{f_{h}, f'_{h + 1}}^{\pi}(Z)]
      \biggr)\\
    &\qquad \leq 4\cN_{\infty}\rbr{\frac{\epsilon\beta}{5}, \{g_{f_{h}, f'_{h + 1}}^{\pi} : f, f'\in F, \pi\in\Pi\}}\exp\rbr{-\frac{\epsilon^{2}(1 - \epsilon)\alpha K}{140H^{2}R_{\max}^{2}(1 + \epsilon)}}\\
    &\hspace{3em}+ 8\cN_{\infty}\rbr{\frac{(\alpha + \beta)\epsilon}{5}, \{g_{f_{h}, f'_{h + 1}}^{\pi}: f, f' \in \cF, \pi \in \Pi\}}\exp\rbr{-\frac{3\epsilon^{2}(\alpha + \beta)K}{640H^{2}R_{\max}^{2}}}.
  \end{align*} Plugging the result back into equation~\eqref{eqn:symmetrization} and we finally know for $K \geq \frac{128H^{2}R_{\max}^{2}}{\epsilon^2(\alpha + \beta)}$,
  \begin{align*}
    &\Pr\biggl(\exists f, f' \in \cF, \pi \in \Pi: \EE[g_{f_{h}, f'_{h + 1}}^{\pi}(Z)] - \frac{1}{K}\sum_{j = 1}^{K}g_{f_{h}, f'_{h + 1}}^{\pi}(Z_j) \geq \epsilon(\alpha + \beta) + \epsilon\EE[g_{f_{h}, f'_{h + 1}}^{\pi}(Z)]\biggr) \\
    &\qquad \leq \frac{32}{7}\cN_{\infty}\rbr{\frac{\epsilon\beta}{5}, \{g_{f_{h}, f'_{h + 1}}^{\pi} : f, f'\in F, \pi\in\Pi\}}\exp\rbr{-\frac{\epsilon^{2}(1 - \epsilon)\alpha K}{140H^{2}R_{\max}^{2}(1 + \epsilon)}}\\
    &\hspace{3em} +\frac{64}{7}\cN_{\infty}\rbr{\frac{(\alpha + \beta)\epsilon}{5}, \{g_{f_{h}, f'_{h + 1}}^{\pi}: f, f' \in \cF, \pi \in \Pi\}}\exp\rbr{-\frac{3\epsilon^{2}(\alpha + \beta)K}{640H^{2}R_{\max}^{2}}}\\
    &\qquad \leq 14\cN_{\infty}\rbr{\frac{\epsilon\beta}{5}, \{g_{f_{h}, f'_{h + 1}}^{\pi} : f, f'\in F, \pi\in\Pi\}}\exp\rbr{-\frac{\epsilon^{2}(1 - \epsilon)\alpha K}{214(1 + \epsilon)H^{4}R_{\max}^{4}}}.
  \end{align*} 
  When $K < \frac{128H^{2}R_{\max}^{2}}{\epsilon^2(\alpha + \beta)}$, $\exp\rbr{-\frac{\epsilon^{2}(1 - \epsilon)\alpha K}{214(1 + \epsilon)H^{4}R_{\max}^{4}}} \geq \exp\rbr{-\frac{128}{214}} \geq \frac{1}{14}$ and the claim trivially holds.
  {\noindent \bf Bounding the Covering Number.} Our final task is bounding $\cN_{\infty}\rbr{\frac{\epsilon\beta}{5}, \{g_{f_{h}, f'_{h + 1}}^{\pi} : f, f'\in F, \pi\in\Pi\}}$ using the covering numbers of $\Pi$ and $\cF$. Let $\cF_{0}$ be a $\frac{\epsilon\beta}{140HR_{\max}}$-covering of $\mathcal{F}$ with respect to $\ell_{\infty}$ and $\Pi_{0}$ a $\frac{\epsilon\beta}{140H^{2}R^{2}_{\mathrm{max}}}$-covering of $\Pi$ with respect to $\|\cdot\|_{\infty, 1}$. We then know that for any $f, f' \in \mathcal{F}, \pi \in \Pi$, there exits some $f^{\dagger}, f^{\ddagger} \in \cF_{0}, \pi^{\dagger} \in \Pi_{0}$ such that
  \begin{align*}
    &\sup_{(s, a) \in \cS \times \cA} |f_{h}(s, a) - f_{h}^{\dagger}(s, a)| \leq \frac{\epsilon\beta}{140HR_{\max}},\\
    &\sup_{(s, a) \in \cS \times \cA} |f'_{h + 1}(s, a) - f_{h + 1}^{\ddagger}(s, a)| \leq \frac{\epsilon\beta}{140HR_{\max}},\\
    &\sup_{s \in\cS}\int_{a \in \cA} |\pi_{h + 1}(a | s) - \pi_{h + 1}^{\dagger}(a | s)| \leq \frac{\epsilon\beta}{140H^{2}R^{2}_{\mathrm{max}}}.
  \end{align*}
  Consider any arbitrary $z = (s, a, r, s') \sim \mu_{h}$. We know that
  \begin{align}\label{eqn:intermediate_bound_g}
    &\abr{g_{f_{h}, f'_{h + 1}}^{\pi_{h + 1}}(z)- g_{f^{\dagger}_{h}, f^{\ddagger}_{h + 1}}^{\pi^{\dagger}_{h + 1}}(z)}\nonumber\\
    &\qquad = \biggl|(f_{h}(s, a) - r - f'_{h + 1}(s', \pi_{h + 1}))^{2} - (\cT_{h, r}^{\pi_{h + 1}}f'_{h + 1}(s, a) - r - f'_{h + 1}(s', \pi_{h + 1}))^{2} - \nonumber\\
    &\hspace{7em}(f_{h}^{\dagger}(s, a) - r - f^{\ddagger}_{h + 1}(s', \pi^{\dagger}_{h + 1}))^{2} + (\cT_{h, r}^{\pi^{\dagger}_{h + 1}}f^{\ddagger}_{h + 1}(s, a) - r - f^{\ddagger}_{h + 1}(s', \pi^{\dagger}_{h + 1}))^{2}\biggr|\nonumber\\
    &\qquad \leq \biggl|(f_{h}(s, a) - r - f'_{h + 1}(s', \pi_{h + 1}))^{2} -(f_{h}^{\dagger}(s, a) - r - f^{\ddagger}_{h + 1}(s', \pi^{\dagger}_{h + 1}))^{2}\biggr| \nonumber\\
    &\hspace{3em} + \biggl|(\cT_{h, r}^{\pi_{h + 1}}f'_{h + 1}(s, a) - r - f'_{h + 1}(s', \pi_{h + 1}))^{2} - (\cT_{h, r}^{\pi^{\dagger}_{h + 1}}f^{\ddagger}_{h + 1}(s, a) - r - f^{\ddagger}_{h + 1}(s', \pi^{\dagger}_{h + 1}))^{2}\biggr|\nonumber\\
    &\qquad\leq\biggl|f_{h}(s, a) + f_{h}^{\dagger}(s, a) - 2r - f'_{h + 1}(s', \pi_{h + 1}) - f^{\ddagger}_{h + 1}(s', \pi^{\dagger}_{h + 1})\biggr| \nonumber\\
    &\hspace{3em}\times\biggl|f_{h}(s, a) - f_{h}^{\dagger}(s, a)  + f'_{h + 1}(s', \pi_{h + 1}) - f^{\ddagger}_{h + 1}(s', \pi^{\dagger}_{h + 1})\biggr| \nonumber\\
    &\hspace{3em}+\biggl|\cT_{h, r}^{\pi_{h + 1}}f'_{h + 1}(s, a) + \cT_{h, r}^{\pi^{\dagger}_{h + 1}}f^{\ddagger}_{h + 1}(s, a) - 2r - f'_{h + 1}(s', \pi_{h + 1})-f^{\ddagger}_{h + 1}(s', \pi^{\dagger}_{h + 1})\biggr| \nonumber\\
    &\hspace{3em}\times\biggl|\cT_{h, r}^{\pi_{h + 1}}f'_{h + 1}(s, a) -\cT_{h, r}^{\pi^{\dagger}_{h + 1}}f^{\ddagger}_{h + 1}(s, a) + f'_{h + 1}(s', \pi_{h + 1}) - f^{\ddagger}_{h + 1}(s', \pi^{\dagger}_{h + 1})\biggr|\nonumber\\
    &\qquad \leq 4HR_{\max}\biggl|f_{h}(s, a) - f_{h}^{\dagger}(s, a)  + f'_{h + 1}(s', \pi_{h + 1}) - f^{\ddagger}_{h + 1}(s', \pi^{\dagger}_{h + 1})\biggr|\nonumber\\
    &\hspace{3em}+4HR_{\max}\biggl|\cT_{h, r}^{\pi_{h + 1}}f'_{h + 1}(s, a) -\cT_{h, r}^{\pi^{\dagger}_{h + 1}}f^{\ddagger}_{h + 1}(s, a) + f'_{h + 1}(s', \pi_{h + 1}) - f^{\ddagger}_{h + 1}(s', \pi^{\dagger}_{h + 1})\biggr|,
  \end{align} where for the last inequality we used the boundedness of functions in $\cF_{h}$ and $\cF_{h + 1}$. We then  notice that
  \begin{align*}
    &\biggl|f_{h}(s, a) - f_{h}^{\dagger}(s, a)  + f'_{h + 1}(s', \pi_{h + 1}) - f^{\ddagger}_{h + 1}(s', \pi^{\dagger}_{h + 1})\biggr|\\
    &\qquad \leq |f_{h}(s, a) - f_{h}^{\dagger}(s, a)| + |f'_{h + 1}(s', \pi_{h + 1}) - f^{\ddagger}_{h + 1}(s', \pi^{\dagger}_{h + 1})|\\
    &\qquad \leq \frac{\epsilon\beta}{140HR_{\max}} + |f'_{h + 1}(s', \pi_{h + 1}) - f'_{h + 1}(s', \pi^{\dagger}_{h + 1})| + |f'_{h + 1}(s', \pi^{\dagger}_{h + 1}) - f^{\ddagger}_{h + 1}(s', \pi^{\dagger}_{h + 1})|\\
    &\qquad \leq \frac{\epsilon\beta}{140HR_{\max}} + \|\pi_{h + 1} - \pi^{\dagger}_{h + 1}\|_{1}\|f'_{h + 1}\|_{\infty} + |f'_{h + 1}(s', \pi^{\dagger}_{h + 1}) - f^{\ddagger}_{h + 1}(s', \pi^{\dagger}_{h + 1})|\\
    &\qquad \leq \frac{\epsilon\beta}{140HR_{\max}} + \frac{\epsilon\beta}{140H^{2}R^{2}_{\mathrm{max}}}HR_{\max} + |f'_{h + 1}(s', \pi^{\dagger}_{h + 1}) - f^{\ddagger}_{h + 1}(s', \pi^{\dagger}_{h + 1})|\\
    &\qquad \leq \frac{\epsilon\beta}{140HR_{\max}} + \frac{\epsilon\beta}{140HR_{\max}} + \EE_{a' \sim \pi^{\dagger}_{h + 1}(\cdot | s')}[|f'_{h + 1}(s', a') - f^{\ddagger}_{h + 1}(s', a')|]\\
    &\qquad \leq \frac{3\epsilon\beta}{140HR_{\max}},
  \end{align*}
  where the third inequality uses Holder's inequality, the fourth definition of $\Pi_{0}$ and boundedness of $\cF_{h}$, the fifth Jensen's inequality, and the last inequality the definition of $\cF_{0}$. Additionally we have
  \begin{align*}
    &|\cT_{h, r}^{\pi_{h + 1}}f'_{h + 1}(s, a) -\cT_{h, r}^{\pi^{\dagger}_{h + 1}}f^{\ddagger}_{h + 1}(s, a) + f'_{h + 1}(s', \pi_{h + 1}) - f^{\ddagger}_{h + 1}(s', \pi^{\dagger}_{h + 1})|\\
    &\qquad\leq |\cT_{h, r}^{\pi_{h + 1}}f'_{h + 1}(s, a) -\cT_{h, r}^{\pi^{\dagger}_{h + 1}}f_{h + 1}^{\ddagger}(s, a)| + |f'_{h + 1}(s', \pi_{h + 1}) - f^{\ddagger}_{h + 1}(s', \pi^{\dagger}_{h + 1})|\\
    &\qquad\leq |\cT_{h, r}^{\pi_{h + 1}}f'_{h + 1}(s, a) -\cT_{h, r}^{\pi^{\dagger}_{h + 1}}f_{h + 1}^{\ddagger}(s, a)| + \frac{2\epsilon\beta}{140HR_{\max}}\\
    &\qquad\leq \EE_{s'' \sim \cP_{h}(\cdot | s, a)}|f'_{h + 1}(s', \pi_{h + 1}) - f^{\ddagger}_{h + 1}(s', \pi^{\dagger}_{h + 1})| + \frac{2\epsilon\beta}{140HR_{\max}}\\
    &\qquad\leq \frac{4\epsilon\beta}{140HR_{\max}},
  \end{align*} 
  where the second inequality uses the same reasoning as above to bound $|f'_{h + 1}(s', \pi_{h + 1}) - f^{\ddagger}_{h + 1}(s', \pi^{\dagger}_{h + 1})|$, the third Jensen's inequality, and the last inequality reuses the bound for $|f'_{h + 1}(s', \pi_{h + 1}) - f^{\ddagger}_{h + 1}(s', \pi^{\dagger}_{h + 1})|$ over arbitrary $s'$. Plugging these back into equation \eqref{eqn:intermediate_bound_g} shows
  \[
  \abr{g_{f_{h}, f'_{h + 1}}^{\pi_{h + 1}}(z)- g_{f^{\dagger}_{h}, f^{\ddagger}_{h + 1}}^{\pi^{\dagger}_{h + 1}}(z)} \le \frac{7\epsilon\beta}{140HR_{\max}} \times 4HR_{\max} = \frac{\epsilon\beta}{5}.
  \]
  Thus
  \begin{equation*}
    \cN_{\infty}\rbr{\frac{\epsilon\beta}{5}, \{g_{f_{h}, f'_{h + 1}}^{\pi} : f, f'\in F, \pi\in\Pi\}} \leq \rbr{\cN_{\infty}\rbr{\frac{\epsilon\beta}{140HR_{\max}}, \cF}}^{2}\cN_{\infty, 1}\rbr{\frac{\epsilon\beta}{140H^{2}R^{2}_{\mathrm{max}}}, \Pi},
  \end{equation*} showing one side of the inequality holds.

  To show the other side holds, simply replace $g_{f, f'}^{\pi}(Z)$ defined in equation~\ref{eqn:defn_g} with its negative and repeat the analysis above. We then complete the proof by taking a union bound over both halves.

\subsection{Proofs of ``Good Event"}\label{subsec:proofs_of_good_event}
With the help of the previous theorem, we are able to show that $\cG(\Pi_{\rm SPI})$ occurs with high probability.

\begin{proof}[Proof of Lemma~\ref{corollary:good_event}]
  Taking a union bound over all $h \in [H]$ and reported reward $r \in \tilde{\cR}$ recalling that $|\tilde{\cR}| \leq n + 1 \leq 2n$, by Lemma~\ref{lemma:gyorfi_variant_bellman_error_concentration}, we have
  \begin{align*}
    &\Pr\Bigl(\exists h \in [H], r \in \tilde{\cR}, f, f' \in \cF, \pi \in \Pi:\\
    &\hspace{6em}\abr{ \EE_{\mu_h}\sbr{\|f_{h} - \cT_{h, r}^{\pi}f'_{h + 1}\|^2} - \cL_{h, r}(f_{h}, f'_{h + 1}, \pi; \cD) + \cL_{h, r}(\cT_{h, r}^{\pi}f'_{h + 1}, f'_{h + 1}, \pi; \cD)}\\
    &\hspace{24em} \geq \epsilon\rbr{\alpha + \beta + \EE_{\mu_h}\sbr{\|f_{h} - \cT_{h, r}^{\pi}f'_{h + 1}\|^2}}\Bigr)\\
    &\qquad \leq 56nH\rbr{\cN_{\infty}\rbr{\frac{\epsilon\beta}{140HR_{\max}}, \cF}}^{2}\cN_{\infty, 1}\rbr{\frac{\epsilon\beta}{140H^{2}R^{2}_{\mathrm{max}}}, \Pi}\exp\rbr{-\frac{\epsilon^{2}(1 - \epsilon)\alpha K}{214(1 + \epsilon)H^{4}R_{\max}^{4}}}.
  \end{align*}
  Letting $\alpha = \beta$ and $\epsilon = \frac{1}{2}$, setting the right hand side to $\delta$, and solving for \(\alpha\) gives us
  \[
    \alpha \leq \frac{1}{K}\max\cbr{5136H^{4}R^{4}_{\max}, 5136H^{4}R^{4}_{\max}\log \frac{56nH\cN_{\infty}\rbr{\frac{HR_{\max}}{K}, \cF}\cN_{\infty, 1}\rbr{\frac{1}{K}, \Pi}}{\delta}}.
  \] As $\log 56 \geq 1$, $n, H \geq 1$, and $0 < 1 < \delta$, the second term always dominates the first and we can simplify the inequality as
  \[
    \alpha\leq \frac{5136H^{4}R^{4}_{\max}}{K}\log \frac{56nH\cN_{\infty}\rbr{\frac{19H^{3}R^{3}_{\max}}{K}, \cF}\cN_{\infty, 1}\rbr{\frac{19H^{4}R^{4}_{\max}}{K}, \Pi}}{\delta},
  \] completing the proof.
\end{proof}

\begin{proof}[Proof of Corollary~\ref{lemma:q_star_small_cE}]
  For convenience, let $\hat{g}^{\pi}_{h, r} = \argmin_{g \in \cF_{h}}\cL_{h, r}(g, f_{h + 1, r}^{\pi, *}, \pi; \cD)$. We then know that
  \begin{align*}
    \cE_{h, r}(f_{h, r}^{\pi, *}, \pi; \cD) &= \cL_{h, r}(f_{h, r}^{\pi, *}, f_{h + 1, r}^{\pi, *}, \pi; \cD) - \cL_{h, r}(\hat{g}^{\pi}_{h, r}, f_{h + 1, r}^{\pi, *}, \pi; \cD)\\
    &= \cL_{h, r}(f_{h, r}^{\pi, *}, f_{h + 1, r}^{\pi, *}, \pi; \cD)  -\cL_{h, r}(\cT_{h, r}^{\pi, *}f_{h + 1, r}^{\pi, *}, f_{h + 1, r}^{\pi, *}, \pi; \cD) \\
    &\quad -\rbr{\cL_{h, r}(\hat{g}^{\pi}_{h, r}, f_{h + 1, r}^{\pi, *}, \pi; \cD) -\cL_{h, r}(\cT_{h, r}^{\pi, *}f_{h + 1, r}^{\pi, *}, f_{h + 1, r}^{\pi, *}, \pi; \cD)}.
  \end{align*}
  By Lemma~\ref{corollary:good_event}, conditionally on the event $\cG(\Pi)$ we have the following simultaneously:
  \begin{align*}
    \cL_{h, r}(f_{h, r}^{\pi, *}, f_{h + 1, r}^{\pi, *}, \pi; \cD)  -\cL_{h, r}(\cT_{h, r}^{\pi, *}f_{h + 1, r}^{\pi, *}, f_{h + 1, r}^{\pi, *}, \pi; \cD) &\leq \epsilon_{\rm S} + \frac{3}{2} \EE_{\mu_h}\sbr{\| f_{h, r}^{\pi, *} - \cT_{h, r}^{\pi, *}f_{h + 1, r}^{\pi, *}\|^{2}},\\
    -\cL_{h, r}(\hat{g}^{\pi}_{h, r}, f_{h + 1, r}^{\pi, *}, \pi; \cD) +\cL_{h, r}(\cT_{h, r}^{\pi, *}f_{h + 1, r}^{\pi, *}, f_{h + 1, r}^{\pi, *}, \pi; \cD) &\leq \epsilon_{\rm S},
  \end{align*} where the second inequality uses the fact that $\|\cdot\|^{2}$ is non-negative. Finally, noticing that
  \begin{align*}
    \EE_{\mu_h}\sbr{\| f_{h, r}^{\pi, *} - \cT_{h, r}^{\pi, *}f_{h + 1, r}^{\pi, *}\|^{2}} &\leq 2\EE_{\mu_h}\sbr{\| f_{h, r}^{\pi, *} - Q_{h}^{\pi}(\cdot, \cdot; r)\|^{2}} + 2\EE_{\mu_h}\sbr{\| \cT_{h, r}^{\pi, *}f_{h + 1, r}^{\pi, *} - \cT_{h, r}^{\pi, *}Q_{h}^{\pi}(\cdot, \cdot; r)\|^{2}}\\
    & \leq 2\epsilon_{\cF} +2\EE_{\mu_{h + 1}'}\sbr{\| f_{h + 1, r}^{\pi, *} - Q_{h + 1}^{\pi}(\cdot, \cdot; r)\|^{2}}\\
    & \leq 4\epsilon_{\cF},
  \end{align*} where $\mu_{h + 1}'$ shares the marginal distribution over $\cS$ with $\mu_{h + 1}$ but the conditional distribution over $\cA$ given $s \in \cS$ is given by $\pi_{h + 1}(\cdot | s)$. The final inequality comes from the fact that $\mu_{h + 1}'$ is an admissible distribution under Assumption~\ref{assumption:realizability}.
\end{proof}

\begin{proof}[Proof of Corollary~\ref{lemma:small_cE_small_bellman_error}]
  Let $\hat{g}^{\pi}_{h, r} = \argmin_{g \in \cF_{h}}\EE_{\mu_h}[\|g - \cT_{h, r}^{\pi}f_{h + 1, r}^{\pi}\|^{2}]$.  Recalling the definition of $\cE_{h, r}$, we have
  \begin{align*}
    \cE_{h, r}(f_{h, r}^{\pi}, \pi; \cD) &= \cL_{h, r}(f_{h, r}^{\pi}, f_{h + 1, r}^{\pi}, \pi; \cD) - \min_{g \in \cF_{h}}\cL_{h, r}(g, f_{h + 1, r}^{\pi}, \pi; \cD)\\
    &\geq \cL_{h, r}(f_{h, r}^{\pi}, f_{h + 1, r}^{\pi}, \pi; \cD) - \cL_{h, r}(\hat{g}^{\pi}_{h, r}, f_{h + 1, r}^{\pi}, \pi; \cD)\\
    &= \cL_{h, r}(f_{h, r}^{\pi}, f_{h + 1, r}^{\pi}, \pi; \cD)  -\cL_{h, r}(\cT_{h, r}^{\pi}f_{h + 1, r}^{\pi}, f_{h + 1, r}^{\pi}, \pi; \cD) \\
    &\quad -\rbr{\cL_{h, r}(\hat{g}^{\pi}_{h, r}, f_{h + 1, r}^{\pi}, \pi; \cD) -\cL_{h, r}(\cT_{h, r}^{\pi}f_{h + 1, r}^{\pi}, f_{h + 1, r}^{\pi}, \pi; \cD)}.
  \end{align*}
  By Lemma~\ref{corollary:good_event}, conditionally on the event $\cG(\Pi)$ we have the following:
  \begin{align*}
    \cL_{h, r}(f_{h, r}^{\pi}, f_{h + 1, r}^{\pi}, \pi; \cD)  -\cL_{h, r}(\cT_{h, r}^{\pi}f_{h + 1, r}^{\pi}, f_{h + 1, r}^{\pi}, \pi; \cD) &\geq -\epsilon_{\rm S} + \frac{1}{2}\EE_{\mu_h}\sbr{\| f_{h, r}^{\pi} - \cT_{h, r}^{\pi}f_{h + 1, r}^{\pi}\|^{2}},\\
    -\cL_{h, r}(\hat{g}^{\pi}_{h, r}, f_{h + 1, r}^{\pi}, \pi; \cD) +\cL_{h, r}(\cT_{h, r}^{\pi}f_{h + 1, r}^{\pi}, f_{h + 1, r}^{\pi}, \pi; \cD) &\geq -\epsilon_{\rm S} -\frac{3}{2}\EE_{\mu_h}\sbr{\| \hat{g}_{h, r}^{\pi} - \cT_{h, r}^{\pi}f_{h + 1, r}^{\pi}\|^{2}}.
  \end{align*} Recalling that $\cE_{h, r}(f, \pi; \cD) \leq \epsilon_{0}$, we have
  \[
    \EE_{\mu_H}\sbr{\| f_{h, r}^{\pi} - \cT_{h, r}^{\pi}f_{h + 1, r}^{\pi}\|^{2}} \leq 4\epsilon_{\rm S} + 3\EE_{\mu_h}\sbr{\| \hat{g}_{h, r}^{\pi} - \cT_{h, r}^{\pi}h_{h + 1, r}^{\pi}\|^{2}} + 2\epsilon_{0}.
  \] We conclude our proof by reminding ourselves of Assumption~\ref{assumption:completeness}.
\end{proof}


\end{document}